\documentclass[11pt]{article}
\ifdefined\XeTeXversion\else
\pdfoutput=1
\fi
\usepackage{booktabs}
\usepackage{mathrsfs}
\usepackage{amsmath,amssymb}
\usepackage{bm}
\usepackage{natbib}
\usepackage[usenames]{color}
\usepackage{amsthm}

\usepackage{multirow} 
\usepackage{enumitem}
\usepackage{graphicx}
\usepackage{array}
\usepackage{placeins}
\usepackage{tikz}
\usetikzlibrary{arrows.meta,backgrounds,fit,positioning}

\usepackage{mylatexstyle}

\usepackage[colorlinks,
linkcolor=red,
anchorcolor=blue,
citecolor=blue
]{hyperref}

\usepackage[left=1in, right=1in, top=1in, bottom=1in]{geometry}

\newtheorem{theorem}{Theorem}
\newtheorem{proposition}{Proposition}
\newtheorem{lemma}{Lemma}

\newtheorem{definition}{Definition}
\newtheorem{assumption}{Assumption}
\newtheorem{remark}{Remark}

\title{\huge Transformers as Cross-Task Learners: Shared Structure Drives Sample Efficiency in In-Context Learning}
\author{
	Zhongjie Shi \thanks{School of Mathematics, Georgia Institute of Technology, Atlanta, GA 30332, United States; E-mail:  {\tt zshi332@gatech.edu}}
    \qquad
    Rongjie Lai \thanks{Department of Mathematics, Purdue University, West Lafayette, IN 47907, United States; E-mail: {\tt lairj@purdue.edu}}
    \qquad
    Alexander Cloninger \thanks{Department of Mathematics and Hal{\i}c{\i}o\u{g}lu Data Science Institute, University of California, San Diego, La Jolla, CA 92093, United States; E-mail: {\tt acloninger@ucsd.edu}}
    \qquad
	Wenjing Liao \thanks{School of Mathematics, Georgia Institute of Technology, Atlanta, GA 30332, United States; E-mail: {\tt wliao60@gatech.edu}}
}

\date{}

\begin{document}

\maketitle

\begin{abstract}

Transformers achieve remarkable performance by jointly learning broad families of tasks during pretraining and adapting to unseen tasks from only a short prompt. Yet a rigorous mathematical and statistical understanding of this phenomenon remains limited. This paper aims to study how Transformers exploit shared cross-task structure and how this structure affects the sample complexity of in-context learning (ICL). 
Specifically, we characterize task-space complexity through covering numbers under a prescribed metric, thereby quantifying the low-dimensional cross-task structure without requiring an explicit parametric representation. The resulting cover provides a set of anchor functions, which we use to introduce a task-identification-and-evaluation procedure: context observations localize an unseen task among the anchor functions, and the response at a query is predicted by aggregating the corresponding anchor function query evaluations.
For approximation, we explicitly construct a Transformer with Softmax attention to approximate this procedure.  For generalization, we derive an error bound that separates the effects of the number of pretraining tasks and the prompt length. The scaling with respect to the number of pretraining tasks is governed by the intrinsic dimensions of the task space and input domain; once sufficiently many tasks are available, the dependence on the prompt context length becomes dimension-free. To the best of our knowledge, this is the first work to quantify cross-task complexity for general nonlinear task families and explicitly construct a Transformer that exploits their low-dimensional structure to perform ICL. Our theory provides a quantitative explanation of how joint pretraining across related tasks improves in-context generalization.

\end{abstract}

\section{Introduction}
Transformers \citep{vaswani2017attention} have demonstrated a remarkable ability to adapt their predictions
to examples supplied at inference time. This phenomenon, known as
\emph{in-context learning} (ICL), allows a single pretrained model to perform a
new task from a prompt of input--output demonstrations without updating its
parameters \citep{radford2019language,brown2020language,garg2022can}. Although ICL was first
popularized by large language models, the underlying problem is considerably
broader: the prompt specifies an unseen task, and the model must infer
the relevant input--output relation before answering a new query.

The empirical success of Transformers has motivated growing interest in understanding their underlying mechanisms. In the single-task setting, their approximation and generalization properties have been investigated in several works, including \citep{yun2020transformers,gurevych2022rate,takakura2023approximation,havrilla2024understanding,shen2025taskmanifold,shi2026transformer,shi2026learning,tai2026mathematical}. In particular, \citep{havrilla2024understanding,shi2026learning} exploit low-dimensional structure in token embeddings to develop theoretical accounts of neural scaling laws whose exponents depend on the intrinsic dimension of the input domain.

At the multi-task level, theoretical studies of Transformer-based ICL have focused extensively on linear models, where the task family consists of a parametrized collection of linear functions \citep{bai2023transformers,vonoswald2023transformers,akyurek2023what,zhang2024trained,wu2024pretraining}. In this setting, the context examples implicitly specify an unseen task through its parameters, and the Transformer learns to infer these parameters and use them to predict responses to new queries. This viewpoint has led to algorithmic interpretations of ICL in terms of classical learning procedures, including gradient descent and least-squares regression.

Beyond the linear setting, ICL for nonlinear models has been studied from several perspectives. For example, \citep{li2023transformers} interpret task inference in ICL through the lens of algorithm learning, while \citep{cheng2024functional} show that Transformers can implement functional gradient descent to learn nonlinear functions in context. A related line of work connects Transformer attention to kernel methods \citep{tsai2019transformer,yu2024nonlocal,han2025kernel,shen2026understanding,yan2026transformers}. By explicitly constructing a Transformer that implements a kernel algorithm, \citep{shen2026understanding} establish a generalization bound for H\"older functions on a low-dimensional manifold. Other approaches analyze Transformer-based ICL through basis representations \citep{kim2024transformers}, feature learning \citep{hsu2026featurizer}, and local polynomial estimation \citep{ching2026efficient}.


Existing work has provided important insights and generalization guarantees for Transformer-based ICL, but has largely focused on prediction from a single prompt. In practice, however, Transformers are pretrained across task families and can exploit shared structure rather than learn each task independently. This perspective connects ICL to multi-task and meta-learning, where shared low-dimensional representations can improve statistical efficiency and adaptation \citep{evgeniou2005learning,argyriou2006multi,maurer2016benefit,pentina2016lifelong,tripuraneni2021provable}. Such sharing has proven effective across many application domains \citep{zhang2018overview,allenspach2024neural}, and empirical evidence suggests that learned task spaces may themselves be effectively low-dimensional \citep{ramesh2022picture}. These observations motivate a family-level theory of ICL that explicitly accounts for cross-task structure.

\begin{table}[t]
\centering
\caption{Test MSE on a common set of prompts from the hidden two-dimensional subspace. The two Transformers are pretrained on the full 21-dimensional quadratic family and the hidden subspace, respectively; ridge regression uses the full 21-dimensional polynomial basis and is fitted independently to each prompt.}
\label{tab:full-vs-two}
\small
\setlength{\tabcolsep}{10pt}
\begin{tabular}{cccc}
\toprule
$n$ & ICL on full family & ICL on 2D subspace & 21D ridge regression for 2D subspace \\
\midrule
4   & 0.66663 & 0.07832 & 0.67697 \\
8   & 0.50026 & 0.03092 & 0.47792 \\
16  & 0.35547 & 0.01718 & 0.16517 \\
20  & 0.26714 & 0.01623 & 0.03750 \\
\bottomrule
\end{tabular}
\end{table}

We illustrate this benefit using ICL for quadratic polynomials on \([-1,1]^5\). Table~\ref{tab:full-vs-two} compares Transformers pretrained on the full 21-dimensional polynomial family and on a hidden two-dimensional subspace whose active basis is not given to the learner, with all other training settings fixed. At \(n=4\), the hidden-family Transformer achieves an MSE of \(0.0783\), compared with \(0.6666\) for the full-family Transformer. It also substantially outperforms 21-dimensional ridge regression on the same hidden-family tasks. Ridge knows the ambient polynomial basis but estimates each task independently, whereas the Transformer learns the shared subspace during pretraining and uses the prompt primarily to identify the new task within it. Consequently, the Transformer attains an MSE of \(0.0783\) with four observations, while ridge first reaches comparable-or-better accuracy at \(n=20\) among the tested prompt lengths.

This numerical comparison motivates our central view: transformer-based ICL is fundamentally a form of cross-task learning. During pretraining, the Transformer encodes the shared geometry of a task family in its parameters; at inference, the context observations primarily localize the new task within this learned family. When the shared structure is low-dimensional and identifiable, substantially fewer context observations may be required than when each task is learned independently. Cross-task learning thus shifts part of the statistical burden from within-task sampling to pretraining.

Consistent with this view, empirical studies show that ICL depends critically on the diversity and structure of the pretraining tasks \citep{chan2022data,raventos2023pretraining,havrilla2024surveying}. Theoretically, \citep{oko2024pretrained} analyze single-index models sharing a low-dimensional subspace, while \citep{cole2026context} derive an error bound governed by the intrinsic dimension of a task manifold. However, a general theory explaining how Transformers exploit shared structure across nonlinear, potentially nonparametric task families remains underdeveloped. We therefore focus on two complementary aspects: the representation of cross-task structure and its statistical benefit at inference time. This leads to two central questions:
\begin{itemize}
	\item[{\bf Q1:}] \textit{How do transformers exploit low-dimensional cross-task nonlinear structure in ICL?}
    \item[{\bf Q2:}] \textit{How does this cross-task structure reduce the context size needed at inference time?}    
\end{itemize}

\begin{figure}[t]
\centering
\resizebox{0.98\textwidth}{!}{%
\begin{tikzpicture}[
    >=Latex,
    flow/.style={-{Latex[length=1.8mm]},line width=0.7pt,draw=gray!65},
    leader/.style={line width=0.45pt,draw=gray!60},
    taskdot/.style={circle,draw=blue!55!black,fill=white,
        minimum size=1.9mm,inner sep=0pt,line width=0.55pt},
    inputdot/.style={circle,draw=teal!65!black,fill=teal!45,
        minimum size=1.8mm,inner sep=0pt,line width=0.45pt},
    contextdot/.style={circle,draw=blue!55!black,fill=blue!45,
        minimum size=1.8mm,inner sep=0pt},
    taskcell/.style={draw=blue!32!gray,fill=blue!8,
        fill opacity=0.15,draw opacity=0.50,line width=0.4pt},
    inputcell/.style={draw=teal!40!gray,fill=teal!8,
        fill opacity=0.15,draw opacity=0.50,line width=0.4pt},
    paneltitle/.style={font=\bfseries\small,align=center},
    note/.style={font=\scriptsize,align=center},
    result/.style={rounded corners=2pt,draw=gray!40,fill=white,
        font=\scriptsize,align=center,inner xsep=6pt,inner ysep=5pt}
]
    \draw[rounded corners=3pt,draw=gray!30,fill=blue!1]
        (0,0) rectangle (7.45,5.80);
    \draw[rounded corners=3pt,draw=gray!30,fill=teal!1]
        (7.75,0) rectangle (15.60,5.80);
    \node[paneltitle] at (3.73,5.45)
        {(a) Task identification in $\cM_f$};
    \node[paneltitle] at (11.68,5.45)
        {(b) Anchor evaluation on $\cMx$};

    \draw[-{Latex[length=1.3mm]},gray!55] (0.45,2.15)--(1.83,2.15);
    \draw[-{Latex[length=1.3mm]},gray!55] (0.45,2.15)--(0.45,3.76);
    \foreach \x/\y in {0.68/2.46,0.89/2.89,1.10/2.63,
        1.30/3.27,1.52/3.00,1.70/3.49}
        \node[contextdot] at (\x,\y) {};
    \node[note,text width=1.80cm] at (1.14,1.72)
        {$n$ context\\observations};
    \node[note] at (2.30,3.70) {identify};

    \path[fill=blue!3]
        (3.06,1.98) .. controls (2.75,2.71) and (2.98,3.89) .. (3.64,4.39)
        .. controls (4.47,4.63) and (5.55,4.24) .. (6.48,4.45)
        .. controls (6.82,3.66) and (6.55,2.42) .. (6.14,1.93)
        .. controls (5.10,1.69) and (3.88,2.02) .. cycle;
    \foreach \x/\y in {3.57/3.96,4.56/4.03,5.56/3.94,6.29/4.02,
        3.40/3.08,4.39/3.17,5.39/3.10,6.25/3.16,
        3.68/2.27,4.71/2.35,5.72/2.28}
        \draw[taskcell] (\x,\y) ellipse[x radius=0.66,y radius=0.59];
    \draw[blue!48!black,line width=0.75pt]
        (3.06,1.98) .. controls (2.75,2.71) and (2.98,3.89) .. (3.64,4.39)
        .. controls (4.47,4.63) and (5.55,4.24) .. (6.48,4.45)
        .. controls (6.82,3.66) and (6.55,2.42) .. (6.14,1.93)
        .. controls (5.10,1.69) and (3.88,2.02) .. cycle;
    \foreach \x/\y in {3.57/3.96,4.56/4.03,5.56/3.94,6.29/4.02,
        3.40/3.08,4.39/3.17,5.39/3.10,6.25/3.16,
        3.68/2.27,4.71/2.35,5.72/2.28}
        \node[taskdot] at (\x,\y) {};
    \node[taskdot,fill=blue!15] (hl) at (6.29,4.02) {};
    \node[note,text=blue!60!black] (hlabel) at (6.03,4.88) {anchor function $h_l$};
    \draw[leader] (hlabel.south)--(hl.north);
    \node[circle,draw=orange!80!black,fill=orange!65,
        minimum size=2.3mm,inner sep=0pt] (f) at (4.95,3.34) {};
    \draw[flow] (1.93,3.08)--(f.west);
    \node[note,text=orange!75!black] (flabel) at (5.21,1.38) {unseen task $f$};
    \draw[leader] (flabel.north)--(f.south);
    \node[result] (taskout) at (4.73,0.52)
        {$\xb_{n+1}\longrightarrow h_l(\xb_{n+1})
        \longrightarrow f(\xb_{n+1})$};
    \draw[flow] (flabel.south)--(taskout.north);

    \path[fill=teal!3]
        (9.51,1.98) .. controls (9.20,2.71) and (9.43,3.89) .. (10.09,4.39)
        .. controls (10.92,4.63) and (12.00,4.24) .. (12.93,4.45)
        .. controls (13.27,3.66) and (13.00,2.42) .. (12.59,1.93)
        .. controls (11.55,1.69) and (10.33,2.02) .. cycle;
    \foreach \x/\y in {10.02/3.96,11.01/4.03,12.01/3.94,12.74/4.02,
        9.85/3.08,10.84/3.17,11.84/3.10,12.70/3.16,
        10.13/2.27,11.16/2.35,12.17/2.28}
        \draw[inputcell] (\x,\y) circle[radius=0.63];
    \draw[teal!55!black,line width=0.75pt]
        (9.51,1.98) .. controls (9.20,2.71) and (9.43,3.89) .. (10.09,4.39)
        .. controls (10.92,4.63) and (12.00,4.24) .. (12.93,4.45)
        .. controls (13.27,3.66) and (13.00,2.42) .. (12.59,1.93)
        .. controls (11.55,1.69) and (10.33,2.02) .. cycle;
    \foreach \x/\y in {10.02/3.96,11.01/4.03,12.01/3.94,12.74/4.02,
        9.85/3.08,10.84/3.17,11.84/3.10,12.70/3.16,
        10.13/2.27,11.16/2.35,12.17/2.28}
        \node[inputdot] at (\x,\y) {};
    \draw[orange!65!black,line width=0.55pt] (11.01,4.03)--(11.40,3.34);
    \draw[orange!65!black,line width=1.0pt] (10.84,3.17)--(11.40,3.34);
    \draw[orange!65!black,line width=0.85pt] (11.84,3.10)--(11.40,3.34);
    \draw[orange!65!black,line width=0.45pt] (11.16,2.35)--(11.40,3.34);
    \node[circle,draw=orange!80!black,fill=orange!65,
        minimum size=2.3mm,inner sep=0pt] (query) at (11.40,3.34) {};
    \node[note,text=orange!75!black] (qlabel) at (11.76,1.38) {query $\xb_{n+1}$};
    \draw[leader] (qlabel.north)--(query.south);
    \node[note,text=teal!60!black] (zlabel) at (14.5,3.96) {anchor point $\zb_m$};
    \draw[leader] (zlabel.west)--(12.74,4.02);
    \node[result] (eval) at (11.68,0.52)
        {$\{h_l(\zb_m)\}_{m=1}^{C_x}
        \longrightarrow\ h_l(\xb_{n+1})$};
    \draw[flow] (qlabel.south)--(eval.north);

    \draw[flow] (hl.east) to[out=25,in=180] (9.43,4.60);
    \node[note,fill=white,inner xsep=3pt,inner ysep=2pt] at (8.10,4.91)
        {evaluate $h_l$};
\end{tikzpicture}%
}
\caption{Task identification and evaluation from task-space and input-domain
coverings. (a) The $n$ context observations identify the unseen task $f$
relative to representative anchor functions $h_l$ in a covering of task space $\cM_f$.
(b) For each anchor function, the construction uses the anchor values
$h_l(\zb_m)$ associated with a covering of input domain $\cM_{\xb}$. Localizing the query
$\xb_{n+1}$ relative to the anchor points $\zb_m$ combines these values to approximate
$h_l(\xb_{n+1})$. The task-identification weights then aggregate the resulting
anchor-function evaluations to approximate $f(\xb_{n+1})$.}
\label{fig:task-localization-evaluation}
\end{figure}
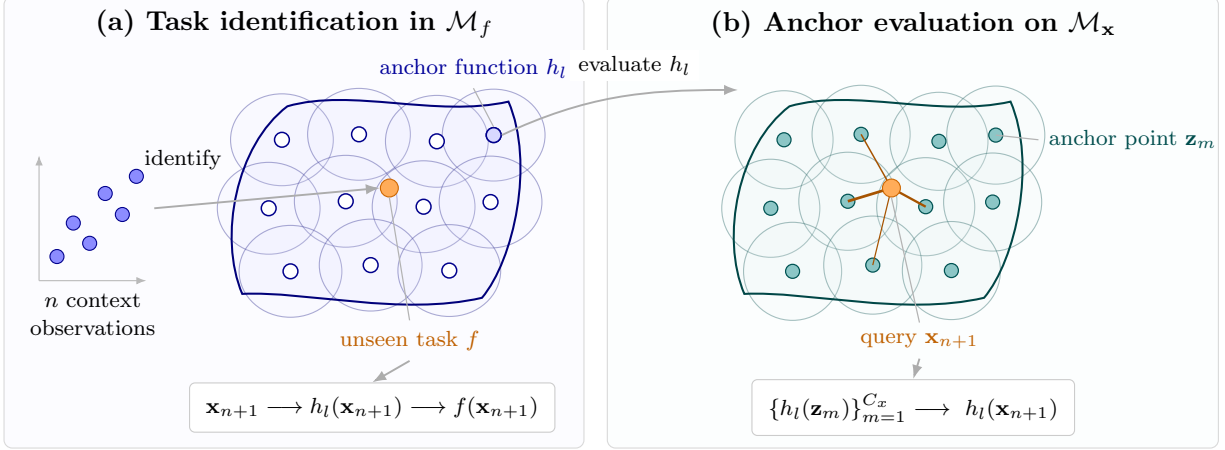

To address the question {\bf Q1}, we proceed in four steps. First, we introduce a complexity characterization for the task space in terms of covering number with a metric in the task space. The scaling of these covering numbers defines an intrinsic task dimension, allowing us to quantify cross-task complexity without assuming a particular parametrization. 
Second, we introduce a
task-identification-and-evaluation procedure as an oracle approximant for ICL. At a fixed covering radius, the centers of the covering balls serve as \emph{anchor functions}.
Given $n$ context observations, this oracle identifies the unseen task relative
to the anchor functions and predicts at the query by aggregating their
query-localized values. As illustrated in Figure~\ref{fig:task-localization-evaluation}, this procedure is realized through a two-level Softmax partition-of-unity (POU) approximation scheme, consisting of
a task-space POU for task identification and an input-domain POU for function
evaluation.
Under suitable continuity conditions, the covering radius and finite-context identification error jointly control the prediction error.
Third, we construct a Transformer with Softmax attention that uniformly approximates this oracle over the admissible tasks and prompts. Combining this approximation with the finite-context analysis yields a high-probability error bound.
Finally, we establish a generalization bound for the empirical risk minimizer over the transformer class by combining the preceding approximation result with a covering-number bound for the Transformer hypothesis class.

To address \textbf{Q2}, we analyze how the generalization error scales with the context length \(n\). Our bound in \eqref{eq:introgen} separates the effects of the number of pretraining tasks and the context length. Its dependence on the number of pretraining tasks follows a power law whose exponent is governed by the intrinsic dimensions of the task and input spaces. This dimension-dependent term constitutes the primary statistical burden and requires a sufficiently large number of pretraining tasks. Once this cost is absorbed during pretraining, the remaining dependence on the context length becomes dimension-free, thereby quantifying inference-time sample efficiency. By contrast, without additional structure, general single-task learning remains subject to the curse of dimensionality. Existing ICL analyses that treat the Transformer as learning each task independently from a single prompt typically inherit the same dimension-dependent rates \citep{kim2024transformers,ching2026efficient,shen2026understanding}.

To our best knowledge, this paper is the first to quantify  cross-task complexity in the general nonlinear setting, and explicitly construct Transformers to perform ICL by exploiting nonlinear cross-task low-dimensional structures.
Our main contributions are summarized as follows.
\begin{enumerate}
	
	\item {\bf Task-space and input-domain complexity and intrinsic dimension.} 
    We quantify task- and input-space complexity through covering numbers under
prescribed metrics. Let \(\cM_{\xb}\) be the input domain and
\(\cM_f\subset L^2(\cM_{\xb})\) the task space.
Assumption~\ref{assum:func_space} postulates that the covering number of
\(\cM_f\) at radius \(r_f\in(0,1]\) scales as \(r_f^{-d_f}\), where \(d_f\)
is regarded as the intrinsic task-space dimension. The input-space dimension \(d_x\) is
defined analogously in Assumption~\ref{assum:input_space}. This formulation
captures nonlinear cross-task structure without requiring an explicit latent
parametrization. Building on classical notions from metric geometry and
statistical learning
\citep{hurewicz2015dimension,kolmogorov1959varepsilon,wainwright2019high},
we connect these complexity measures directly to Transformer approximation
and generalization in ICL.

	
	\item \textbf{A task-identification-and-evaluation procedure implemented by two-level softmax POU.} We introduce a task-identification-and-evaluation procedure as an oracle approximant for
	ICL,  
	illustrated in Figure \ref{fig:task-localization-evaluation}. Furthermore, this procedure is implemented by a two-level softmax POU, consisting of a task-space POU for task identification and an input-domain POU for function evaluation. 
	This procedure adapts to the intrinsic dimensions $d_f$ and $d_{x}$ simultaneously.     
	The resulting method is discretization-free,
	as it allows context observations to vary randomly across prompts without
	imposing a common sampling grid. It also aligns naturally with the
	Transformer's native dot-product Softmax attention architecture.

	\item \textbf{Constructive approximation by a standard Transformer.}
	To approximate the task-identification-and-evaluation procedure above, we explicitly construct a shallow, wide, and dense Transformer $\rmT^*$ 
    as stated in Theorem~\ref{thm:total_approx}. For admissible
	$\epsilon$, the same network has expected squared $L^2(\rhox)$ approximation error
	\begin{equation}
    \sup_{f\in\cM_f}
    \EE_{\xb_1,\ldots,\xb_n}
    \|\rmT^*(\frakc,\cdot)-f\|_{L^2(\rhox)}^2
    =\mathcal O\left(
    \epsilon^2+\frac{\epsilon^{-4}(\log(1/\epsilon))^3}{n}
    \right).
    \label{eq:introapprox}
	\end{equation}
	Here $\frakc$ denotes the sampled context, and the expectation is over its context inputs. This error is achieved with
	$n^2\epsilon^{-q}$ dense parameters up to $\log$ factors with
	$q:=d_f+\frac{3d_x}{\alpha}$.

	\item \textbf{Generalization error across tasks.}
	We establish a generalization error bound for the 
	empirical risk minimizer $\widehat{\rmT}$ trained on $\Gamma$
	meta-training prompts with context length $n$ over a Transformer class
	(Theorem~\ref{thm:icl_generalization}). Let $\cL(\widehat{\rmT}_{\frakS})$ denote the expected
	squared prediction error of $\widehat{\rmT}$ on test samples. Our analysis employs a covering-number bound for the Transformer class and a Bernstein-type oracle inequality to yield
	\begin{equation}
	\EE_{\frakS}\cL(\widehat{\rmT}_{\frakS})
	=\mathcal O\left(
	\max\left\{
	\left(\frac{(\log n)^3}{n}\right)^{
		\frac13},
	\left(
	\frac{n^2(\log n)^{d_x/\alpha+1}}{\Gamma}
	\right)^{\frac{2}{q+2}}
	\right\}
	\right),
    \label{eq:introgen}
	\end{equation}
	where the expectation $\EE_{\frakS}$ is taken over the training data.
	The two terms in \eqref{eq:introgen} separately quantify the effects of the context length $n$
	and the number of meta-training prompts $\Gamma$. When $\Gamma$ is large, i.e., sufficiently many
	meta-training prompts are available, the leading context-length term is of order $n^{-1/3}$ up to logarithmic factors, with a polynomial exponent independent of $d_f$, $d_x$, and $\alpha$. 
\end{enumerate}

The remainder of the paper is organized as follows. Section \ref{sec:setup-architecture} introduces the in-context regression problem and specifies the Transformer architecture.
Section~\ref{sec:main-results} states our assumptions, task-identification-and-evaluation oracle, and main approximation and generalization results. 
Section~\ref{sec:related-work} discusses related works and compares our theoretical results with related
ICL results. Section~\ref{sec:proof-techniques} contains our proof of main results. Complete proofs and the explicit
Transformer construction are provided in
Appendices~\ref{app:oracle-form}--\ref{app:generalization_proofs}.

\paragraph{Notation.}
We use lower-case bold letters for vectors, upper-case bold letters for matrices, and calligraphic letters for sets, spaces, and operators. We write $\NN$ and $\RR$ for the natural and real numbers, respectively, and $[N]:=\{1,\ldots,N\}$ for $N\in\NN$. For vectors $\bx,\by\in\RR^q$, $\langle\bx,\by\rangle$ denotes the Euclidean inner product and $\|\bx\|_p$ denotes the standard $\ell_p$ norm. We use $\be_r$ for the $r$-th standard basis vector, whose ambient dimension is inferred from context so that each matrix product is well defined. We use $\bm{0}_q$ and $\bm{1}_q$ for the zero and all-one vectors in $\RR^q$, $\bm{0}_{p\times q}$ for the zero matrix in $\RR^{p\times q}$, and $\bI_q$ for the $q\times q$ identity matrix. For $\bX\in\RR^{p\times q}$, $(\bX)_{r,j}$ is its $(r,j)$-th entry, $(\bX)_{:,j}$ and $(\bX)_{r,:}$ are its $j$-th column and $r$-th row, $\mathrm{vec}(\bX)$ is its column-wise vectorization, and $\|\bX\|_{\max}:=\max_{r,j}|(\bX)_{r,j}|$. For measurable functions $f,g$, $\langle f,g\rangle_{L^2(\rhox)}:=\int f(\xb)g(\xb)\,d\rhox(\xb)$, $\|f\|_{L^2(\rhox)}:=\langle f,f\rangle_{L^2(\rhox)}^{1/2}$, and $\|f\|_{L^\infty(\cMx)}:=\sup_{\xb\in\cMx}|f(\xb)|$. We write $1_{\{E\}}$ for the indicator of an event or condition $E$, $\mathcal N(\eta,\cG,\|\cdot\|)$ for the minimal cardinality of an $\eta$-cover of $\cG$ under $\|\cdot\|$, and $\EE$ and $\PP$ for expectation and probability. The notation $\widetilde{\mathcal O}$ suppresses logarithmic factors.

\section{Problem Setup and Transformer Architecture}
\label{sec:setup-architecture}

This section introduces the in-context regression problem, the meta-training
data set, the population and empirical risks, and the Transformer architecture
used throughout the paper.

\subsection{In-Context Regression Problem}
\label{sec:problem-setup}

Let $\cM_f$ be a task function space and let
$\cMx\subset[0,1]^d$ be an input domain. 
A task $f:\cMx\rightarrow \RR$ is sampled from a probability distribution $\rho_f$ supported on
$\cM_f$. Independently of $f$, the inputs
$\xb_1,\ldots,\xb_{n+1}$ are sampled i.i.d. from a probability distribution
$\rhox$ supported on $\cMx$. Given
$f\sim\rho_f$ and $\xb_1,\ldots,\xb_{n+1}\overset{\mathrm{i.i.d.}}{\sim}\rhox$,
the context $\frakc$ and prompt $\fraks$ are
\begin{equation}
    \label{eq:onesample}
    \frakc:=\{(\xb_i,y_i)\}_{i=1}^n,
    \qquad
    \fraks:=(\frakc,\xb_{n+1}),
    \qquad y_i:=f(\xb_i).
\end{equation}
The $n$ input--output pairs form the context $\frakc$, $\xb_{n+1}$ is the
query, and the prediction target is $y_{n+1}=f(\xb_{n+1})$.
For a predictor $\rmT$ on prompts, we write
$\rmT(\fraks)=\rmT(\frakc,\xb_{n+1})$; when $\frakc$ fixed,
$\rmT(\frakc,\cdot)$ denotes its prediction as a function of the query.
The population risk is
\begin{equation}
    \label{population-risk}
    \begin{aligned}
    \cL(\rmT)
    &:=
    \EE_{f\sim\rho_f}
    \EE_{\xb_1,\ldots,\xb_{n+1}\overset{\mathrm{i.i.d.}}{\sim}\rhox}
    \left[
    \big(\rmT(\fraks)-f(\xb_{n+1})\big)^2
    \right]\\
    &=
    \EE_{f\sim\rho_f}
    \EE_{\xb_1,\ldots,\xb_n\overset{\mathrm{i.i.d.}}{\sim}\rhox}
    \left[\|\rmT(\frakc,\cdot)-f\|_{L^2(\rhox)}^2\right].
    \end{aligned}
\end{equation}

In practice, we only have access to finite training data, so we replace the population risk by its empirical counterpart.
The meta-training data set
\begin{equation} \label{eq:meta-training-data}
	\frakS
	:=\{(\fraks^\gamma,y_{n+1}^\gamma)\}_{\gamma=1}^{\Gamma}
\end{equation}
consists of $\Gamma$ independent task prompts and their corresponding prediction targets.
Specifically,
$
f^1,\ldots,f^\Gamma
\overset{\mathrm{i.i.d.}}{\sim}\rho_f,
$
and for each $\gamma\in[\Gamma]$, independently of $f^\gamma$, $
\xb_1^\gamma,\ldots,\xb_{n+1}^\gamma
\overset{\mathrm{i.i.d.}}{\sim}\rhox.$
We then set
\begin{equation*}
    \begin{aligned}
    \frakc^\gamma&:=\{(\xb_i^\gamma,f^\gamma(\xb_i^\gamma))\}_{i=1}^n,
    \qquad \fraks^\gamma:=(\frakc^\gamma,\xb_{n+1}^\gamma), \qquad y_{n+1}^\gamma&:=f^\gamma(\xb_{n+1}^\gamma).
    \end{aligned}
\end{equation*}

The corresponding empirical risk is
\begin{equation}
	\label{eq:empirical-risk}
	\cL_{\frakS}(\rmT)
	:=\frac1\Gamma\sum_{\gamma=1}^{\Gamma}
	\big(\rmT(\fraks^\gamma)-y_{n+1}^\gamma\big)^2.
\end{equation}
For a specified hypothesis class $\cG$, we define an empirical risk minimizer
by
\begin{equation}
	\label{eq:erm}
	\widehat{\rmT}_{\frakS}
	\in
	\argmin_{\rmT\in\cG}\cL_{\frakS}(\rmT).
\end{equation}
To quantify the generalization performance of this data-dependent estimator,
our goal is to bound its expected population risk over the meta-training data $\frakS$:
\[
    \EE_{\frakS}\cL(\widehat{\rmT}_{\frakS}).
\]
The clipped Transformer hypothesis class used in our generalization result is specified
in Section~\ref{sec:generalization}.

\subsection{Transformer Architecture}
\label{sec:architecture}

The architecture consists of a pre-processing stage followed by $L$ encoder
blocks, each comprising a multi-head attention (MHA) layer and a point-wise
feed-forward neural network (FFN) layer. For the prompt $\fraks =\{(\xb_i,y_i)_{i=1}^n,\xb_{n+1}\}$ in
\eqref{eq:onesample} and a sequence length $P\ge n+1$, define the padded token matrix
\[
    \bX(\fraks)
    :=
    \begin{bmatrix}
        \xb_1 & \cdots & \xb_n & \xb_{n+1}
        & \bm{0}_{d\times(P-n-1)}\\
        y_1 & \cdots & y_n & 0
        & \bm{0}_{P-n-1}^{\top}
    \end{bmatrix}
    \in\RR^{(d+1)\times P}.
\]
The pre-processing step applies a shared
affine embedding to these tokens and adds a structural-positional encoding:
\[
    \bZ_0=\cP(\fraks)
    :=\bW_E\bX(\fraks)+\bb_E\bm{1}_P^\top+\bP,
\]
where $\bW_E\in\RR^{D\times(d+1)}$, $\bb_E\in\RR^D$, and
$\bP\in\RR^{D\times P}$ is the structural-positional encoding matrix. The
positional rows in our construction use sinusoidal encodings, consistent with
the standard Transformer architecture \citep{vaswani2017attention}.

Consider the $\ell$-th encoder block, $\ell\in[L]$, with input
$\bZ_{\ell-1} \in\RR^{D\times P}$. Its MHA layer
$\cA_\ell:\RR^{D\times P}\to\RR^{D\times P}$ has $H^\ell$ heads.
For each $h\in[H^\ell]$, the query, key, and value matrices satisfy
$\bQ_\ell^h\in\RR^{d_k^\ell\times D}$,
$\bK_\ell^h\in\RR^{d_k^\ell\times D}$, and
$\bV_\ell^h\in\RR^{d_v^\ell\times D}$, respectively. The output of the
$h$-th attention head ${\head}_\ell^h\in\RR^{d_v^\ell\times P}$ is
computed as
\begin{equation*}
    {\head}_\ell^h=\bV_\ell^h\bZ_{\ell-1}\bA_\ell^h,
\end{equation*}
where $\bA_\ell^h\in\RR^{P\times P}$ is the attention probability matrix.
For $t,j\in[P]$, define
\begin{equation*}
    s^{\ell,h}_{t,j}:=
    \left(\bZ_{\ell-1}^\top {\bK_\ell^h}^\top
    \bQ_\ell^h \bZ_{\ell-1}\right)_{t,j},
    \qquad
    (\bA_\ell^h)_{:,j}
    :=\softmax\!\left(
    \big(s^{\ell,h}_{1,j},\ldots,s^{\ell,h}_{P,j}\big)^\top
    \right).
\end{equation*}
Here $\softmax: \RR^{P} \to \RR^{P}$ is defined component-wise by
\begin{equation*}
    \softmax(\bx)
    :=\left(
    \frac{\exp(x_1)}{\sum_{t=1}^P\exp(x_t)},\ldots,
    \frac{\exp(x_P)}{\sum_{t=1}^P\exp(x_t)}
    \right)^\top.
\end{equation*}
The MHA output
$\widehat{\bZ}_\ell\in\RR^{D\times P}$ is obtained by concatenating the
outputs of all $H^\ell$ heads along the feature dimension and applying a
linear projection:
\begin{equation*}
    \widehat{\bZ}_\ell
    =\bW_\ell^O
    \begin{bmatrix}
        {\head}_\ell^1\\ \vdots\\ {\head}_\ell^{H^\ell}
    \end{bmatrix},
\end{equation*}
where $\bW_\ell^O\in\RR^{D\times(H^\ell d_v^\ell)}$ is the output
projection matrix.

The MHA output $\widehat{\bZ}_\ell\in\RR^{D\times P}$ is subsequently
passed through the point-wise FFN
$\cF_\ell:\RR^{D\times P}\to\RR^{D\times P}$. Its weight matrices satisfy
$\bW_\ell^1\in\RR^{d_{\mathrm{ff}}^\ell\times D}$,
$\bW_\ell^2\in\RR^{D\times d_{\mathrm{ff}}^\ell}$, and bias vectors satisfy
$\bb_\ell^1\in\RR^{d_{\mathrm{ff}}^\ell}$, and
$\bb_\ell^2\in\RR^D$, and it computes
\begin{equation} \label{FFNlayer}
    (\bZ_\ell)_{:,j}
    =\bW_\ell^2\sigma\!\left(
    \bW_\ell^1(\widehat{\bZ}_\ell)_{:,j}+\bb_\ell^1
    \right)+\bb_\ell^2,
    \qquad j\in[P],
\end{equation}
where $\sigma(\cdot) = \max(0, \cdot)$ is the ReLU activation function applied component-wise.

Finally, for the Transformer model with $L$ encoder blocks, its output is
\begin{equation*}
    \cT_L(\fraks) = \cF_L \circ \cA_L \circ \cdots \circ \cF_1 \circ \cA_1 \circ \cP(\fraks) \in \RR^{D \times P}.
\end{equation*}
The final scalar output is obtained through the linear map
$\bc_{L+1}^\top \mathrm{vec}(\cT_L(\fraks))$. The following definition
formalizes the resulting Transformer hypothesis class.

\begin{definition}[Transformer Network Class] \label{def_transformer_class}
    For depth $L \in \NN$, embedding dimension $D \in \NN$, sequence length $P \in \NN$, layer configurations $\{H^\ell\}_{\ell=1}^L, \{d_k^\ell\}_{\ell=1}^L, \{d_v^\ell\}_{\ell=1}^L, \{d_{\text{ff}}^\ell\}_{\ell=1}^L \subset \NN$, and parameter magnitude bound $M > 0$, we define the class of Transformer networks as
    \begin{align*}
        &\mathcal{T}\Big(L, D, P, \{H^\ell\}_{\ell=1}^L, \{d_k^\ell\}_{\ell=1}^L, \{d_v^\ell\}_{\ell=1}^L, \{d_{\text{ff}}^\ell\}_{\ell=1}^L, M\Big) \\
        &\quad = \left\{ f_\theta \;\left|\; 
        \begin{aligned}
            &f_\theta(\fraks)= \bc_{L+1}^\top \mathrm{vec}(\cT_L(\fraks)) \text{ is an } L\text{-block Transformer with embedding dim } D,  \\
            &\text{ sequence length } P, \text{parameter bound } \|\theta\|_\infty \le M, \text{ and for each block } \ell \in [L]:  \\
            &\text{there are } H^\ell \text{ heads, with query/key dimension }
            d_k^\ell\text{ and value dimension }d_v^\ell,\\
            &\text{and the FFN hidden width is }d_{\text{ff}}^\ell
        \end{aligned}
        \right. \right\}.
    \end{align*}
\end{definition}

\section{Main Results}
\label{sec:main-results}

This section develops the four components of our main theory. Section
\ref{sec:structural-assumptions} characterizes the complexity of the task and
input spaces through metric covering numbers and their associated intrinsic
dimensions. Section \ref{sec:oracle-approximant} uses these coverings to build
a task-identification-and-evaluation oracle implemented by a two-level
Softmax POU. Section \ref{sec:approximation} constructs a standard Transformer
that approximates this oracle, and Section \ref{sec:generalization}
establishes its generalization across tasks when the Transformer is trained by
empirical risk minimization over meta-training data set.

\subsection{Assumptions on Task-Space and Input-Domain Complexity}
\label{sec:structural-assumptions}

To study general, possibly nonlinear, cross-task structure, we seek a notion of complexity that does not rely on a prescribed finite-dimensional parameterization or coordinate system. We therefore characterize the task space through its metric covering numbers, which measure how many representative tasks are needed to describe the entire task family at a given resolution. This provides a direct geometric description of cross-task complexity and naturally yields anchor functions as centers of a task-space cover. We characterize the input domain analogously by its metric covering complexity. Finally, we impose uniform boundedness and H\"older regularity to control the approximation of anchor-function values on the input domain.

\begin{assumption}[Task Function Space $\cM_f$]
\label{assum:func_space}
The tasks are sampled from a prior distribution $\rho_f$ supported on the task
function space $\cM_f$, which satisfies the following properties:
\begin{enumerate}
    \item \textbf{Low-dimensional Structure:} For any $r_f\in(0,1]$, there
    exists a set of anchor functions
    $\{h_l\}_{l=1}^{C_f}\subset\cM_f$ forming an $r_f$-cover under the
    $L^2(\rhox)$ norm, with
    \begin{equation}
        C_f\le C_{\cM}r_f^{-d_f},
        \label{eq:Cf}
    \end{equation}
    for some universal constant $C_{\cM}\ge1$ and intrinsic dimension
    $d_f$.
    \item \textbf{H\"older Smoothness:} Any function $f \in \cM_f$ is $\alpha$-H\"older smooth on $\cMx$ with $\alpha \in (0,1]$ and uniformly bounded $\mathcal{C}^\alpha$-norm. That is, there exists a uniform constant $L_f > 0$ such that $|f(\xb) - f(\xb')| \le L_f \|\xb - \xb'\|_2^\alpha$ for all $\xb, \xb' \in \cMx$.
    \item \textbf{Uniform Boundedness:} There exists $B_f > 0$ such that $\|f\|_{L^\infty(\cMx)} \le B_f$ for all $f \in \cM_f$.
\end{enumerate}
\end{assumption}

\begin{remark}[On the Low-dimensional Structure of $\cM_f$]
The covering bound $C_f\le C_{\cM}r_f^{-d_f}$ implies
$\log C_f=\mathcal O(\log(1/r_f))$ and encodes a finite intrinsic
dimension $d_f$ for the task function space $\cM_f$. In contrast, classical
infinite-dimensional smoothness classes $\cF$, such as
Sobolev and Besov balls, typically have metric entropy of the form
$\log\mathcal N(r,\cF,\|\cdot\|)\asymp r^{-p}$, where $p>0$ is the
metric-entropy exponent determined by the domain dimension, smoothness, and
covering norm. For example, an $s$-smooth class on a $d$-dimensional domain
typically has $p=d/s$ under standard choices of norm. The present assumption
is therefore suited to task families governed by low-dimensional latent
structure. In the ICL setting, a representative example is a family
$f_{\zb}(\cdot)$ parameterized by a latent concept vector
$\zb\in\RR^{d_f}$. Other examples include bounded-degree polynomial
families, generalized linear models, functions parameterized by
low-dimensional manifolds, and neural networks with fixed architectures and
bounded parameters.
\end{remark}

\begin{remark}[The task-space metric and prediction risk]
The $L^2(\rhox)$ metric matches the population risk in
\eqref{population-risk}, which averages the squared prediction error over an
independent query. We therefore measure task approximation in $L^2(\rhox)$.
Theorem~\ref{thm:total_approx} bounds the expected squared $L^2(\rhox)$
approximation error over the context inputs uniformly over tasks, and thus
controls the population risk for any task distribution supported on $\cM_f$.
\end{remark}

\begin{assumption}[Input Domain $\cMx$]
\label{assum:input_space}
The context and query inputs
are i.i.d. samples from a distribution $\rhox$ supported on a compact set
$\cMx\subset[0,1]^d$.
Given any radius $r_x\in(0,1]$, there exists a set of anchor points
$\{\zb_m\}_{m=1}^{C_x}\subset\cMx$ forming an $r_x$-cover under the ambient
Euclidean norm $\|\cdot\|_2$. Its cardinality satisfies
\begin{equation}
    C_x \le C_2 r_x^{-d_x},
            \label{eq:Cx}
\end{equation}
for some universal constant $C_2 > 0$ and intrinsic dimension $d_x \le d$.
\end{assumption}

\begin{remark}
Assumption \ref{assum:input_space} characterizes the geometric complexity of
the input data. By defining the covering number directly under the ambient
$\ell_2$ norm, it accommodates the full cube $[0,1]^d$, compact subsets of
lower-dimensional affine subspaces, and compact Riemannian manifolds embedded
in $[0,1]^d$, without requiring explicit intrinsic coordinates or a
prescribed geometric model.
\end{remark}

\subsection{Task Identification and Evaluation via Two-Level Softmax POU}
\label{sec:oracle-approximant}

Building upon the covering complexity characterization in
Section~\ref{sec:structural-assumptions}, here we introduce a
task-identification-and-evaluation oracle for in-context regression. At a
fixed resolution, the task-space cover provides representative anchor
functions, while the input-domain cover provides representative anchor
points. Given a finite prompt, the oracle uses the $n$ context observations
to identify the unseen task relative to the anchor functions and then
predicts at the query using the corresponding anchor-function values. In
this way, the prediction exploits the low-dimensional cross-task structure
of $\cM_f$, rather than treating the unseen task as an independent regression
problem.

We realize this procedure through a two-level Softmax POU. The task-space POU identifies the unseen task $f$ relative to the anchor
functions $\{h_l\}_{l=1}^{C_f}\subset\cM_f$, while the input-domain POU approximates the values of
these anchor functions at the query from their values at the anchor points
$\{\zb_m\}_{m=1}^{C_x}\subset\cMx$. This construction accommodates randomly
located context observations without requiring a common sampling grid and
naturally aligns with the Transformer's dot-product Softmax attention
mechanism.

We develop the oracle through three approximation steps, followed by a joint
Softmax representation. The proofs for this subsection are provided in
Appendix~\ref{app:oracle-form}.
We begin at the population level. The task-space POU uses population inner
products to identify the unseen task relative to the anchor functions.

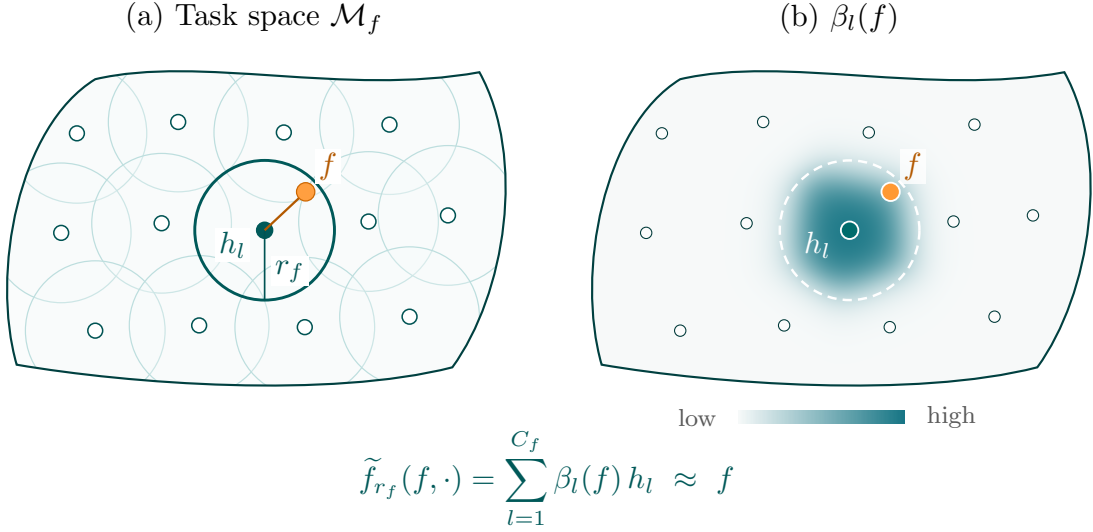
\begin{figure}[t]
\centering
\resizebox{0.90\textwidth}{!}{%
\begin{tikzpicture}[
    >=Latex,
    anchorpoint/.style={circle,draw=teal!65!black,fill=white,
        minimum size=1.7mm,inner sep=0pt,line width=0.5pt},
    covercell/.style={draw=teal!26,fill=teal!3,
        fill opacity=0.28,line width=0.4pt},
    note/.style={font=\small,inner sep=1.2pt}
]
    \definecolor{pouLow}{rgb}{0.97,0.98,0.98}
    \definecolor{pouHigh}{rgb}{0.09,0.46,0.52}
    \pgfdeclarefunctionalshading{pouWeight2}{\pgfpoint{0bp}{0bp}}{\pgfpoint{100bp}{100bp}}{}{%
        25 sub 50 div 4.4 mul exch 25 sub 50 div 6.4 mul exch
        0
        2 index 2 index 3.3 sub dup mul exch 1.17 sub dup mul add -3.2 mul 2.718281828 exch exp add
        2 index 2 index 3.43 sub dup mul exch 2.33 sub dup mul add -3.2 mul 2.718281828 exch exp add
        2 index 2 index 3.31 sub dup mul exch 3.54 sub dup mul add -3.2 mul 2.718281828 exch exp add
        2 index 2 index 3.4 sub dup mul exch 4.75 sub dup mul add -3.2 mul 2.718281828 exch exp add
        2 index 2 index 2.16 sub dup mul exch 0.99 sub dup mul add -3.2 mul 2.718281828 exch exp add
        2 index 2 index 2.27 sub dup mul exch 2.14 sub dup mul add -3.2 mul 2.718281828 exch exp add
        2 index 2 index 2.19 sub dup mul exch 3.32 sub dup mul add -3.2 mul 2.718281828 exch exp add
        2 index 2 index 2.29 sub dup mul exch 4.51 sub dup mul add -3.2 mul 2.718281828 exch exp add
        2 index 2 index 2.36 sub dup mul exch 5.42 sub dup mul add -3.2 mul 2.718281828 exch exp add
        2 index 2 index 1.04 sub dup mul exch 1.38 sub dup mul add -3.2 mul 2.718281828 exch exp add
        2 index 2 index 1.1 sub dup mul exch 2.57 sub dup mul add -3.2 mul 2.718281828 exch exp add
        2 index 2 index 1.08 sub dup mul exch 3.78 sub dup mul add -3.2 mul 2.718281828 exch exp add
        2 index 2 index 1.2 sub dup mul exch 4.98 sub dup mul add -3.2 mul 2.718281828 exch exp add
        2 index 2 index 2.19 sub dup mul exch 3.32 sub dup mul add -3.2 mul 2.718281828 exch exp exch div
        exch pop exch pop
        dup -0.88 mul 0.97 add exch
        dup -0.52 mul 0.98 add exch
        -0.46 mul 0.98 add
    }
    \node[font=\small] at (3.20,4.62) {(a) Task space $\cM_f$};
    \node[font=\small] at (9.90,4.62) {(b) $\beta_l(f)$};
    \path[fill=teal!2] (0.48,0.65) .. controls (0.18,1.70) and (0.50,3.38) .. (1.38,3.92)
            .. controls (2.58,4.20) and (4.24,3.72) .. (5.78,4.00)
            .. controls (6.32,3.02) and (6.08,1.46) .. (5.47,0.62)
            .. controls (4.06,0.26) and (1.77,0.43) .. cycle;
    \begin{scope}
        \clip (0.48,0.65) .. controls (0.18,1.70) and (0.50,3.38) .. (1.38,3.92)
            .. controls (2.58,4.20) and (4.24,3.72) .. (5.78,4.00)
            .. controls (6.32,3.02) and (6.08,1.46) .. (5.47,0.62)
            .. controls (4.06,0.26) and (1.77,0.43) .. cycle;
    \foreach \x/\y in {1.17/3.3,2.33/3.43,3.54/3.31,4.75/3.4,0.99/2.16,2.14/2.27,3.32/2.19,4.51/2.29,5.42/2.36,1.38/1.04,2.57/1.1,3.78/1.08,4.98/1.2}
        \draw[covercell] (\x,\y) circle[radius=0.80];
    \end{scope}
    \draw[teal!50!black,line width=0.7pt] (0.48,0.65) .. controls (0.18,1.70) and (0.50,3.38) .. (1.38,3.92)
            .. controls (2.58,4.20) and (4.24,3.72) .. (5.78,4.00)
            .. controls (6.32,3.02) and (6.08,1.46) .. (5.47,0.62)
            .. controls (4.06,0.26) and (1.77,0.43) .. cycle;
    \draw[teal!70!black,line width=0.95pt] (3.32,2.19) circle[radius=0.80];
    \foreach \x/\y in {1.17/3.3,2.33/3.43,3.54/3.31,4.75/3.4,0.99/2.16,2.14/2.27,3.32/2.19,4.51/2.29,5.42/2.36,1.38/1.04,2.57/1.1,3.78/1.08,4.98/1.2}
        \node[anchorpoint] at (\x,\y) {};
    \node[circle,fill=teal!70!black,inner sep=0pt,minimum size=2mm]
        at (3.32,2.19) {};
    \node[note,anchor=east,text=teal!70!black,fill=white]
        at (3.16,2.02) {$h_l$};
    \draw[teal!65!black,line width=0.55pt] (3.32,2.19)--(3.32,1.39);
    \node[note,anchor=west,text=teal!65!black,fill=white]
        at (3.39,1.72) {$r_f$};
    \draw[orange!70!black,line width=0.75pt] (3.32,2.19)--(3.79,2.63);
    \node[circle,draw=orange!80!black,fill=orange!75,
        minimum size=2.1mm,inner sep=0pt] at (3.79,2.63) {};
    \node[note,anchor=south west,text=orange!70!black,fill=white]
        at (3.89,2.72) {$f$};
    \begin{scope}[xshift=6.7cm]
        \begin{scope}
            \clip (0.48,0.65) .. controls (0.18,1.70) and (0.50,3.38) .. (1.38,3.92)
            .. controls (2.58,4.20) and (4.24,3.72) .. (5.78,4.00)
            .. controls (6.32,3.02) and (6.08,1.46) .. (5.47,0.62)
            .. controls (4.06,0.26) and (1.77,0.43) .. cycle;
            \shade[shading=pouWeight2] (0,0) rectangle (6.4,4.4);
        \end{scope}
        \draw[teal!50!black,line width=0.7pt] (0.48,0.65) .. controls (0.18,1.70) and (0.50,3.38) .. (1.38,3.92)
            .. controls (2.58,4.20) and (4.24,3.72) .. (5.78,4.00)
            .. controls (6.32,3.02) and (6.08,1.46) .. (5.47,0.62)
            .. controls (4.06,0.26) and (1.77,0.43) .. cycle;
        \foreach \x/\y in {1.17/3.3,2.33/3.43,3.54/3.31,4.75/3.4,0.99/2.16,2.14/2.27,3.32/2.19,4.51/2.29,5.42/2.36,1.38/1.04,2.57/1.1,3.78/1.08,4.98/1.2}
            \node[circle,draw=teal!45!black,fill=white,fill opacity=0.65,
                minimum size=1.3mm,inner sep=0pt,line width=0.35pt] at (\x,\y) {};
        \draw[white,densely dashed,line width=0.8pt]
            (3.32,2.19) circle[radius=0.80];
        \node[circle,draw=white,fill=teal!85!black,
            minimum size=2mm,inner sep=0pt,line width=0.6pt] at (3.32,2.19) {};
        \node[note,anchor=east,text=white] at (3.16,2.02) {$h_l$};
        \node[circle,draw=white,fill=orange!80,
            minimum size=2.1mm,inner sep=0pt,line width=0.6pt] at (3.79,2.63) {};
        \node[note,anchor=south west,text=orange!70!black,fill=white,
            fill opacity=0.8,text opacity=1] at (3.89,2.72) {$f$};
    \end{scope}
    \shade[left color=pouLow,right color=pouHigh]
        (8.75,-0.02) rectangle (10.65,0.12);
    \node[font=\scriptsize,anchor=east,text=black!65] at (8.63,0.05) {low};
    \node[font=\scriptsize,anchor=west,text=black!65] at (10.77,0.05) {high};
    \node[font=\small,text=teal!70!black] at (6.55,-0.65)
        {$\displaystyle \tilde f_{r_f}(f,\cdot)
        =\sum_{l=1}^{C_f}\beta_l(f)\,h_l
        \ \approx\ f$};
\end{tikzpicture}%
}
\caption{Task-space Softmax POU. Left: an $r_f$-cover of $\cM_f$, with an unseen task $f$ and a highlighted anchor function $h_l$. Right: schematic concentration of the task-identification weight $\beta_l(f)$ around $h_l$, illustrated by a continuous color scale. The weight is normalized over all anchor functions.}
\label{fig:task-space-softmax-pou}
\end{figure}

\begin{lemma}[Task-Space POU: Task Identification]
\label{lemma:outer_pou}
Under Assumption \ref{assum:func_space}, fix any covering radius $r_f\in(0,1]$.
There exists a finite set of anchor functions
$\{h_l\}_{l=1}^{C_f}\subset\cM_f$ forming an $r_f$-cover of $\cM_f$
under the $L^2(\rhox)$ norm. The number of anchor functions satisfies \eqref{eq:Cf}.
Set the scaling parameter
\[
    M_f
    :=
    \max\left\{
    1,\,
    \frac{1}{3r_f^2}
    \log\frac{2B_fC_f}{r_f}
    \right\}.
\]
For $f\in\cM_f$ and $l\in[C_f]$, define the population task-space POU weight
\begin{align}
\beta_l(f)
&:=
\frac{\exp\left(M_f(r_f^2-\|f-h_l\|_{L^2(\rhox)}^2)\right)}
{\sum_{k=1}^{C_f}\exp\left(M_f(r_f^2-\|f-h_k\|_{L^2(\rhox)}^2)\right)} =
\frac{\exp\left(2M_f\langle f,h_l\rangle_{L^2(\rhox)}
-M_f\|h_l\|_{L^2(\rhox)}^2\right)}
{\sum_{k=1}^{C_f}
\exp\left(2M_f\langle f,h_k\rangle_{L^2(\rhox)}
-M_f\|h_k\|_{L^2(\rhox)}^2\right)} .
\end{align}
The weights satisfy $\beta_l(f)\ge0$ and
$\sum_{l=1}^{C_f}\beta_l(f)=1$. The resulting population task-space POU approximant
\[
    \tilde f_{r_f}(f,\xb):=\sum_{l=1}^{C_f}\beta_l(f)h_l(\xb)
\]
satisfies
\begin{equation}
\label{outer-pou-error}
    \|\tilde f_{r_f}(f,\cdot)-f\|_{L^2(\rhox)}
    \le 3r_f,\qquad \forall f\in\cM_f.
\end{equation}
\end{lemma}

Lemma \ref{lemma:outer_pou} is proved in Appendix \ref{appsub:prooflemma:outer_pou}. 
Figure \ref{fig:task-space-softmax-pou} illustrates the task-space Softmax
POU. Every anchor function receives a positive weight, while anchors closer to
the unseen task in $L^2(\rhox)$ receive larger weights. The resulting convex
combination provides an approximation based on task identification without requiring coordinates
on the task space.

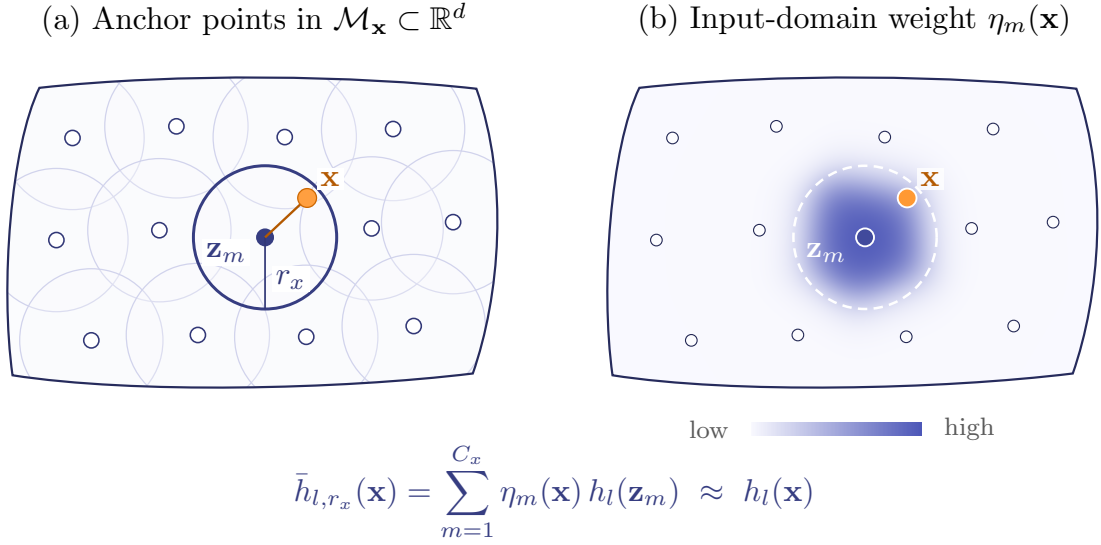
\begin{figure}[t]
\centering
\resizebox{0.90\textwidth}{!}{%
\definecolor{pouInput}{rgb}{0.30,0.34,0.72}
\begin{tikzpicture}[
    >=Latex,
    anchorpoint/.style={circle,draw=pouInput!65!black,fill=white,
        minimum size=1.7mm,inner sep=0pt,line width=0.5pt},
    covercell/.style={draw=pouInput!26,fill=pouInput!3,
        fill opacity=0.28,line width=0.4pt},
    note/.style={font=\small,inner sep=1.2pt}
]
    \definecolor{pouLow}{rgb}{0.98,0.98,1.00}
    \definecolor{pouHigh}{rgb}{0.30,0.34,0.72}
    \pgfdeclarefunctionalshading{pouWeight3}{\pgfpoint{0bp}{0bp}}{\pgfpoint{100bp}{100bp}}{}{%
        25 sub 50 div 4.4 mul exch 25 sub 50 div 6.4 mul exch
        0
        2 index 2 index 3.3 sub dup mul exch 1.17 sub dup mul add -3.2 mul 2.718281828 exch exp add
        2 index 2 index 3.43 sub dup mul exch 2.33 sub dup mul add -3.2 mul 2.718281828 exch exp add
        2 index 2 index 3.31 sub dup mul exch 3.54 sub dup mul add -3.2 mul 2.718281828 exch exp add
        2 index 2 index 3.4 sub dup mul exch 4.75 sub dup mul add -3.2 mul 2.718281828 exch exp add
        2 index 2 index 2.16 sub dup mul exch 0.99 sub dup mul add -3.2 mul 2.718281828 exch exp add
        2 index 2 index 2.27 sub dup mul exch 2.14 sub dup mul add -3.2 mul 2.718281828 exch exp add
        2 index 2 index 2.19 sub dup mul exch 3.32 sub dup mul add -3.2 mul 2.718281828 exch exp add
        2 index 2 index 2.29 sub dup mul exch 4.51 sub dup mul add -3.2 mul 2.718281828 exch exp add
        2 index 2 index 2.36 sub dup mul exch 5.42 sub dup mul add -3.2 mul 2.718281828 exch exp add
        2 index 2 index 1.04 sub dup mul exch 1.38 sub dup mul add -3.2 mul 2.718281828 exch exp add
        2 index 2 index 1.1 sub dup mul exch 2.57 sub dup mul add -3.2 mul 2.718281828 exch exp add
        2 index 2 index 1.08 sub dup mul exch 3.78 sub dup mul add -3.2 mul 2.718281828 exch exp add
        2 index 2 index 1.2 sub dup mul exch 4.98 sub dup mul add -3.2 mul 2.718281828 exch exp add
        2 index 2 index 2.19 sub dup mul exch 3.32 sub dup mul add -3.2 mul 2.718281828 exch exp exch div
        exch pop exch pop
        dup -0.68 mul 0.98 add exch
        dup -0.64 mul 0.98 add exch
        -0.28 mul 1.00 add
    }
    \node[font=\small] at (3.20,4.62) {(a) Anchor points in $\cMx\subset\mathbb{R}^d$};
    \node[font=\small] at (9.90,4.62) {(b) Input-domain weight $\eta_m(\xb)$};

    \path[fill=pouInput!2] (0.50,0.65) .. controls (0.36,1.65) and (0.48,3.20) .. (0.80,3.86)
            .. controls (2.34,4.03) and (4.35,4.00) .. (5.68,3.85)
            .. controls (6.05,2.80) and (5.94,1.53) .. (5.64,0.67)
            .. controls (4.15,0.49) and (1.85,0.44) .. cycle;
    \begin{scope}
        \clip (0.50,0.65) .. controls (0.36,1.65) and (0.48,3.20) .. (0.80,3.86)
            .. controls (2.34,4.03) and (4.35,4.00) .. (5.68,3.85)
            .. controls (6.05,2.80) and (5.94,1.53) .. (5.64,0.67)
            .. controls (4.15,0.49) and (1.85,0.44) .. cycle;
    \foreach \x/\y in {1.17/3.3,2.33/3.43,3.54/3.31,4.75/3.4,0.99/2.16,2.14/2.27,3.32/2.19,4.51/2.29,5.42/2.36,1.38/1.04,2.57/1.1,3.78/1.08,4.98/1.2}
        \draw[covercell] (\x,\y) circle[radius=0.80];
    \end{scope}
    \draw[pouInput!50!black,line width=0.7pt] (0.50,0.65) .. controls (0.36,1.65) and (0.48,3.20) .. (0.80,3.86)
            .. controls (2.34,4.03) and (4.35,4.00) .. (5.68,3.85)
            .. controls (6.05,2.80) and (5.94,1.53) .. (5.64,0.67)
            .. controls (4.15,0.49) and (1.85,0.44) .. cycle;
    \draw[pouInput!70!black,line width=0.95pt] (3.32,2.19) circle[radius=0.80];
    \foreach \x/\y in {1.17/3.3,2.33/3.43,3.54/3.31,4.75/3.4,0.99/2.16,2.14/2.27,3.32/2.19,4.51/2.29,5.42/2.36,1.38/1.04,2.57/1.1,3.78/1.08,4.98/1.2}
        \node[anchorpoint] at (\x,\y) {};
    \node[circle,fill=pouInput!70!black,inner sep=0pt,minimum size=2mm]
        at (3.32,2.19) {};
    \node[note,anchor=east,text=pouInput!70!black,fill=white]
        at (3.16,2.02) {$\zb_m$};
    \draw[pouInput!65!black,line width=0.55pt] (3.32,2.19)--(3.32,1.39);
    \node[note,anchor=west,text=pouInput!65!black,fill=white]
        at (3.39,1.72) {$r_x$};
    \draw[orange!70!black,line width=0.75pt] (3.32,2.19)--(3.79,2.63);
    \node[circle,draw=orange!80!black,fill=orange!75,
        minimum size=2.1mm,inner sep=0pt] at (3.79,2.63) {};
    \node[note,anchor=south west,text=orange!70!black,fill=white]
        at (3.89,2.72) {$\xb$};
    \begin{scope}[xshift=6.7cm]
        \begin{scope}
            \clip (0.50,0.65) .. controls (0.36,1.65) and (0.48,3.20) .. (0.80,3.86)
            .. controls (2.34,4.03) and (4.35,4.00) .. (5.68,3.85)
            .. controls (6.05,2.80) and (5.94,1.53) .. (5.64,0.67)
            .. controls (4.15,0.49) and (1.85,0.44) .. cycle;
            \shade[shading=pouWeight3] (0,0) rectangle (6.4,4.4);
        \end{scope}
        \draw[pouInput!50!black,line width=0.7pt] (0.50,0.65) .. controls (0.36,1.65) and (0.48,3.20) .. (0.80,3.86)
            .. controls (2.34,4.03) and (4.35,4.00) .. (5.68,3.85)
            .. controls (6.05,2.80) and (5.94,1.53) .. (5.64,0.67)
            .. controls (4.15,0.49) and (1.85,0.44) .. cycle;
        \foreach \x/\y in {1.17/3.3,2.33/3.43,3.54/3.31,4.75/3.4,0.99/2.16,2.14/2.27,3.32/2.19,4.51/2.29,5.42/2.36,1.38/1.04,2.57/1.1,3.78/1.08,4.98/1.2}
            \node[circle,draw=pouInput!45!black,fill=white,fill opacity=0.65,
                minimum size=1.3mm,inner sep=0pt,line width=0.35pt] at (\x,\y) {};
        \draw[white,densely dashed,line width=0.8pt]
            (3.32,2.19) circle[radius=0.80];
        \node[circle,draw=white,fill=pouInput!85!black,
            minimum size=2mm,inner sep=0pt,line width=0.6pt] at (3.32,2.19) {};
        \node[note,anchor=east,text=white] at (3.16,2.02) {$\zb_m$};
        \node[circle,draw=white,fill=orange!80,
            minimum size=2.1mm,inner sep=0pt,line width=0.6pt] at (3.79,2.63) {};
        \node[note,anchor=south west,text=orange!70!black,fill=white,
            fill opacity=0.8,text opacity=1] at (3.89,2.72) {$\xb$};
    \end{scope}
    \shade[left color=pouLow,right color=pouHigh]
        (8.75,-0.02) rectangle (10.65,0.12);
    \node[font=\scriptsize,anchor=east,text=black!65] at (8.63,0.05) {low};
    \node[font=\scriptsize,anchor=west,text=black!65] at (10.77,0.05) {high};
    \node[font=\small,text=pouInput!70!black] at (6.55,-0.65)
        {$\displaystyle \bar h_{l,r_x}(\xb)
        =\sum_{m=1}^{C_x}\eta_m(\xb)\,h_l(\zb_m)
        \ \approx\ h_l(\xb)$};
\end{tikzpicture}%
}
\caption{Input-domain Softmax POU. Left: an $r_x$-cover of the input domain by anchor points $\zb_m$, together with a query $\xb$. Right: schematic concentration of the normalized weight $\eta_m(\xb)$ around $\zb_m$, with normalization over all anchor points. The weighted anchor values $h_l(\zb_m)$ are then used to approximate $h_l(\xb)$ for a fixed function $h_l$.}
\label{fig:input-domain-softmax-pou}
\end{figure}

The task-space POU requires evaluating the anchor functions at the query.  Next, we introduce an input-domain POU based on the anchor points to
approximate these point evaluations.

\begin{lemma}[Input-Domain POU: Anchor-Function Approximation]
\label{lemma:inner_pou}
Under Assumptions \ref{assum:func_space} and \ref{assum:input_space},
let $\{h_l\}_{l=1}^{C_f}\subset\cM_f$ be the anchor functions in Lemma
\ref{lemma:outer_pou}. Each $h_l$ is $\alpha$-H\"older continuous on $\cMx$
with the uniform constant $L_f$, where $\alpha\in(0,1]$. Fix any covering
radius $r_x\in(0,1]$.
There exists a finite set of anchor points
$\{\zb_m\}_{m=1}^{C_x}\subset\cMx$ forming an $r_x$-cover of $\cMx$
under the ambient Euclidean norm $\|\cdot\|_2$.
The number of anchor points satisfies \eqref{eq:Cx}.
Set the scaling parameter
\[
    M_x
    :=
    \max\left\{
    1,\,
    \frac{1}{3r_x^2}
    \log\frac{2B_fC_x}{L_fr_x^\alpha}
    \right\}.
\]
For $\xb\in\cMx$ and $m\in[C_x]$, define the input-domain POU weight
\begin{align}
\label{input-pou-weights}
    \eta_m(\xb)
    &:=
    \frac{\exp\left(M_x(r_x^2-\|\xb-\zb_m\|_2^2)\right)}
    {\sum_{m'=1}^{C_x}
    \exp\left(M_x(r_x^2-\|\xb-\zb_{m'}\|_2^2)\right)} =
    \frac{\exp\left(2M_x\langle \xb,\zb_m\rangle-M_x\|\zb_m\|_2^2\right)}
    {\sum_{m'=1}^{C_x}\exp\left(2M_x\langle \xb,\zb_{m'}\rangle-M_x\|\zb_{m'}\|_2^2\right)}.
\end{align}
Then $\eta_m(\xb)\ge0$ and $\sum_{m=1}^{C_x}\eta_m(\xb)=1$.
For $l\in[C_f]$, set
\[
    \bar h_{l,r_x}(\xb):=\sum_{m=1}^{C_x}\eta_m(\xb)h_l(\zb_m).
\]
Then
\begin{equation}
\label{inner-pou-error}
    \|\bar h_{l,r_x}-h_l\|_{L^\infty(\cMx)}
    \le 3L_fr_x^\alpha,\qquad \forall l\in[C_f].
\end{equation}
\end{lemma}

Lemma \ref{lemma:inner_pou} is proved in Appendix \ref{appsub:lemma:inner_pou}.
Figure \ref{fig:input-domain-softmax-pou} shows the analogous construction on
the input domain. The query assigns larger Softmax weights to nearby anchor
points, and the corresponding stored values of each anchor function are
combined to approximate its value at the query.

The population task-identification scores are unavailable from a finite prompt.
Using the input-domain POU approximations, we construct empirical scores from
the random context observations and control their approximation error
uniformly over the anchor functions.

For a context $\frakc=\{(\xb_i,f(\xb_i))\}_{i=1}^n$, define
\begin{equation}
\label{eq:computable_inner_prod}
    \langle f,\bar h_{l,r_x}\rangle_n
    :=\frac1n\sum_{i=1}^n f(\xb_i)\bar h_{l,r_x}(\xb_i)
    =\frac1n\sum_{i=1}^n\sum_{m=1}^{C_x}
    \eta_m(\xb_i)y_i h_l(\zb_m).
\end{equation}
The empirical task-identification weights are
\begin{equation}
\label{beta-empirical}
    \hat\beta_{l,n}(\frakc)
    :=
    \frac{\exp\left(2M_f\langle f,\bar h_{l,r_x}\rangle_n
    -M_f\|h_l\|_{L^2(\rhox)}^2\right)}
    {\sum_{k=1}^{C_f}\exp\left(2M_f\langle f,\bar h_{k,r_x}\rangle_n
    -M_f\|h_k\|_{L^2(\rhox)}^2\right)}.
\end{equation}
These weights depend only on the context. The next lemma bounds the error
in the empirical inner products used to compute them.

\begin{lemma}[Empirical Approximation of Task-Identification Scores]
\label{lemma:emp_inner}
Under Assumptions \ref{assum:func_space} and \ref{assum:input_space},
use the notation of Lemmas \ref{lemma:outer_pou} and
\ref{lemma:inner_pou}, and let $r_f,r_x\in(0,1]$.
For each $f\in\cM_f$ and $\delta\in(0,1)$ satisfying
$r_f\delta\le e^{-1}$, with probability at least $1-\delta$
over $\xb_1,\ldots,\xb_n$,
\begin{equation}
\label{empirical-error-bound}
    \max_{l\in[C_f]}
    \left|\langle f,\bar h_{l,r_x}\rangle_n
    -\langle f,h_l\rangle_{L^2(\rhox)}\right|
    \le3B_fL_fr_x^\alpha
    +B_f^2\sqrt{2\big(d_f+1+\log(2C_{\cM})\big)}\,
    \sqrt{\frac{\log(1/(r_f\delta))}{n}}.
\end{equation}
The constants are independent of $f$; the probability event
may depend on $f$.
\end{lemma}

Lemma \ref{lemma:emp_inner} is proved in Appendix \ref{appsub:lemma:emp_inner}.
The task-identification-and-evaluation oracle is
\begin{equation}
\begin{aligned}
    \hat f_{r_f,r_x,n}(\fraks)
    =\hat f_{r_f,r_x,n}(\frakc,\xb_{n+1})
    &:=\sum_{l=1}^{C_f}\hat\beta_{l,n}(\frakc)
    \bar h_{l,r_x}(\xb_{n+1})\\
    &=\sum_{l=1}^{C_f}\sum_{m=1}^{C_x}
    \hat\beta_{l,n}(\frakc)\eta_m(\xb_{n+1})h_l(\zb_m).
\end{aligned}
\end{equation}
Here $\hat\beta_{l,n}(\frakc)$ identifies the task from the context, while
$\eta_m(\xb_{n+1})$ localizes the query. Their product forms a joint Softmax
weight on the stored anchor values, as shown in the following proposition.

\begin{proposition}[Joint Softmax Representation of the Oracle]
\label{prop:joint_softmax}
Under Assumptions \ref{assum:func_space} and \ref{assum:input_space},
use the notation above. For $l\in[C_f]$ and $m\in[C_x]$, define the
joint Softmax POU weight
\begin{align}
\gamma_{l,m}(\fraks)
&:=
\frac{\exp\left(\Xi_{l,m}(\fraks)\right)}
{\sum_{k=1}^{C_f}\sum_{m'=1}^{C_x}
\exp\left(\Xi_{k,m'}(\fraks)\right)} ,
\end{align}
where the joint logit feature is
\begin{equation}
\label{joint-logit-def}
\begin{aligned}
\Xi_{l,m}(\fraks)
&=
\underbrace{\left(2M_f\langle f,\bar h_{l,r_x}\rangle_n
-M_f\|h_l\|_{L^2(\rhox)}^2\right)}_{\text{Empirical Task Localization Logit}}+
\underbrace{\left(2M_x\langle \xb_{n+1},\zb_m\rangle
-M_x\|\zb_m\|_2^2\right)}_{\text{Query Localization Logit}} .
\end{aligned}
\end{equation}
Then
$\gamma_{l,m}(\fraks)=\hat\beta_{l,n}(\frakc)\eta_m(\xb_{n+1})$, and
the oracle approximant in \eqref{eq:fully_computable} admits the joint Softmax
representation
\begin{equation} \label{eq:fully_computable}
\hat{f}_{r_f,r_x,n}(\fraks)
=
\sum_{l=1}^{C_f}\sum_{m=1}^{C_x}
\gamma_{l,m}(\fraks)h_l(\zb_m).
\end{equation}

Moreover, for each $f\in\cM_f$ and $\delta\in(0,1)$ satisfying
$r_f\delta\le e^{-1}$, with probability at least $1-\delta$
over $\xb_1,\ldots,\xb_n$,
\begin{equation}
\label{joint-estimator-error}
\begin{aligned}
    \|\hat f_{r_f,r_x,n}(\frakc,\cdot)-f\|_{L^2(\rhox)}
    &\le3r_f+3L_f(1+4B_f^2M_f)r_x^\alpha\\
    &\quad+4B_f^3M_f\sqrt{2\big(d_f+1+\log(2C_{\cM})\big)}\,
    \sqrt{\frac{\log(1/(r_f\delta))}{n}}.
\end{aligned}
\end{equation}
\end{proposition}

Proposition \ref{prop:joint_softmax} is proved in Appendix \ref{appsub:prop:joint_softmax}. Proposition \ref{prop:joint_softmax} establishes an oracle estimator \eqref{joint-estimator-error} for the function $f$ given the context $\frakc$, which is naturally aligned with the softmax attention mechanism. We next construct a Transformer network to approximate this oracle estimator.

\subsection{Constructive Approximation by a Standard Transformer}
\label{sec:approximation}

Building on the task-identification-and-evaluation oracle developed in Section
\ref{sec:oracle-approximant}, we establish a constructive Transformer
approximation result. The joint Softmax representation in Proposition
\ref{prop:joint_softmax} serves as the intermediate target realized by our
construction. We further construct a shallow, wide, and dense
Transformer $\rmT^*$ with only three encoder blocks, Softmax attention, sinusoidal
positional encodings, and point-wise single-hidden-layer ReLU FFNs. The network
takes the finite prompt
$\fraks=(\frakc,\xb_{n+1})$ as input and approximates $f$ as a function
of the query, with root mean square $L^2(\rhox)$ error over the context
inputs controlled uniformly over $f\in\cM_f$.

\begin{theorem}[Transformer Approximation for In-Context Regression]
\label{thm:total_approx}
Let Assumptions \ref{assum:func_space} and \ref{assum:input_space}
hold, and fix $n\in\NN$. Let $\epsilon_0\in(0,e^{-1}]$ be the structural constant defined
in \eqref{epsilon0-explicit}. For any $\epsilon\in(0,\epsilon_0]$,
choose task and input covers at the radii $r_f$ and $r_x$ in
\eqref{epsilon-radii-choice}. If their cardinalities satisfy
$C_fC_x+1\ge n+2$, then there exists a
Transformer network
\[
    \rmT^*
    \in
    \mathcal T\Big(
    L,D,P,
    \{H^\ell\}_{\ell=1}^L,
    \{d_k^\ell\}_{\ell=1}^L,
    \{d_v^\ell\}_{\ell=1}^L,
    \{d_{\mathrm{ff}}^\ell\}_{\ell=1}^L,
    M_{\max}
    \Big)
\]
such that
\begin{equation}
\label{total-approx-bound}
    \sup_{f\in\cM_f}
    \EE_{\xb_1,\ldots,\xb_n}
    \|\rmT^*(\frakc,\cdot)-f\|_{L^2(\rhox)}^2
    \le4\epsilon^2
    +4C_{\mathrm{stat}}^2
    \frac{\epsilon^{-4}(\log(1/\epsilon))^3}{n}.
\end{equation}
Here the expectation is over the independent context inputs
$\xb_1,\ldots,\xb_n\sim\rhox$.
The network can be chosen with the following structural parameters:
\begin{itemize}
    \item $L=3$, $D=d+2n+9$, and $P=C_fC_x+1$, where $P
        \le C_P
        \epsilon^{-q}
        \left(\log\frac1\epsilon\right)^{\frac{d_x}{\alpha}}.$
    \item In block $\ell=1$, $H^1=(2n+3)P+3$, $d_k^1=5$,
    $d_v^1=2$, and $d_{\mathrm{ff}}^1=2D$.
    \item In block $\ell=2$, $H^2=nP+3$, $d_k^2=5$,
    $d_v^2=2$, and $d_{\mathrm{ff}}^2=2D$.
    \item In block $\ell=3$, $H^3=1$, $d_k^3=d_v^3=1$, and
    $d_{\mathrm{ff}}^3=2D$.
\end{itemize}
The parameter magnitude and the total number of architectural (dense)
parameters
satisfy
\begin{align}
    M_{\max}
    &\le
    C_{\mathrm{mag}}\epsilon^{-q_M}
    \left(\log\frac1\epsilon\right)^{1+\frac{2}{\alpha}+\frac{4d_x}{\alpha}},
    \label{param-magnitude-bound}\\
    \mathcal N_{\mathrm{total}}
    &\le
    C_N n^2\epsilon^{-q}
    \left(\log\frac1\epsilon\right)^{\frac{d_x}{\alpha}},
    \label{param-cover-bound}
\end{align}
Here
$q:=d_f+\frac{3d_x}{\alpha}$ and
$q_M:=4q+\frac{6}{\alpha}$.
The positive structural constants
$C_{\mathrm{stat}},C_P,C_N$, and $C_{\mathrm{mag}}$ are specified in
\eqref{Cstat-def}, \eqref{CP-def}, \eqref{CN-def}, and
\eqref{Cmag-def}, respectively, in the proof.
\end{theorem}

\begin{figure}[t]
\centering
\resizebox{0.98\textwidth}{!}{%
\begin{tikzpicture}[
    >=Latex,
    flow/.style={-{Latex[length=2mm]},draw=black!58,line width=0.8pt},
    panel/.style={rounded corners=3pt,line width=0.7pt},
    heading/.style={font=\small\bfseries,align=center},
    body/.style={font=\small,align=center,text width=4.30cm},
    terminal/.style={rounded corners=3pt,draw=black!35,fill=black!2,
        line width=0.6pt,inner xsep=8pt,inner ysep=5pt,
        font=\small,align=center}
]
    \draw[panel,draw=blue!48!black,fill=blue!3]
        (0.02,0.95) rectangle (4.90,4.30);
    \draw[panel,draw=teal!58!black,fill=teal!4]
        (5.53,0.95) rectangle (10.41,4.30);
    \draw[panel,draw=violet!48!black,fill=violet!4]
        (11.04,0.95) rectangle (15.92,4.30);
    \node[heading,text=blue!55!black] at (2.46,3.99) {Encoder block 1};
    \node[heading,text=teal!65!black] at (7.97,3.99) {Encoder block 2};
    \node[heading,text=violet!65!black] at (13.48,3.99) {Encoder block 3};

    \node[terminal] (input) at (2.46,4.96) {Input prompt $\fraks$};
    \draw[flow] (input.south)--(2.46,4.33);
    \node[body,font=\footnotesize] at (2.46,3.37) {Context input and query\\ localization affine features};
    \node[body] at (2.46,2.76) {$\langle\xb_i,\zb_m\rangle$};
    \node[body] at (2.46,2.05) {Anchor value features};
    \node[body] at (2.46,1.49) {$y_i h_l(\zb_m),\quad h_l(\zb_m)$};

    \node[body,font=\footnotesize] at (7.97,3.37) {Context input-domain POU};
    \node[body] at (7.97,2.84) {to get $y_i\bar h_{l,r_x}(\xb_i)$};
    \node[body] (assemble) at (7.97,1.91) {Assemble to get\\joint logit feature $\Xi_{l,m}(\fraks)$};
    \draw[flow] (7.97,2.53)--(assemble.north);
    \node[body,font=\footnotesize,text=black!70] at (7.97,1.23) {Retain $h_l(\zb_m)$};

    \node[body,font=\footnotesize] at (13.48,3.37) {Task-space POU along with};
    \node[body] at (13.48,2.84) {query input-domain POU};
    \node[body] at (13.48,2.05) {to get oracle approximation};
    \node[body] at (13.48,1.49) {$\rmT^*(\fraks)\approx\hat f_{r_f,r_x,n}(\fraks)$};
    \node[terminal] (output) at (13.48,0.18) {Output $\rmT^*(\fraks)$};
    \draw[flow] (13.48,0.92)--(output.north);
    \draw[flow] (4.98,2.62)--(5.45,2.62);
    \draw[flow] (10.49,2.62)--(10.96,2.62);
\end{tikzpicture}%
}
\caption[Overview of the three-block Transformer construction]{Overview of the
three-block Transformer approximation of the oracle in \eqref{eq:fully_computable}.}
\label{fig:transformer-oracle-overview}
\end{figure}
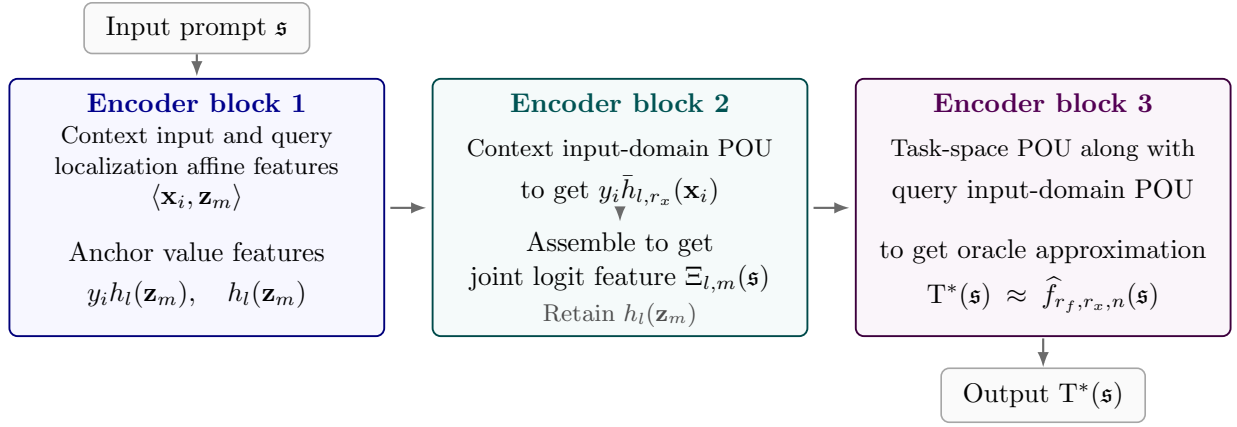

The detailed Transformer construction lemmas are presented in Section \ref{sec:construction-lemmas}, and the proof of Theorem \ref{thm:total_approx} is given in Section~\ref{sec:approximation-proof}. 
Figure~\ref{fig:transformer-oracle-overview} summarizes the three-block Transformer construction in terms
of the two-level POU from Section~\ref{sec:oracle-approximant}.
The first block extracts inner-product-based localization features and anchor
value features. The second block uses the input-domain POU to form
$y_i\bar h_{l,r_x}(\xb_i)$, then assembles to get joint logit features $\Xi_{l,m}(\fraks)$ in
\eqref{joint-logit-def}, while retaining the anchor value features $h_l(\zb_m)$.
The third block approximates the oracle
$\hat f_{r_f,r_x,n}(\fraks)$ in \eqref{eq:fully_computable}, which is a joint Softmax representation by combining the
empirically approximated task-space POU based on the context with a second
input-domain POU for approximating anchor-function values at the query.

Theorem \ref{thm:total_approx} has several important implications:

\textbf{Context sample efficiency from shared cross-task information.}
Theorem~\ref{thm:total_approx} reveals how shared cross-task information can
improve context sample efficiency. The Transformer encodes representative
functions and related population quantities in its parameters, so the $n$
context observations need only localize the target task rather than estimate it
from scratch as an independent regression problem. At any fixed approximation
resolution $\epsilon$, this task-identification step achieves the
dimension-free squared-error rate $\widetilde O(n^{-1})$, while the geometric
complexity of the task family is absorbed into the model representation and
the $\epsilon$-dependent complexity of the construction.

\textbf{Geometric flexibility and architectural alignment.}
The two-level Softmax POU extends this mechanism to general, possibly nonlinear,
task spaces without requiring a prescribed linear basis or global
parametrization. Task-space covering provides representative functions, while
the two POU levels perform task localization and query evaluation from randomly
located context observations, yielding a discretization-free construction.
Because both levels use normalized exponential weighting and aggregation, they
align naturally with Softmax attention in Transformer architecture. Moreover, Theorem~\ref{thm:total_approx} shows that a shallow, wide, and dense Transformer can exploit general, possibly nonlinear, low-dimensional cross-task structure together with the geometry of the input domain, with approximation and parameter-complexity exponents governed by the intrinsic dimensions $d_f$ and $d_x$ rather than by the ambient dimension. This alignment also suggests
an architectural inductive bias for learning nonlinear cross-task structure
during pretraining.


\FloatBarrier

\subsection{Generalization Error Across Tasks}
\label{sec:generalization}

Building on Theorem~\ref{thm:total_approx}, we bound the population risk of
the empirical risk minimizer trained on the meta-training data set
$\frakS$ defined in \eqref{eq:meta-training-data}.
For fixed $n$ and $\epsilon$, denote by $\cT_{n,\epsilon}$ the Transformer
class in Definition \ref{def_transformer_class} obtained by substituting the
structural parameters and the parameter magnitude bound $M_{\max}$ specified
in Theorem \ref{thm:total_approx}. Since the evaluation target satisfies
$
|f(\xb_{n+1})|\le B_f$, we define the clipping operator
\[
    \pi_{B_f}(u):=\min\{B_f,\max\{-B_f,u\}\}.
\]
The corresponding clipped Transformer class is
\[
    \pi_{B_f}\cT_{n,\epsilon}
    :=
    \{\pi_{B_f}\circ\rmT:\rmT\in\cT_{n,\epsilon}\}.
\]
Define its empirical risk minimizer by
\begin{equation}
\label{eq:clipped-erm}
    \widehat{\rmT}_{\frakS}
    \in
    \argmin_{\rmT\in\pi_{B_f}\cT_{n,\epsilon}}
    \cL_{\frakS}(\rmT).
\end{equation}
The following theorem bounds the expected population risk of $\widehat{\rmT}_{\frakS}$ and
separates the effects of the context length $n$ and the number of
meta-training prompts $\Gamma$. Its proof is given in
Section~\ref{sec:generalization-proof}.

\begin{theorem}[Generalization Error of the Empirical Risk Minimizer]
\label{thm:icl_generalization}
Let Assumptions \ref{assum:func_space} and \ref{assum:input_space} hold,
and let $n,\Gamma\ge2$. Let $\epsilon_0$ be a constant specified
in \eqref{epsilon0-explicit}. For any $\epsilon\in(0,\epsilon_0]$ whose covers satisfy
$C_fC_x+1\ge n+2$, let $\widehat{\rmT}_{\frakS}$ be the empirical risk minimizer defined in
\eqref{eq:clipped-erm}. Then
\begin{align}
\label{icl-generalization-bound}
    \EE_{\frakS}\cL(\widehat{\rmT}_{\frakS})
    &\le
    C_{\mathrm{gen}}\left[
    \epsilon^2
    +
    \frac{\epsilon^{-4}
    \left(\log\frac1\epsilon\right)^3}{n}
    +
    \frac{
    n^2\epsilon^{-q}
    \left(\log\frac{n}{\epsilon}\right)^{\frac{d_x}{\alpha}+1}
    }{\Gamma}
    \right],
\end{align}
where
$q:=d_f
+\frac{3d_x}{\alpha}$, and the positive constant $C_{\mathrm{gen}}$ is specified in
\eqref{Cgen-def}. Moreover, choose
\[
    \epsilon
    =
    \max\left\{
    \left(\frac{(\log n)^3}{n}\right)^{
    \frac16},
    \left(
    \frac{n^2(\log n)^{\frac{d_x}{\alpha}+1}}{\Gamma}
    \right)^{\frac{1}{q+2}}
    \right\}.
\]
Provided this choice lies in $(0,\epsilon_0]$ and satisfies the preceding sequence-length condition, we have
\begin{equation}
\label{icl-generalization-optimized}
    \EE_{\frakS}\cL(\widehat{\rmT}_{\frakS})
    \le
    C_{\mathrm{gen}}
    \max\left\{
    \left(\frac{(\log n)^3}{n}\right)^{
    \frac13},
    \left(
    \frac{n^2(\log n)^{\frac{d_x}{\alpha}+1}}{\Gamma}
    \right)^{\frac{2}{q+2}}
    \right\}.
\end{equation}
\end{theorem}

Theorem \ref{thm:icl_generalization} has several important implications:

\textbf{Separation of pretraining and context complexity.}
Theorem~\ref{thm:icl_generalization} gives an end-to-end population-risk
guarantee that separates the roles of the context length $n$ and the number
$\Gamma$ of meta-training prompts. Here, $n$ controls the amount of
task-specific information available within each prompt, whereas $\Gamma$
controls how accurately the shared cross-task structure is learned during
pretraining. Accordingly, the two terms in
\eqref{icl-generalization-optimized} quantify, respectively, within-prompt task
identification and across-task estimation. The latter carries the
dimension-dependent complexity through the intrinsic dimensions $d_f$ and
$d_x$ and therefore requires sufficiently many pretraining tasks. Once
$\Gamma$ is sufficiently large so that the meta-training term is negligible,
the remaining dependence on $n$ has a dimension-independent exponent.

\textbf{Context sample efficiency from cross-task information.}
This dimension-free dependence on the context length highlights the statistical
benefit of exploiting cross-task information. Once the meta-training term is
sufficiently small, Theorem~\ref{thm:icl_generalization} yields
$\widetilde{\mathcal O}(n^{-1/3})$ population risk. By contrast, single-task
nonparametric regression typically exhibits the dimension-dependent rate
$\widetilde{\mathcal O}(n^{-2\alpha/(2\alpha+d_x)})$, a dependence also
appearing in existing Transformer ICL analyses for Besov and H\"older function
classes \citep{kim2024transformers,shen2026understanding,ching2026efficient} (see Table~\ref{tab:icl-rate-comparison} in detail). Thus, information learned across many
tasks during pretraining can reduce the statistical burden placed on each new
context.

\textbf{Implications for few-shot prediction.}
The empirical effectiveness of few-shot prompting in large pretrained language
models \citep{brown2020language} is consistent with the statistical principle
captured by our theory. In our setting, sufficiently rich pretraining allows
shared cross-task information to be encoded in the model, so a short context
can primarily serve to identify the current task and access this shared
information rather than estimate the task from scratch. Here, $\Gamma$ controls
the accuracy with which cross-task information is learned during pretraining,
whereas $n$ controls the task-specific evidence available at inference time.
Although our theory does not model natural language directly, it formalizes a
general mechanism by which learning across many tasks can reduce the amount of
task-specific information required from each new prompt.

\section{Related Work and Discussion}
\label{sec:related-work}

In this section, we discuss existing works that are closely related to this
paper. Several statistical and probabilistic perspectives interpret ICL through
Bayesian inference over latent tasks. \citep{zhang2025bayesian} show that, under
a latent-variable model of prompts, perfectly pretrained models perform
Bayesian model averaging; they further establish its approximate attention-based
implementation for Gaussian linear ICL and decompose the pretraining error into
approximation and generalization components. \citep{han2025kernel}
further show that Bayesian prediction under a generative model of prompts
asymptotically takes a kernel-regression form and empirically observe analogous
behavior in the attention and hidden representations of pretrained language
models. More recently, \citep{wakayama2026bayesian} develops a Bayesian
meta-learning framework that decomposes ICL risk into a Bayes gap and posterior
variance and derives finite-sample bounds for the Bayes gap. Beyond these
formulations, ICL has also been studied from several other perspectives.

\begin{table}[!t]
\centering
\caption{Comparison of generalization error bounds for ICL studied in our work and related works.}
\label{tab:icl-rate-comparison}
\renewcommand{\arraystretch}{1.25}
\setlength{\tabcolsep}{3.5pt}
\resizebox{\textwidth}{!}{%
\begin{tabular}{>{\raggedright\arraybackslash}p{3.0cm}
>{\raggedright\arraybackslash}p{5.35cm}
>{\raggedright\arraybackslash}p{4.55cm}
>{\centering\arraybackslash}m{4.8cm}}
\hline
\textbf{Work}
& \textbf{Task-space assumptions}
& \textbf{Input and observation assumptions}
& \textbf{Generalization error bound} \\
\hline
\textbf{This work}
& A uniformly bounded, uniformly $\alpha$-H\"older task space $\cM_f$ with $\alpha \in (0,1]$, covering complexity
$\mathcal N(r,\cM_f,L^2(\rhox))=\mathcal O(r^{-d_f})$.
& Inputs sampled i.i.d. from an arbitrary distribution supported on a compact
$\cMx\subset[0,1]^d$ with
$\mathcal N(r,\cMx,\|\cdot\|_2)=\mathcal O(r^{-d_x})$; noiseless responses.
& $\displaystyle
\begin{gathered}
\widetilde{\mathcal O}\!\left(
n^{-\frac13}
+\left(\frac{n^2}{\Gamma}\right)^{\frac{2}{q+2}}
\right)\\[2pt]
q=d_f
+\frac{3d_x}{\alpha}
\end{gathered}$ \\
\hline
Kim et al.~\citep{kim2024transformers}
& A Besov ball $B_{p_0,q_0}^\alpha([0,1]^d)$ with $\alpha>d/p_0, p_0 \geq 2$; under a B-spline wavelet expansion, the task coefficients are centered, independent, and satisfy a prescribed scale-dependent variance decay.
& Inputs sampled i.i.d. from a distribution with density bounded above and
below on $[0,1]^d$;
bounded, mean-zero observation noise.
& $\displaystyle
\widetilde{\mathcal O}\!\left(
n^{-\frac{2\alpha}{2\alpha+d}}
+\frac{n^{\frac{2d}{2\alpha+d}}}{\Gamma}
\right)$ \\
\hline
Shen et al.~\citep{shen2026understanding}
& Uniformly bounded $\alpha$-H\"older functions with $\alpha \in (0,1]$, defined on a compact
$d$-dimensional Riemannian manifold with positive reach.
& Inputs sampled i.i.d. from the uniform distribution on the manifold; noiseless responses.
& $\displaystyle
\widetilde{\mathcal O}\!\left(
n^{-\frac{2\alpha}{2\alpha+d}} +\frac{n}{\sqrt{\Gamma}}
\right)$ \\
\hline
Ching et al.~\citep{ching2026efficient}
& Tasks drawn from a distribution supported on a uniformly bounded
$\alpha$-H\"older ball on $[0,1]^d$, with a common $\alpha>0$.
& Inputs sampled i.i.d. from a distribution with density bounded above and
below on $[0,1]^d$; bounded observation noise satisfying
$\mathbb E[\varepsilon\mid X]=0$.
& $\displaystyle
\widetilde{\mathcal O}\!\left(
n^{-\frac{2\alpha}{2\alpha+d}}+\frac1\Gamma
\right)$ \\
\hline
Hsu et al.~\citep{hsu2026featurizer}
& A uniformly bounded function class within $L^\infty$
distance $\delta$ of a fixed finite-dimensional polynomial space with bounded coefficients.
& Inputs sampled i.i.d. from a distribution supported on a bounded interval, with uniformly well-conditioned feature covariance; noiseless responses.
& $\displaystyle
\widetilde{\mathcal O}\!\left(
\frac1n+\delta^2+\sqrt{\frac n\Gamma}
\right)$ \\
\hline
\end{tabular}%
}
\vspace{0.5em}
\end{table}

\paragraph{Approximation and generalization perspectives.}
One viewpoint treats ICL as prediction from the empirical distribution of context examples, linking it to distribution and functional regression \citep{shi2025learning,shi2025nonlinear}. \citep{mroueh2023statistical} derives unseen-task guarantees under Wasserstein regularity, while \citep{furuya2025universal} proves universal approximation of Wasserstein-continuous in-context maps for arbitrary context lengths. \citep{liu2026ghost} analyzes linear-Transformer ICL as a map from context distributions to response functions under a two-stage domain-generalization model, with approximation and generalization guarantees. \citep{li2025universal} constructs a Transformer for general task-function classes using shared universal features and in-context estimation of task-specific coefficients.

More concrete works develop explicit approximation constructions together with cross-task statistical guarantees for nonlinear and nonparametric ICL. For Besov and piecewise-smooth task classes, \citep{kim2024transformers} use FFN feature extractors for finite-dimensional basis approximation and linear attention for task-specific regression, establishing upper bounds and information-theoretic minimax lower bounds. \citep{shen2026understanding} connect Softmax attention with Nadaraya--Watson regression for H\"older functions on manifolds and attain the intrinsic-dimensional minimax rate. \citep{ching2026efficient} realize local-polynomial regression through kernel-weighted polynomial features and gradient descent, attaining the minimax H\"older rate with only $\Theta(\log n)$ parameters. For tasks supported near a prescribed finite-dimensional polynomial or spline space, \citep{hsu2026featurizer} construct the corresponding features directly through attention and derive finite-sample generalization bounds.
These concrete cross-task guarantees exploit a common regularity class, a prescribed regression rule, or a prescribed finite-dimensional representation, whereas our framework allows general, possibly nonlinear, low-dimensional task geometry characterized by intrinsic covering conditions. Table~\ref{tab:icl-rate-comparison} compares the finite-sample generalization guarantees most directly related to our nonlinear in-context regression setting.

\paragraph{Algorithmic perspectives.}
An influential perspective views ICL as algorithm learning: a Transformer uses its forward pass to implement a learning procedure on the in-context examples. Early experiments showed that Transformers trained over simple function classes can infer unseen linear and nonlinear functions from context \citep{garg2022can}. Motivated by this phenomenon, \citep{akyurek2023what} and \citep{vonoswald2023transformers} related Transformer forward passes to least-squares estimation and gradient-descent updates, while \citep{dai2023gpt} interpreted GPT-based ICL as implicit finetuning through attention-generated meta-gradients. For in-context linear regression, \citep{ahn2023transformers} showed that trained linear Transformers can implement preconditioned gradient descent, whereas \citep{fu2024transformers} provided empirical and constructive evidence that Transformers can realize iterative Newton-type second-order methods. This viewpoint was extended by \citep{bai2023transformers}, who construct Transformers that implement and select among a broad class of statistical learning algorithms, and by \citep{li2023transformers}, who relate ICL generalization to the stability of the implemented algorithm for both i.i.d. prompts and dynamical trajectories, including transfer to unseen tasks. Beyond linear models, Transformers have been shown to implement approximate gradient descent on neural-network parameters or functional gradient descent in function space, enabling nonlinear function learning in context \citep{wang2024incontext,cheng2024functional}.

\paragraph{Optimization perspectives.}
A separate line of work studies how gradient-based optimization gives rise to in-context behavior. For linear-regression tasks, \citep{zhang2024trained} proves global convergence of gradient flow for a single linear self-attention layer, while \citep{wu2024pretraining} analyzes online SGD pretraining with independent linear-regression tasks and quantifies the number of tasks needed for a single-layer linear-attention model to approach Bayes-optimal and ridge-regression performance. For looped linear Transformers, \citep{gatmiry2024looped} shows that the population-loss minimizer implements data-adaptive multi-step preconditioned gradient descent and proves fast convergence of gradient flow to this algorithmic solution. Moving to Softmax attention, \citep{huang2024convergence} establishes finite-time convergence of gradient-descent training to near-zero prediction error for a one-layer model under structured balanced and imbalanced feature distributions. \citep{he2025demystified} further shows that multi-head Softmax attention trained from random initialization develops structured attention patterns that approximately implement a debiased gradient-descent predictor. Under a fixed finite-dimensional task representation, \citep{yang2024representations} proves linear convergence of gradient descent and shows that a trained multi-head Softmax Transformer performs ridge regression over the shared basis. Beyond fixed linear representations, \citep{kim2024nonlinear} analyzes the nonconvex population landscape and Wasserstein gradient-flow dynamics of an MLP feature extractor followed by linear attention in mean-field and two-timescale limits. For prescribed nonlinear task families, \citep{li2024nonlinear} establishes in-domain and structured out-of-domain guarantees for SGD-trained Softmax Transformers, while \citep{oko2024pretrained} shows that gradient descent can learn a shared low-dimensional subspace of single-index tasks, yielding ICL sample complexity governed by intrinsic rather than ambient task dimension. 

\FloatBarrier
\section{Proof of Main Results}
\label{sec:proof-techniques}

This section provides the proofs of the two main theorems. Section
\ref{sec:proof-theorem1} first states the five constructive lemmas that are used to
approximate the oracle in Section \ref{sec:oracle-approximant} with a three-block
Transformer and then proves Theorem \ref{thm:total_approx}. Section \ref{sec:proof-theorem2} states
the two statistical lemmas used in the generalization analysis and then proves
Theorem \ref{thm:icl_generalization}.

\subsection{Transformer Construction and Proof of Theorem \ref{thm:total_approx}}
\label{sec:proof-theorem1}

Section~\ref{sec:construction-lemmas} presents the lemmas used to construct
the Transformer. Section~\ref{sec:approximation-proof} combines these lemmas
with the oracle approximation result to prove Theorem~\ref{thm:total_approx}.

\subsubsection{Transformer Construction Lemmas}
\label{sec:construction-lemmas}

Figure~\ref{fig:two-level-softmax-pou-oracle} illustrates our construction in more detail. The construction uses $L=3$ encoder blocks, embedding dimension
$D=d+2n+9$, and sequence length $P=C_fC_x+1$.
We first summarize the intermediate features and operations, then state five constructive
lemmas covering prompt preprocessing, feature extraction, input-domain
Softmax POU, joint-logit assembly, and final joint Softmax POU and readout. Their proofs are given in
Appendix~\ref{app:transformer-construction}.

Figure~\ref{fig:two-level-softmax-pou-oracle} illustrates how the two-level
Softmax POU oracle from Section~\ref{sec:oracle-approximant} is approximated by a Transformer with
three encoder blocks. For clarity, we first describe the corresponding ideal
features; the actual network constructs approximations of these quantities, as
indicated by the tildes and overbars in the figure.

Starting from the preprocessed prompt $\bZ_0=\cP(\fraks)$ achieved by Lemma \ref{lemma:preprocessing}, the first block
(Lemma~\ref{lemma_1stblock}) constructs approximations of the ideal
input-localization logits
$U_{i,m}
    :=2M_x\langle\xb_i,\zb_m\rangle-M_x\|\zb_m\|_2^2$, 
the task-anchor penalties
$V_l:=M_f\|h_l\|_{L^2(\rhox)}^2$, and the value features
$y_i h_l(\zb_m)$ and $h_l(\zb_m)$.

The second block MHA layer (Lemma~\ref{lemma_2ndblock}) uses the approximate
input-localization logits to realize an approximation of the input-domain
Softmax POU at each context point. In particular, it constructs approximations of the ideal context contribution
\[
    W_{i,l}
    :=
    \frac{2M_f}{n}y_i
    \sum_{m=1}^{C_x}\eta_m(\xb_i)h_l(\zb_m)
    =
    \frac{2M_f}{n}y_i\bar h_{l,r_x}(\xb_i).
\]
Its FFN layer (Lemma~\ref{lemma_2ndblock_ffn}) then aggregates these approximate context contributions and combines
them with the approximate task-anchor penalty and query-localization logit,
producing approximations of the ideal joint logit feature
\[
    \Xi_{l,m}(\fraks)
    =
    \underbrace{\sum_{i=1}^n W_{i,l}-V_l}_{\text{empirical task-localization logit}}
    +
    \underbrace{U_{n+1,m}}_{\text{query-localization logit}}.
\]
The approximations of anchor value features $h_l(\zb_m)$ are retained from the first block.

Finally, the third block (Lemma~\ref{lemma_final_block}) applies Softmax to the
approximate joint logits $\widetilde\Xi_{l,m}$ and aggregates the approximate
anchor values $h_l(\zb_m)$. The resulting output satisfies
\[
    \rmT^*(\fraks)
    \approx
    \sum_{l,m}\gamma_{l,m}(\fraks)h_l(\zb_m)
    =
    \hat f_{r_f,r_x,n}(\fraks),
\]
thereby approximating the oracle in \eqref{eq:fully_computable}.

\suppressfloats[t]
\begin{figure}[!t]
\centering
\resizebox{0.98\textwidth}{!}{%
\begin{tikzpicture}[
    >=Latex,
    font=\small,
    flow/.style={-{Latex[length=1.8mm]},draw=black!58,line width=0.75pt},
    panel/.style={rounded corners=3pt,line width=0.7pt},
    feature/.style={rectangle,draw=black!35,fill=white,line width=0.45pt,
        text width=3.82cm,minimum width=4.02cm,minimum height=0.43cm,inner sep=2pt,
        align=center,font=\footnotesize},
    operation/.style={rounded corners=2pt,draw=black!45,fill=white,
        text width=3.82cm,minimum height=0.43cm,inner xsep=5pt,inner ysep=3pt,
        align=center,font=\footnotesize},
    minor/.style={font=\scriptsize,text=black!68,align=center},
    operator/.style={circle,draw=black!55,fill=white,line width=0.65pt,
        minimum size=0.46cm,inner sep=0pt,font=\normalsize}
]
    \draw[panel,draw=blue!48!black,fill=blue!3] (0.02,0.02) rectangle (4.90,6.56);
    \draw[panel,draw=teal!58!black,fill=teal!4] (5.53,0.02) rectangle (10.41,6.56);
    \draw[panel,draw=violet!48!black,fill=violet!4] (11.04,0.02) rectangle (15.92,6.56);
    \node[font=\small\bfseries,text=blue!55!black] at (2.46,6.30) {Encoder block 1};
    \node[font=\small\bfseries,text=teal!65!black] at (7.97,6.30) {Encoder block 2};
    \node[font=\small\bfseries,text=violet!65!black] at (13.48,6.30) {Encoder block 3};

    \node[operation] (route) at (2.46,5.67)
        {$\bZ_0\ \longrightarrow$\ \textbf{MHA}: extract};
    \node[operation] (keep1) at (2.46,5.04)
        {\textbf{FFN}: restore routing};
    \draw[flow] (route)--(keep1);
    \node[font=\small] at (2.46,4.56) {$\bZ_1$};
    \node[feature] at (2.46,4.11) {$\widetilde U_{1,l,m}$\quad context input $1$};
    \node[minor,font=\tiny] at (2.46,3.78) {$\vdots$};
    \node[feature] at (2.46,3.45) {$\widetilde U_{n,l,m}$\quad context input $n$};
    \node[feature,fill=teal!7] at (2.46,2.98) {$\widetilde U_{n+1,l,m}$\quad query logit};
    \node[feature] at (2.46,2.51) {$\approx y_1h_l(\zb_m)$\; context value};
    \node[minor,font=\tiny] at (2.46,2.18) {$\vdots$};
    \node[feature] at (2.46,1.85) {$\approx y_nh_l(\zb_m)$\; context value};
    \node[feature,fill=orange!9] at (2.46,1.38) {$\approx h_l(\zb_m)$\quad anchor value};
    \node[feature,fill=blue!6] at (2.46,0.91) {$\widetilde V_{l,m}$\quad anchor penalty};
    \draw[draw=black!50,line width=0.65pt] (0.36,4.35)--(0.24,4.35)--(0.24,0.67)--(0.36,0.67);
    \draw[draw=black!50,line width=0.65pt] (4.56,4.35)--(4.68,4.35)--(4.68,0.67)--(4.56,0.67);
    \node[minor,text width=4.02cm] at (2.46,0.33) {$j=(l-1)C_x+m<P$};

    \node[operation] (soft2) at (7.97,5.62)
        {\textbf{MHA}: $\operatorname{Softmax}$\\[-1pt]input-domain POU};
    \draw[flow] (soft2.south)--(7.97,5.07);
    \node[feature] at (7.97,4.84) {$\widetilde W_{1,l,m}$\quad context term $1$};
    \node[minor,font=\tiny] at (7.97,4.51) {$\vdots$};
    \node[feature] at (7.97,4.18) {$\widetilde W_{n,l,m}$\quad context term $n$};
    \draw[draw=black!50,line width=0.65pt] (5.87,5.08)--(5.75,5.08)--(5.75,3.94)--(5.87,3.94);
    \draw[draw=black!50,line width=0.65pt] (10.07,5.08)--(10.19,5.08)--(10.19,3.94)--(10.07,3.94);
    \node[minor,text width=4.02cm] at (7.97,3.63) {$\overline U_{n+1,l,m},\,\overline V_{l,m}$ retained};
    \node[font=\scriptsize\bfseries] at (6.13,3.12) {FFN};
    \node[operator] (sum2) at (6.84,3.12) {$+$};
    \node[minor] at (6.84,2.71) {$n$ rows};
    \node[minor] at (9.00,2.71) {suppress null logit};
    \node[font=\footnotesize] at (8.58,3.12) {$-\ \overline V_{l,m}\quad+\ \overline U_{n+1,l,m}$};
    \draw[flow] (7.97,2.56)--(7.97,2.18);
    \node[font=\small] at (7.97,1.93) {$\bZ_2$};
    \node[feature,fill=violet!8] at (7.97,1.49) {$\widetilde\Xi_j$\quad joint logit};
    \node[feature,fill=orange!9] at (7.97,1.02) {$\approx h_l(\zb_m)$\quad anchor value};
    \node[feature] at (7.97,0.55) {$1$\quad constant query};
    \draw[draw=black!50,line width=0.65pt] (5.87,1.73)--(5.75,1.73)--(5.75,0.31)--(5.87,0.31);
    \draw[draw=black!50,line width=0.65pt] (10.07,1.73)--(10.19,1.73)--(10.19,0.31)--(10.07,0.31);

    \node[operation] (soft3) at (13.48,5.62)
        {\textbf{MHA}: $\operatorname{Softmax}$\\[-1pt]joint POU};
    \draw[flow] (soft3)--(13.48,5.08);
    \node[font=\small] at (13.48,4.81) {$\widetilde\gamma_t$};
    \node[minor,text width=4.02cm] at (13.48,4.44) {normalize all $P$ columns};
    \node[operator] (multiply3) at (13.48,3.85) {$\times$};
    \node[font=\scriptsize,anchor=east] at (12.91,3.85) {$(\bZ_2)_{2,t}$};
    \draw[flow] (13.48,4.17)--(multiply3.north);
    \draw[flow] (12.96,3.85)--(multiply3.west);
    \node[operator] (sum3) at (13.48,3.11) {$\Sigma$};
    \node[minor,anchor=west] at (13.94,3.11) {$t=1,\ldots,P$};
    \draw[flow] (multiply3)--(sum3);
    \node[operation] (keep3) at (13.48,2.39) {\textbf{FFN}: retain prediction};
    \draw[flow] (sum3)--(keep3);
    \node[font=\small] at (13.48,1.91) {$\bZ_3$};
    \node[feature,fill=violet!8] (predrow) at (13.48,1.47)
        {$\rmT^*(\fraks)\quad\cdots\quad\rmT^*(\fraks)$};
    \node[feature] at (13.48,1.00) {$0\qquad\cdots\qquad0$};
    \draw[draw=black!50,line width=0.65pt] (11.38,1.71)--(11.26,1.71)--(11.26,0.76)--(11.38,0.76);
    \draw[draw=black!50,line width=0.65pt] (15.58,1.71)--(15.70,1.71)--(15.70,0.76)--(15.58,0.76);
    \node[font=\scriptsize] at (13.48,0.34) {readout: select $(\bZ_3)_{1,1}$};

    \draw[flow] (4.98,3.26)--(5.45,3.26);
    \draw[flow] (10.49,3.26)--(10.96,3.26);
\end{tikzpicture}%
}
\caption[Three-block realization of the joint-logit construction]{Three-block Transformer construction for approximating the oracle in
\eqref{eq:fully_computable}. The three panels correspond to the three encoder
blocks, and the arrows indicate the flow of intermediate representations
through the network. The displayed rows highlight the main quantities involved
at each stage of the construction.}
\label{fig:two-level-softmax-pou-oracle}
\end{figure}
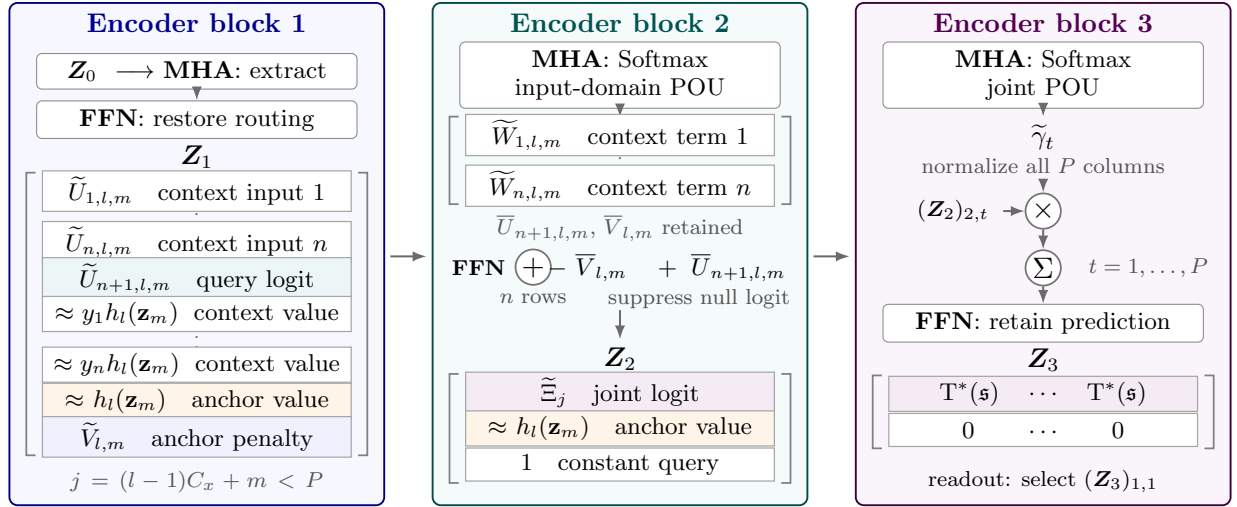

We first construct $\bZ_0$, the input to block 1 in
Figure~\ref{fig:two-level-softmax-pou-oracle}, by combining the prompt and
positional encodings.

\begin{lemma}[Pre-processing Step] \label{lemma:preprocessing}
Let $\fraks=\{(\xb_i,y_i)_{i=1}^n,\xb_{n+1}\}$ be the input prompt,
and assume $P\ge n+2$. Set
\[
    \theta_j=\frac{2\pi j}{P},\quad j\in[P],\qquad
    \phi_\ell=\frac{2\pi\ell}{C_f},\quad \ell\in[C_f].
\]
There exists a pre-processing operator $\cP$ such that
$\bZ_0=\cP(\fraks)\in\RR^{D\times P}$ has the form
\begin{equation} \label{eq:Z0_aligned}
\bZ_0 = \begin{bmatrix}
    \xb_1 & \xb_2 & \cdots & \xb_n & \xb_{n+1} & \bm{0} & \cdots & \bm{0} & \bm{0} \\
    y_1 & y_2 & \cdots & y_n & 0 & 0 & \cdots & 0 & 0 \\
    0 & 0 & \cdots & 0 & 0 & 0 & \cdots & 0 & 1 \\
    \multicolumn{9}{c}{\bm{0}_{(2n+1) \times P}} \\
    1 & 1 & \cdots & 1 & 1 & 1 & \cdots & 1 & 1 \\
    1 & 1 & \cdots & 1 & 1 & 1 & \cdots & 1 & 0 \\
    \multicolumn{3}{c}{\underbrace{\sin(\phi_1) \ \cdots \ \sin(\phi_1)}_{C_x}} & \multicolumn{3}{c}{\underbrace{\sin(\phi_2) \ \cdots \ \sin(\phi_2)}_{C_x}} & \cdots & \underbrace{\sin(\phi_{C_f}) \ \cdots \ \sin(\phi_{C_f})}_{C_x} & 0 \\
    \multicolumn{3}{c}{\underbrace{\cos(\phi_1) \ \cdots \ \cos(\phi_1)}_{C_x}} & \multicolumn{3}{c}{\underbrace{\cos(\phi_2) \ \cdots \ \cos(\phi_2)}_{C_x}} & \cdots & \underbrace{\cos(\phi_{C_f}) \ \cdots \ \cos(\phi_{C_f})}_{C_x} & 0 \\
    \sin(\theta_1) & \sin(\theta_2) & \cdots & \cdots & \cdots & \cdots & \cdots & \sin(\theta_{P-1}) & \sin(\theta_P) \\
    \cos(\theta_1) & \cos(\theta_2) & \cdots & \cdots & \cdots & \cdots & \cdots & \cos(\theta_{P-1}) & \cos(\theta_P)
\end{bmatrix}.
\end{equation}
\end{lemma}

The block 1 of Figure~\ref{fig:two-level-softmax-pou-oracle} is realized
by extracting its feature rows with MHA and restoring the positional encodings
with the point-wise FFN.

\begin{lemma}[Parallel Feature Extraction via MHA] \label{lemma_1stblock}
    Let $\bZ_0 \in \RR^{D \times P}$ be defined in Lemma \ref{lemma:preprocessing}, and assume $P\ge n+2$. For every active column $j\in[P-1]$, write
    \[
        \ell_j=\Big\lceil \frac{j}{C_x}\Big\rceil,\qquad
        m_j=j-(\ell_j-1)C_x .
    \]
    Define the target affine feature terms
    \begin{align*}
        U_{i,m} &= 2M_x \langle \xb_i,\zb_m\rangle-M_x\|\zb_m\|_2^2,
        && i\in[n+1],\ m\in[C_x],\\
        Y_{i,\ell,m} &= y_i h_\ell(\zb_m),
        && i\in[n],\ \ell\in[C_f],\ m\in[C_x],\\
        H_{\ell,m} &= h_\ell(\zb_m),
        && \ell\in[C_f],\ m\in[C_x],\\
        V_\ell &= M_f\|h_\ell\|_{L^2(\rhox)}^2,
        && \ell\in[C_f].
    \end{align*}
    Fix a prescribed null level $M_{null}>0$ and set $\xb_0=\bm{0}$.
    For any $M>1$, set
    \[
        c=1-\cos(2\pi/P).
    \]
    We take $M_x,M_f,M_{null}\ge1$.
    There exist a first-block MHA layer
    $\cA_1:\RR^{D\times P}\to\RR^{D\times P}$ and a point-wise FFN
    $\cF_1:\RR^{D\times P}\to\RR^{D\times P}$. The MHA layer has
    \[
        H^1=(2n+3)P+3,\qquad d_k^1=5,\qquad d_v^1=2,
    \]
    and produces
    \[
        \widehat{\bZ}_1=\cA_1(\bZ_0).
    \]
    The FFN output is
    \[
        \bZ_1
        =\cF_1(\widehat{\bZ}_1)
        =(\cF_1\circ\cA_1)(\bZ_0),
    \]
    where $\bZ_1$ has the form
    \begin{equation} \label{Z1hat}
        \bZ_1 = \begin{bmatrix}
        \widetilde{U}_{1,1,1} & \cdots & \widetilde{U}_{1,1,C_x} & \cdots & \widetilde{U}_{1,C_f,1} & \cdots & \widetilde{U}_{1,C_f,C_x} & \widetilde{M}_{null,1} \\
        \vdots & \ddots & \vdots & \ddots & \vdots & \ddots & \vdots & \vdots \\
        \widetilde{U}_{n+1,1,1} & \cdots & \widetilde{U}_{n+1,1,C_x} & \cdots & \widetilde{U}_{n+1,C_f,1} & \cdots & \widetilde{U}_{n+1,C_f,C_x} & \widetilde{M}_{null,n+1} \\
        \widetilde{Y}_{1,1,1} & \cdots & \widetilde{Y}_{1,1,C_x} & \cdots & \widetilde{Y}_{1,C_f,1} & \cdots & \widetilde{Y}_{1,C_f,C_x} & \widetilde{0}_{n+2} \\
        \vdots & \ddots & \vdots & \ddots & \vdots & \ddots & \vdots & \vdots \\
        \widetilde{Y}_{n,1,1} & \cdots & \widetilde{Y}_{n,1,C_x} & \cdots & \widetilde{Y}_{n,C_f,1} & \cdots & \widetilde{Y}_{n,C_f,C_x} & \widetilde{0}_{2n+1} \\
        \widetilde{H}_{1,1} & \cdots & \widetilde{H}_{1,C_x} & \cdots & \widetilde{H}_{C_f,1} & \cdots & \widetilde{H}_{C_f,C_x} & \widetilde{0}_{2n+2} \\
        \widetilde{V}_{1,1} & \cdots & \widetilde{V}_{1,C_x} & \cdots & \widetilde{V}_{C_f,1} & \cdots & \widetilde{V}_{C_f,C_x} & \widetilde{0}_{2n+3} \\
        \multicolumn{8}{c}{\bm{0}_{(D - 2n - 9) \times P}}  \\
        1 & \cdots & 1 & \cdots & 1 & \cdots & 1 & 1 \\
        1 & \cdots & 1 & \cdots & 1 & \cdots & 1 & 0 \\
        \widetilde{\sin}(\phi_1) & \cdots & \widetilde{\sin}(\phi_1) & \cdots & \widetilde{\sin}(\phi_{C_f}) & \cdots & \widetilde{\sin}(\phi_{C_f}) & \widetilde{0}_{D-3} \\
        \widetilde{\cos}(\phi_1) & \cdots & \widetilde{\cos}(\phi_1) & \cdots & \widetilde{\cos}(\phi_{C_f}) & \cdots & \widetilde{\cos}(\phi_{C_f}) & \widetilde{0}_{D-2} \\
        \sin(\theta_1) & \cdots & \sin(\theta_{C_x}) & \cdots & \sin(\theta_{P-C_x}) & \cdots & \sin(\theta_{P-1}) & \sin(\theta_P) \\
        \cos(\theta_1) & \cdots & \cos(\theta_{C_x}) & \cdots & \cos(\theta_{P-C_x}) & \cdots & \cos(\theta_{P-1}) & \cos(\theta_P)
        \end{bmatrix}.
    \end{equation}
    The output components satisfy:
    \begin{enumerate}
        \item \textbf{Feature outputs.} For all $j\in[P-1]$,
        \begin{align*}
            |\widetilde{U}_{i,\ell_j,m_j}-U_{i,m_j}|
            &\le 2P^2\max\{M_{null},dM_x\}e^{-M},\quad i\in[n+1],\\
            |\widetilde{M}_{null,i}-M_{null}|
            &\le 2P^2\max\{M_{null},dM_x\}e^{-M},\quad i\in[n+1],\\
            |\widetilde{Y}_{i,\ell_j,m_j}-Y_{i,\ell_j,m_j}|
            &\le 2P^2B_f^2e^{-M},\quad i\in[n],\\
            |\widetilde{0}_{n+1+i}|
            &\le 2P^2B_f^2e^{-M},\quad i\in[n],\\
            |\widetilde{H}_{\ell_j,m_j}-H_{\ell_j,m_j}|
            &\le 2P^2B_fe^{-M},\\
            |\widetilde{0}_{2n+2}|
            &\le 2P^2B_fe^{-M},\\
            |\widetilde{V}_{\ell_j,m_j}-V_{\ell_j}|
            &\le 2P^2M_fB_f^2e^{-M},\\
            |\widetilde{0}_{2n+3}|
            &\le 2P^2M_fB_f^2e^{-M}.
        \end{align*}
        \item \textbf{Outer positional outputs.} For all $j\in[P-1]$,
        \begin{align*}
            |\widetilde{\sin}(\phi_{\ell_j})-\sin(\phi_{\ell_j})|
            &\le 2P e^{-M},\\
            |\widetilde{\cos}(\phi_{\ell_j})-\cos(\phi_{\ell_j})|
            &\le 2P e^{-M},\\
            |\widetilde{0}_{D-3}|
            &\le P e^{-M},\\
            |\widetilde{0}_{D-2}|
            &\le P e^{-M}.
        \end{align*}
        \item \textbf{Parameter bounds.}
        Define the parameter magnitudes of the first-block MHA and FFN layers by
        \begin{align*}
            M_{\cA_1}
            &:=
            \max\left\{
            \|\bW_1^O\|_{\max},
            \max_{h\in[H^1]}
            \max\left\{
            \|\bQ_1^h\|_{\max},
            \|\bK_1^h\|_{\max},
            \|\bV_1^h\|_{\max}
            \right\}
            \right\},\\
            M_{\cF_1}
            &:=
            \max\left\{
            \|\bW_1^1\|_{\max},
            \|\bb_1^1\|_\infty,
            \|\bW_1^2\|_{\max},
            \|\bb_1^2\|_\infty
            \right\}.
        \end{align*}
        Whenever $P e^{-M}\le1/4$,
        \[
            \max\{M_{\cA_1},M_{\cF_1}\}
            \le C_1
            \max\Big\{M_{null},\,M_x,\,M_f,\,\frac{M}{c}\Big\},
        \]
        where $C_1:=2d+2B_f+2B_f^2+5$.
    \end{enumerate}
\end{lemma}

The MHA step in block 2 of Figure~\ref{fig:two-level-softmax-pou-oracle}
applies the input-domain Softmax POU to each context input, producing the
context-contribution rows $\widetilde W_{i,\ell,m}$ while carrying forward
the query-localization and anchor features.

\begin{lemma}[Second block MHA]\label{lemma_2ndblock}
Let $\bZ_1$ be the output constructed in Lemma \ref{lemma_1stblock},
with $P\ge n+2$.
For $j\in[P-1]$, write
\[
    \ell_j=\Big\lceil \frac{j}{C_x}\Big\rceil,\qquad
    m_j=j-(\ell_j-1)C_x .
\]
Let
\[
    W_{i,\ell}:=\frac{2M_f}{n}y_i\bar h_{\ell,r_x}(\xb_i)
    =\frac{2M_f}{n}\sum_{m=1}^{C_x}\eta_m(\xb_i)y_i h_\ell(\zb_m)
    =\frac{2M_f}{n}\sum_{m=1}^{C_x}\eta_m(\xb_i)Y_{i,\ell,m},
    \qquad i\in[n],\ \ell\in[C_f].
\]
Let
\[
    c:=1-\cos(2\pi/P)=2\sin^2(\pi/P).
\]
The elementary sine bounds give
\begin{equation}
\label{c-P-bound}
    \frac{8}{P^2}\le c\le\frac{2\pi^2}{P^2}.
\end{equation}
Set
\[
    M_1:=\frac{2dM_x+M+3}{c},\qquad
    M_2:=\frac{2(dM_x+2M_1+M+2)}{c},
\]
and choose the first-block null level as
\[
    M_{null}:=M_2\Big(1-\frac c2\Big).
\]
We take $M_x,M_f,M_1,M_2\ge1$. Then $M_2\ge P$.
Set
\[
    a_M:=\frac{(P-1)(e^M-e^{-M})}{(P-1)e^M+1}.
\]
There exists a second-block MHA layer
$\cA_2:\RR^{D\times P}\to\RR^{D\times P}$ with
$H^2=nP+3$, $d_k^2=5$, and $d_v^2=2$. Its output is
\[
    \widehat{\bZ}_2=\cA_2(\bZ_1),
\]
where $\widehat{\bZ}_2$ has the form
\begin{equation}
\label{Z2hat}
    \widehat{\bZ}_2=
    \begin{bmatrix}
        \widetilde W_{1,1,1}&\cdots&\widetilde W_{1,1,C_x}&\cdots&\widetilde W_{1,C_f,1}&\cdots&\widetilde W_{1,C_f,C_x}&\overline{0}_1\\
        \vdots&\ddots&\vdots&\ddots&\vdots&\ddots&\vdots&\vdots\\
        \widetilde W_{n,1,1}&\cdots&\widetilde W_{n,1,C_x}&\cdots&\widetilde W_{n,C_f,1}&\cdots&\widetilde W_{n,C_f,C_x}&\overline{0}_n\\
        \overline U_{n+1,1,1}&\cdots&\overline U_{n+1,1,C_x}&\cdots&\overline U_{n+1,C_f,1}&\cdots&\overline U_{n+1,C_f,C_x}&\overline M_{null}\\
        \overline H_{1,1}&\cdots&\overline H_{1,C_x}&\cdots&\overline H_{C_f,1}&\cdots&\overline H_{C_f,C_x}&\overline{0}_H\\
        \overline V_{1,1}&\cdots&\overline V_{1,C_x}&\cdots&\overline V_{C_f,1}&\cdots&\overline V_{C_f,C_x}&\overline{0}_V\\
        1&\cdots&1&\cdots&1&\cdots&1&1\\
        a_M&\cdots&a_M&\cdots&a_M&\cdots&a_M&0\\
        \multicolumn{8}{c}{\bm{0}_{(D-n-5)\times P}}
    \end{bmatrix}.
\end{equation}
Assume $M$ is large enough that
\[
    12M_2P^2e^{-M}\le1,
    \qquad
    Pe^{-M}\le\frac14 .
\]
Then, for all $i\in[n]$ and $j\in[P-1]$,
\begin{align*}
    |\widetilde W_{i,\ell_j,m_j}-W_{i,\ell_j}|
    &\le \frac{42B_f^2}{n}M_fM_2P^2e^{-M},\\
    |\overline{0}_i|
    &\le \frac{8B_f^2}{n}M_fP^3e^{-M},\\
    |\overline U_{n+1,\ell_j,m_j}-U_{n+1,m_j}|
    &\le 3M_2P^2e^{-M},\\
    |\overline M_{null}-M_{null}|
    &\le 3M_2P^2e^{-M},\\
    |\overline H_{\ell_j,m_j}-H_{\ell_j,m_j}|
    &\le 5B_fP^2e^{-M},\\
    |\overline{0}_H|
    &\le 4B_fP^2e^{-M},\\
    |\overline V_{\ell_j,m_j}-V_{\ell_j}|
    &\le 5B_f^2M_fP^2e^{-M},\\
    |\overline{0}_V|
    &\le 4B_f^2M_fP^2e^{-M}.
\end{align*}
Moreover, define the parameter magnitude of the second-block MHA layer by
\[
    M_{\cA_2}
    :=
    \max\left\{
    \|\bW_2^O\|_{\max},
    \max_{h\in[H^2]}
    \max\left\{
    \|\bQ_2^h\|_{\max},
    \|\bK_2^h\|_{\max},
    \|\bV_2^h\|_{\max}
    \right\}
    \right\}.
\]
Then its parameters satisfy
\[
    M_{\cA_2}
    \le 2\max\{M_f,M_2\}.
\]
\end{lemma}

The FFN step in block 2 aggregates these rows into the joint-logit row of
$\bZ_2$ shown in Figure~\ref{fig:two-level-softmax-pou-oracle}, and preserves the
anchor values.

\begin{lemma}[Second block FFN]\label{lemma_2ndblock_ffn}
Under the notation and assumptions of Lemma \ref{lemma_2ndblock}, the following construction holds.
For $j\in[P-1]$, the joint logit $\Xi_{\ell_j,m_j}(\fraks)$ defined in
\eqref{joint-logit-def} can equivalently be written as
\begin{equation}
    \Xi_{\ell_j,m_j}(\fraks)
    =\sum_{i=1}^n W_{i,\ell_j}-V_{\ell_j}+U_{n+1,m_j}.
\end{equation}
There exists a point-wise FFN layer
$\cF_2:\RR^{D\times P}\to\RR^{D\times P}$ in the second encoder block such that
$ \bZ_2=\cF_2(\widehat{\bZ}_2)$ is
\begin{equation}
\label{Z2}
    \bZ_2=
    \begin{bmatrix}
        \widetilde\Xi_1&\cdots&\widetilde\Xi_{P-1}&\widetilde\Xi_P\\
        \overline H_{1,1}&\cdots&\overline H_{C_f,C_x}&\overline{0}_H\\
        1&\cdots&1&1\\
        \multicolumn{4}{c}{\bm{0}_{(D-3)\times P}}
    \end{bmatrix}.
\end{equation}
For all $j\in[P-1]$,
\begin{align*}
    |\widetilde\Xi_j-\Xi_{\ell_j,m_j}(\fraks)|
    &\le (3+47B_f^2)M_fM_2P^2e^{-M},\\
    |\overline H_{\ell_j,m_j}-H_{\ell_j,m_j}|
    &\le 5B_fP^2e^{-M}.
\end{align*}
For the null column,
\begin{align*}
    \widetilde\Xi_P
    -\max_{j\in[P-1]}\widetilde\Xi_j
    &\le -M,\\
    |\overline{0}_H|
    &\le 4B_fP^2e^{-M}.
\end{align*}
Moreover, defining
\[
    M_{\cF_2}
    :=
    \max\left\{
    \|\bW_2^1\|_{\max},\|\bb_2^1\|_\infty,
    \|\bW_2^2\|_{\max},\|\bb_2^2\|_\infty
    \right\},
\]
the FFN parameters satisfy
\[
    M_{\cF_2}
    \le
    (6+16B_f^2)\max\{M_f,M_2\}.
\]
\end{lemma}

The block 3 of Figure~\ref{fig:two-level-softmax-pou-oracle} completes
the construction by applying Softmax to the approximate joint logits
$\widetilde\Xi_{l,m}$ and aggregating the approximate anchor values to produce
$\rmT^*(\fraks)\approx\hat f_{r_f,r_x,n}(\fraks)$.

\begin{lemma}[Final encoder block and readout]\label{lemma_final_block}
Under the notation and assumptions of Lemma \ref{lemma_2ndblock_ffn}, the following construction holds.
There exist a third-block MHA layer
$\cA_3:\RR^{D\times P}\to\RR^{D\times P}$, a point-wise FFN
$\cF_3:\RR^{D\times P}\to\RR^{D\times P}$, and a readout vector
$\bc_4\in\RR^{DP}$ such that
\[
    \widehat{\bZ}_3=\cA_3(\bZ_2),
    \qquad
    \bZ_3=\cF_3(\widehat{\bZ}_3).
\]
The resulting scalar network output is
\[
    \rmT^*(\fraks)
    :=\bc_4^\top\mathrm{vec}(\bZ_3).
\]
It satisfies
\begin{equation}
\label{final-network-error}
    \begin{aligned}
    \left|\rmT^*(\fraks)-\hat f_{r_f,r_x,n}(\fraks)\right|
    &\le
    (13+94B_f^2)B_fM_fM_2P^2e^{-M}.
    \end{aligned}
\end{equation}
\end{lemma}

\subsubsection{Proof of Theorem \ref{thm:total_approx}} \label{sec:approximation-proof}

Together, Lemmas \ref{lemma:preprocessing}--\ref{lemma_final_block} approximate
the two-level Softmax POU oracle in Proposition \ref{prop:joint_softmax} with a
three-block Transformer. It remains to combine the oracle and Transformer approximation
errors, choose the covering radii and Softmax scales, and verify the resulting
architectural and parameter bounds. This yields the proof of Theorem
\ref{thm:total_approx}.

\begin{proof}[Proof of Theorem \ref{thm:total_approx}]
Set
\begin{equation}
\label{radius-constants-def}
    A_x:=\max\{12L_f(1+4B_f^2C_{M_f}),1\},
\end{equation}
where
\begin{equation}
\label{CMf-def}
    C_{M_f}
    :=\max\left\{1,\,
    48\left[
    \max\{\log(2B_fC_{\cM}12^{d_f+1}),0\}+d_f+1
    \right]\right\}.
\end{equation}
We take
\begin{equation}
\label{epsilon0-explicit}
    \epsilon_0
    :=\min\left\{
    e^{-1},\,24B_fe^{-1},\,
    \left(\frac{2B_fA_x}{eL_f}\right)^{1/3}
    \right\}.
\end{equation}

Choose
\begin{equation}
\label{epsilon-radii-choice}
    r_f=\frac{\epsilon}{12},\qquad
    r_x=A_x^{-1/\alpha}\epsilon^{3/\alpha}
    \left(\log\frac1\epsilon\right)^{-1/\alpha}.
\end{equation}
Apply Lemmas \ref{lemma:outer_pou} and \ref{lemma:inner_pou} with these radii,
and set $P=C_fC_x+1$ and $D=d+2n+9$.
By hypothesis, $P\ge n+2$.
The definition of $\epsilon_0$ implies
$\log(1/\epsilon)\ge1$, $r_f\le e^{-1}$, and $r_x\le1$.
Moreover,
\begin{align*}
    \frac{2B_fC_f}{r_f}
    &=\frac{24B_fC_f}{\epsilon}\ge e, \qquad
    \frac{2B_fC_x}{L_fr_x^\alpha}
    =\frac{2B_fA_xC_x}{L_f}\epsilon^{-3}\log\frac1\epsilon
    \ge e.
\end{align*}
Since $r_f\le e^{-1}$ and $r_x^\alpha\le e^{-3}$, the second entries
in the max-definitions of $M_f$ and $M_x$ exceed $1$. The covering bounds give
\begin{align}
    M_f
    &=\frac{1}{3r_f^2}\log\frac{2B_fC_f}{r_f}
    \le C_{M_f}\epsilon^{-2}\log\frac1\epsilon,
    \label{Mf-eps-bound}\\
    M_x
    &=\frac{1}{3r_x^2}\log\frac{2B_fC_x}{L_fr_x^\alpha}
    \le C_{M_x}\epsilon^{-6/\alpha}
    \left(\log\frac1\epsilon\right)^{1+2/\alpha},
    \label{Mx-eps-bound}
\end{align}
where
\begin{equation}
\label{CMx-def}
    C_{M_x}
    :=\max\left\{1,\,
    \frac{A_x^{2/\alpha}}{3}
    \left[
    \max\left\{\log\frac{2B_fC_2}{L_f},0\right\}
    +\frac{d_x+\alpha}{\alpha}(\log A_x+4)
    \right]\right\}.
\end{equation}

\textbf{Step 1: Oracle and empirical approximation error.}
Fix $f\in\cM_f$ and $\delta\in(0,1)$. Since $r_f\delta\le e^{-1}$,
Proposition \ref{prop:joint_softmax} gives, with probability at least
$1-\delta$ over the context inputs,
\begin{equation}
\label{oracle-stat-bound}
\begin{aligned}
    \|\hat f_{r_f,r_x,n}(\frakc,\cdot)-f\|_{L^2(\rhox)}
    &\le3r_f+3L_f(1+4B_f^2M_f)r_x^\alpha\\
    &\quad+4B_f^3M_f\sqrt{2\big(d_f+1+\log(2C_{\cM})\big)}\,
    \sqrt{\frac{\log(1/(r_f\delta))}{n}}.
\end{aligned}
\end{equation}
The chosen radii satisfy
\[
    3r_f=\frac{\epsilon}{4},\qquad
    3L_f(1+4B_f^2M_f)r_x^\alpha\le\frac{\epsilon}{4},
\]
where we use \eqref{radius-constants-def}, \eqref{epsilon-radii-choice},
\eqref{Mf-eps-bound}, and $\epsilon^2/\log(1/\epsilon)\le1$.
Also,
\[
    \log\frac1{r_f\delta}
    =\log12+\log\frac1\epsilon+\log\frac1\delta
    \le C_{\log}\left(\log\frac1\epsilon+\log\frac1\delta\right),
\]
where we use $r_f=\epsilon/12$, $\log(1/\epsilon)\ge1$, and set
\begin{equation}
\label{Clog-def}
    C_{\log}:=1+\log12.
\end{equation}
Substituting these estimates into \eqref{oracle-stat-bound} yields
\begin{equation}
\label{oracle-final-bound}
    \|\hat f_{r_f,r_x,n}(\frakc,\cdot)-f\|_{L^2(\rhox)}
    \le\frac{\epsilon}{2}
    +C_{\mathrm{stat}}\epsilon^{-2}\log\frac1\epsilon
    \sqrt{\frac{\log(1/\epsilon)+\log(1/\delta)}{n}},
\end{equation}
where
\begin{equation}
\label{Cstat-def}
    C_{\mathrm{stat}}:=4B_f^3C_{M_f}
    \sqrt{2\big(d_f+1+\log(2C_{\cM})\big)C_{\log}}.
\end{equation}

\textbf{Step 2: Transformer implementation error.}
The covering bounds and \eqref{epsilon-radii-choice} give
\begin{equation}
\label{P-eps-bound}
    P\le C_P\epsilon^{-q}
    \left(\log\frac1\epsilon\right)^{d_x/\alpha},
\end{equation}
where
\begin{equation}
\label{CP-def}
    C_P:=1+C_{\cM}C_212^{d_f}A_x^{d_x/\alpha}.
\end{equation}
Choose
\begin{equation}
\label{M-choice}
    M:=\log\!\left(
    \frac{C_{\mathrm{net}}C_{\mathrm{aux}}
    (1+M_f)(1+M_x)P^6\log(1/\epsilon)}{\epsilon}
    \right),
\end{equation}
with
\begin{equation}
\label{Cnet-def}
    C_{\mathrm{net}}:=\max\left\{
    (13+94B_f^2)B_f(d+1),\,3(d+1)e^{-1},\,4\right\}
\end{equation}
and
\begin{equation}
\label{Caux-def}
    C_{\mathrm{aux}}:=2\left[
    \log\!\left(C_{\mathrm{net}}(1+C_{M_f})(1+C_{M_x})C_P^6\right)
    +6+\frac8\alpha+6q+\frac{6d_x}{\alpha}\right].
\end{equation}
Then
\[
    M\le\log C_{\mathrm{aux}}
    +\frac{C_{\mathrm{aux}}}{2}\log\frac1\epsilon
    \le C_{\mathrm{aux}}\log\frac1\epsilon,
\]
where we use \eqref{Mf-eps-bound}, \eqref{Mx-eps-bound},
\eqref{P-eps-bound}, $\log\log(1/\epsilon)\le\log(1/\epsilon)$,
and $\log C_{\mathrm{aux}}\le C_{\mathrm{aux}}/2$.
As in Lemma~\ref{lemma_2ndblock}, set $c=1-\cos(2\pi/P)$ and
\[
    M_1=\frac{2dM_x+M+3}{c},\qquad
    M_2=\frac{2(dM_x+2M_1+M+2)}{c},\qquad
    M_{null}=M_2\left(1-\frac c2\right).
\]
The routing scales satisfy
\[
    M_1\le\frac{2dM_x+M+3}{8}P^2,\qquad
    M_2\le\frac{2dM_x+M+3}{8}P^4,\qquad M_{null}\le M_2,
\]
where we use $c^{-1}\le P^2/8$ from \eqref{c-P-bound} and $P\ge2$.
Together with \eqref{Mx-eps-bound}, \eqref{P-eps-bound}, and the bound on $M$,
this gives
\begin{equation}
\label{M12-eps-bound}
\begin{aligned}
    M_1&\le\frac{2dC_{M_x}+C_{\mathrm{aux}}+3}{8}C_P^2
    \epsilon^{-6/\alpha-2q}
    \left(\log\frac1\epsilon\right)^{1+2/\alpha+2d_x/\alpha},\\
    M_2&\le\frac{2dC_{M_x}+C_{\mathrm{aux}}+3}{8}C_P^4
    \epsilon^{-6/\alpha-4q}
    \left(\log\frac1\epsilon\right)^{1+2/\alpha+4d_x/\alpha}.
\end{aligned}
\end{equation}
The choice \eqref{M-choice} also gives
\[
    M_2P^2e^{-M}
    \le\frac{\epsilon(2dM_x+M+3)}
    {8C_{\mathrm{net}}C_{\mathrm{aux}}(1+M_f)(1+M_x)\log(1/\epsilon)}
    \le\frac{(d+1)\epsilon}{4C_{\mathrm{net}}(1+M_f)},
\]
where we use $M\le C_{\mathrm{aux}}\log(1/\epsilon)$ and
$C_{\mathrm{aux}},\log(1/\epsilon),M_x\ge1$. Consequently,
\[
    12M_2P^2e^{-M}\le\frac{3(d+1)\epsilon}{C_{\mathrm{net}}}\le1,
    \qquad Pe^{-M}\le\frac{\epsilon}{C_{\mathrm{net}}}\le\frac14,
\]
where we use \eqref{Cnet-def} and $\epsilon\le e^{-1}$.
Thus the hypotheses of Lemmas~\ref{lemma_1stblock} and~\ref{lemma_2ndblock}
hold. Lemmas~\ref{lemma:preprocessing}--\ref{lemma_final_block} therefore
give a three-block Transformer satisfying, for every prompt,
\[
    |\rmT^*(\fraks)-\hat f_{r_f,r_x,n}(\fraks)|
    \le(13+94B_f^2)B_fM_fM_2P^2e^{-M}\le\frac{\epsilon}{4},
\]
where the last inequality follows from the preceding bound on
$M_2P^2e^{-M}$ and \eqref{Cnet-def}.

\textbf{Step 3: Mean-square approximation error.}
The network constructed in Step 2 is independent of $f$ and $\delta$.
Combining its uniform implementation bound with
\eqref{oracle-final-bound} gives, for each $f\in\cM_f$ and every
$\delta\in(0,1)$,
\[
    \PP_{\xb_1,\ldots,\xb_n}\left\{
    \|\rmT^*(\frakc,\cdot)-f\|_{L^2(\rhox)}
    >\epsilon+C_{\mathrm{stat}}\epsilon^{-2}\log\frac1\epsilon
    \sqrt{\frac{\log(1/\epsilon)+\log(1/\delta)}{n}}
    \right\}\le\delta.
\]
For a nonnegative random variable $Z$, the quantile formula gives
\[
    \EE Z^2=\int_0^1 Q_Z(u)^2\,du,
    \qquad
    Q_Z(u):=\inf\{t\ge0:\PP(Z\le t)\ge u\},
    \quad u\in(0,1).
\]
Applying this to $Z=\|\rmT^*(\frakc,\cdot)-f\|_{L^2(\rhox)}$,
with $u=1-\delta$, yields
\begin{align*}
    &\EE_{\xb_1,\ldots,\xb_n}
    \|\rmT^*(\frakc,\cdot)-f\|_{L^2(\rhox)}^2
    =\int_0^1 Q_Z(1-\delta)^2\,d\delta\\
    &\quad\le\int_0^1\left(
    \epsilon+C_{\mathrm{stat}}\epsilon^{-2}\log\frac1\epsilon
    \sqrt{\frac{\log(1/\epsilon)+\log(1/\delta)}{n}}
    \right)^2\,d\delta\\
    &\quad\le2\epsilon^2+
    \frac{2C_{\mathrm{stat}}^2\epsilon^{-4}(\log(1/\epsilon))^2}{n}
    \int_0^1\left(\log\frac1\epsilon+\log\frac1\delta\right)\,d\delta\\
    &\quad\le4\epsilon^2+4C_{\mathrm{stat}}^2
    \frac{\epsilon^{-4}(\log(1/\epsilon))^3}{n},
\end{align*}
where the first inequality follows from the preceding tail bound,
the second uses $(a+b)^2\le2a^2+2b^2$, and the last uses
$\int_0^1\log(1/\delta)\,d\delta=1$ and $\log(1/\epsilon)\ge1$.
Taking the supremum over $f\in\cM_f$ proves \eqref{total-approx-bound}.

\textbf{Step 4: Parameter complexity.}
We count all dense entries, including the fixed preprocessing parameters.
The preprocessing and readout contribute $D(P+d+2)$ and $DP$ parameters;
the three MHA layers contribute $14D((2n+3)P+3)$, $14D(nP+3)$, and $4D$,
respectively; each FFN contributes $4D^2+3D$. Hence
\begin{align}
\label{dense-param-count}
    \mathcal N_{\mathrm{total}}
    &=D(P+d+2)+DP+14D((3n+3)P+6)+4D+3(4D^2+3D)\notag\\
    &=(42n+44)DP+12D^2+(d+99)D.
\end{align}
Consequently,
\[
    \mathcal N_{\mathrm{total}}
    \le(d+11)(13d+317)n^2P
    \le C_Nn^2\epsilon^{-q}\left(\log\frac1\epsilon\right)^{d_x/\alpha},
\]
where we use $D\le(d+11)n$, $42n+44\le86n$, $P\ge1$,
and \eqref{P-eps-bound}, with
\begin{equation}
\label{CN-def}
    C_N:=(d+11)(13d+317)C_P.
\end{equation}
This proves \eqref{param-cover-bound}.

\textbf{Step 5: Parameter magnitude.}
The layerwise bounds in Lemmas~\ref{lemma_1stblock}--\ref{lemma_final_block},
together with the preprocessing bound, give
\[
    M_{\max}\le C_{\mathrm{arch}}
    \max\{M_x,M_f,M_1,M_2,M_{null},M/c\},
\]
where
\begin{equation}
\label{Carch-def}
    C_{\mathrm{arch}}:=\max\{2d+2B_f+2B_f^2+5,\,6+16B_f^2\}.
\end{equation}
Using \eqref{c-P-bound}, \eqref{Mf-eps-bound}, \eqref{Mx-eps-bound},
\eqref{P-eps-bound}, and \eqref{M12-eps-bound} yields
\[
    M_{\max}\le C_{\mathrm{mag}}\epsilon^{-q_M}
    \left(\log\frac1\epsilon\right)^{1+2/\alpha+4d_x/\alpha},
\]
where we use $M\le C_{\mathrm{aux}}\log(1/\epsilon)$,
$6/\alpha\ge2$, $\log(1/\epsilon)\ge1$, and $P\ge2$, and set
\begin{equation}
\label{Cmag-def}
    C_{\mathrm{mag}}:=C_{\mathrm{arch}}C_P^4
    \max\left\{C_{M_x},C_{M_f},\frac{C_{\mathrm{aux}}}{8},
    \frac{2dC_{M_x}+C_{\mathrm{aux}}+3}{8}\right\}.
\end{equation}
This proves \eqref{param-magnitude-bound}.
\end{proof}

\subsection{Statistical Tools and Proof of Theorem \ref{thm:icl_generalization}}
\label{sec:proof-theorem2}

Section~\ref{sec:statistical-tools} presents two statistical lemmas.
Section~\ref{sec:generalization-proof} combines these lemmas with
Theorem~\ref{thm:total_approx} to prove Theorem~\ref{thm:icl_generalization}.

\subsubsection{Statistical Lemma Tools}
\label{sec:statistical-tools}

We use two statistical lemmas, proved in
Appendix~\ref{app:generalization_proofs}. The first lemma bounds the uniform-norm
metric entropy of the clipped three-block Transformer class in terms of its
architectural dimensions, dense parameter count, and parameter magnitude.

\begin{lemma}[Covering Number of the ICL Transformer Class]
\label{lemma:icl_covering_number}
Let $\cT$ be the three-block Transformer class with the architecture in Theorem \ref{thm:total_approx}, sequence length $P$, embedding dimension $D$, total number of architectural (dense) parameters $\mathcal N_{\mathrm{total}}$, and parameter magnitude bounded by $M_{\max}\ge1$. Then for every $\eta\in(0,1]$,
\begin{equation}
\label{icl-covering-bound}
    \log \mathcal N\big(\eta,\pi_{B_f}\cT,\|\cdot\|_\infty\big)
    \le
    \mathcal N_{\mathrm{total}}
    \log\left(
    \frac{C_{\mathrm{cov}}P^{17}D^{55}n^{16}
    M_{\max}^{86}}{\eta}
    \right),
\end{equation}
where
$
    C_{\mathrm{cov}}
    :=3\exp(82)
    \left((d+1)\max\{1,B_f\}+2\right)^{13}.
$
\end{lemma}

Lemma~\ref{lemma:icl_covering_number} is proved in
Appendix~\ref{appsub:lemma:icl_covering_number}. The second lemma establishes an
oracle inequality for the empirical risk minimizer that quantifies the
bias--variance tradeoff: the bias term is controlled by approximation error,
and the variance term by the metric entropy of the hypothesis class.

\begin{lemma}[Oracle Inequality for the Clipped Empirical Risk Minimizer]
\label{lemma:icl_oracle}
Let $\cT$ be a compact class of functions on prompts such that
$|\rmT(\fraks)|\le B_f$ for every $\rmT\in\cT$, and assume that
$|y_{n+1}|\le B_f$ almost surely. Suppose that, for some $A,V\ge1$,
\[
    \log\mathcal N(\eta,\cT,\|\cdot\|_\infty)
    \le A\log\frac{V}{\eta},
    \qquad 0<\eta\le1.
\]
Let
\[
    \widehat{\rmT}_{\frakS}\in\argmin_{\rmT\in\cT} \cL_{\frakS}(\rmT) .
\]
Here
$\frakS:=\{(\fraks^\gamma,y_{n+1}^\gamma)\}_{\gamma=1}^{\Gamma}$
denotes the meta-training dataset.

If $R\in(0,2B_f]$ satisfies
$\inf_{\rmT\in\cT}\cL(\rmT)\le R^2$,
then
\begin{equation}
\label{icl-oracle-expectation}
    \EE_{\frakS}\cL(\widehat{\rmT}_{\frakS})
    \le
    640R^2
    +640B_f^2
    \frac{A\log(640(1+B_f)V/R)+1}{\Gamma}.
\end{equation}
\end{lemma}

Lemma \ref{lemma:icl_oracle} is proved in Appendix \ref{appsub:lemma:icl_oracle}. 

\subsubsection{Proof of Theorem \ref{thm:icl_generalization}} \label{sec:generalization-proof}
Combining Lemma \ref{lemma:icl_covering_number} and Lemma \ref{lemma:icl_oracle} with Theorem \ref{thm:total_approx} gives the desired population
risk bound. We then optimize the approximation resolution $\epsilon$ to obtain the
rates on $n$ and $\Gamma$.

\begin{proof}[Proof of Theorem \ref{thm:icl_generalization}]
Define $R_{n,\epsilon}>0$ by
\[
    R_{n,\epsilon}^2
    :=4\epsilon^2+4C_{\mathrm{stat}}^2
    \frac{\epsilon^{-4}(\log(1/\epsilon))^3}{n}.
\]
Theorem~\ref{thm:total_approx} gives a Transformer
$\rmT^*\in\cT_{n,\epsilon}$. Since $|f|\le B_f$, clipping does not increase
the squared prediction error. Thus, by \eqref{population-risk},
\begin{align*}
    \inf_{\rmT\in\pi_{B_f}\cT_{n,\epsilon}}\cL(\rmT)
    &\le\cL(\pi_{B_f}\rmT^*)\le\cL(\rmT^*)\\
    &=\EE_{f\sim\rho_f}\EE_{\xb_1,\ldots,\xb_n}
    \|\rmT^*(\frakc,\cdot)-f\|_{L^2(\rhox)}^2\\
    &\le\sup_{f\in\cM_f}\EE_{\xb_1,\ldots,\xb_n}
    \|\rmT^*(\frakc,\cdot)-f\|_{L^2(\rhox)}^2
    \le R_{n,\epsilon}^2.
\end{align*}

For the application of Lemma~\ref{lemma:icl_oracle}, set
\[
    A=\mathcal N_{\mathrm{total}},\qquad
    V=C_{\mathrm{cov}}P^{17}D^{55}n^{16}M_{\max}^{86}.
\]
Lemma~\ref{lemma:icl_covering_number} gives the required entropy bound
with these $A$ and $V$.
If $R_{n,\epsilon}\le2B_f$, apply Lemma \ref{lemma:icl_oracle} with
$R=R_{n,\epsilon}$. If $R_{n,\epsilon}>2B_f$, then the clipped loss is
at most $4B_f^2<R_{n,\epsilon}^2$. Thus, in either case, since
$R_{n,\epsilon}\ge2\epsilon$,
\[
    \EE_{\frakS}\cL(\widehat{\rmT}_{\frakS})
    \le
    640R_{n,\epsilon}^2
    +640B_f^2
    \frac{\mathcal N_{\mathrm{total}}
    \log(640(1+B_f)V/\epsilon)+1}{\Gamma}.
\]
The architectural bounds in Theorem~\ref{thm:total_approx} imply
\begin{align*}
    \log\frac{640(1+B_f)V}{\epsilon}
    &\le C_V\log\frac n\epsilon,\\
    \mathcal N_{\mathrm{total}}\log\frac{640(1+B_f)V}{\epsilon}+1
    &\le(C_NC_V+1)n^2\epsilon^{-q}
    \left(\log\frac n\epsilon\right)^{d_x/\alpha+1},
\end{align*}
where we use $D\le(d+7)n$,
$\log\log(1/\epsilon)\le\log(1/\epsilon)\le\log(n/\epsilon)$,
and $\log(n/\epsilon),n^2\epsilon^{-q}\ge1$, with
\begin{align*}
    K_V&:=640(1+B_f)C_{\mathrm{cov}}
    C_P^{17}C_{\mathrm{mag}}^{86}(d+7)^{55},\\
    C_V&:=\max\{\log K_V,0\}+72+17q+86q_M+\frac{17d_x}{\alpha}
    +86\left(1+\frac2\alpha+\frac{4d_x}{\alpha}\right).
\end{align*}
Substituting these estimates into the risk bound gives
\begin{align*}
    \EE_{\frakS}\cL(\widehat{\rmT}_{\frakS})
    &\le2560\left[\epsilon^2+C_{\mathrm{stat}}^2
    \frac{\epsilon^{-4}(\log(1/\epsilon))^3}{n}\right]
    +640B_f^2(C_NC_V+1)
    \frac{n^2\epsilon^{-q}(\log(n/\epsilon))^{d_x/\alpha+1}}{\Gamma}\\
    &\le C_{\mathrm{gen}}\left[
    \epsilon^2+\frac{\epsilon^{-4}(\log(1/\epsilon))^3}{n}
    +\frac{n^2\epsilon^{-q}(\log(n/\epsilon))^{d_x/\alpha+1}}{\Gamma}
    \right],
\end{align*}
where we substitute the definition of $R_{n,\epsilon}^2$ and set
\begin{equation}
\label{Cgen-def}
    C_{\mathrm{gen}}:=2560(1+C_{\mathrm{stat}}^2)
    +640B_f^2(C_NC_V+1)\left(\frac76\right)^{d_x/\alpha+1}.
\end{equation}
This proves \eqref{icl-generalization-bound}.

For the optimized statement, take
\[
    \epsilon:=\max\left\{
    \left(\frac{(\log n)^3}{n}\right)^{1/6},
    \left(\frac{n^2(\log n)^{d_x/\alpha+1}}{\Gamma}\right)^{1/(q+2)}
    \right\},
\]
which is admissible by hypothesis. The preceding explicit bound gives
\begin{align*}
    \EE_{\frakS}\cL(\widehat{\rmT}_{\frakS})
    &\le\left[2560\left(1+\frac{C_{\mathrm{stat}}^2}{6^3}\right)
    +640B_f^2(C_NC_V+1)\left(\frac76\right)^{d_x/\alpha+1}\right]\epsilon^2\\
    &\quad\le C_{\mathrm{gen}}\epsilon^2
    =C_{\mathrm{gen}}\max\left\{
    \left(\frac{(\log n)^3}{n}\right)^{1/3},
    \left(\frac{n^2(\log n)^{d_x/\alpha+1}}{\Gamma}\right)^{2/(q+2)}
    \right\},
\end{align*}
where we use $\log(1/\epsilon)\le\frac16\log n$,
$\log(n/\epsilon)\le\frac76\log n$, the two lower bounds defining
$\epsilon$, and \eqref{Cgen-def}.
This proves \eqref{icl-generalization-optimized}.
\end{proof}

\section{Conclusion}
We developed an approximation and generalization theory for Transformer-based
ICL that quantifies how shared cross-task structure can improve context sample
efficiency. We propose a task-space covering to provide a geometric description of general,
possibly nonlinear, task families without requiring an explicit
parametrization, while a dense three-block Transformer can realize a task-identification-and-evaluation method through two-level
Softmax POU to achieve efficient approximation error. Our generalization
bound separates the roles of number of pretraining tasks and context
length: with sufficiently rich pretraining, the dependence on context length
has a dimension-independent exponent. This formalizes how information learned
across tasks and stored in model parameters can reduce the amount of
task-specific information required from each new prompt.
Promising future directions include extending the framework beyond uniform task
regularity, and developing Transformer architectures
and theory that can accommodate varying context lengths within a single model.

\section*{Acknowledgments}
Zhongjie Shi and Wenjing Liao acknowledge support from the National Science Foundation under the NSF DMS 2145167 and the U.S. Department of Energy under the DOE SC0024348. Alex Cloninger acknowledges support from the National Science Foundation under the NSF CISE 2403452, NSF DMS 2608292, and a fellowship from the Simons Foundation. Rongjie Lai acknowledges support from the National Science Foundation under the NSF DMS 2401297. 

\section*{Declaration of AI-assisted technologies in the manuscript preparation process}
During the preparation of this work the authors used  ChatGPT (OpenAI) in order to assist with the aspects of the numerical implementation in Table \ref{tab:full-vs-two}, to plot illustrative figures and to improve the language and readability of the manuscript. After using this tool, the authors reviewed and edited the content as needed and take full responsibility for the content of the published article.

\appendix

\section*{Appendix}

\section{Proofs for the Oracle Approximation Scheme}\label{app:oracle-form}

This section establishes the proof of oracle approximation scheme introduced in
Section~\ref{sec:oracle-approximant}. We prove the task-space and input-domain POU approximation
bounds in Lemmas~\ref{lemma:outer_pou} and \ref{lemma:inner_pou}, derive the
empirical task-identification score bound in
Lemma~\ref{lemma:emp_inner}, and
then prove the error bound of two-level Softmax POU oracle in
Proposition~\ref{prop:joint_softmax}. The arguments use the following Softmax
Lipschitz estimate \citep[Corollary A.7]{edelman2022inductive}.
\begin{lemma} \label{Softmaxlipchitz}
    For any $\btheta, \btheta' \in \RR^d$, we have
    \begin{equation*}
        \left\| \softmax(\btheta)- \softmax(\btheta') \right\|_1 \leq 2 \|\btheta- \btheta'\|_\infty.
    \end{equation*}
\end{lemma}

\subsection{Proof of Lemma \ref{lemma:outer_pou}}
\label{appsub:prooflemma:outer_pou}
We first establish the task-space POU approximation in Lemma \ref{lemma:outer_pou}.  The proof of Lemma \ref{lemma:outer_pou} rewrites the
distance-based task-identification weights in affine form and separates the
contributions of nearby and distant anchor functions.

\begin{proof}[Proof of Lemma \ref{lemma:outer_pou}]
By Assumption \ref{assum:func_space}, choose an $r_f$-cover
$\{h_l\}_{l=1}^{C_f}\subset\cM_f$ in $L^2(\rhox)$ such that
$C_f\le C_{\cM}r_f^{-d_f}$.
Fix $f\in\cM_f$. For every $l\in[C_f]$,
\begin{align*}
    M_f\left(r_f^2-\|f-h_l\|_{L^2(\rhox)}^2\right)
    &=M_f\left(r_f^2-\|f\|_{L^2(\rhox)}^2\right)
    +2M_f\langle f,h_l\rangle_{L^2(\rhox)}
    -M_f\|h_l\|_{L^2(\rhox)}^2.
\end{align*}
Therefore,
\[
    \beta_l(f)
    =\frac{\exp\left(2M_f\langle f,h_l\rangle_{L^2(\rhox)}
    -M_f\|h_l\|_{L^2(\rhox)}^2\right)}
    {\sum_{k=1}^{C_f}\exp\left(2M_f\langle f,h_k\rangle_{L^2(\rhox)}
    -M_f\|h_k\|_{L^2(\rhox)}^2\right)}.
\]
Choose $l^*\in\argmin_{l\in[C_f]}\|f-h_l\|_{L^2(\rhox)}$.
Since $\{h_l\}_{l=1}^{C_f}$ is an $r_f$-cover, we have
$\|f-h_{l^*}\|_{L^2(\rhox)}\le r_f$.
The normalizing denominator in the distance form of $\beta_l(f)$ satisfies
\[
    \sum_{k=1}^{C_f}
    \exp\left(M_f(r_f^2-\|f-h_k\|_{L^2(\rhox)}^2)\right)
    \ge\exp\left(M_f(r_f^2-\|f-h_{l^*}\|_{L^2(\rhox)}^2)\right)
    \ge1.
\]
Let
\[
    \mathcal J_f:=\{l\in[C_f]:\|f-h_l\|_{L^2(\rhox)}\le2r_f\}.
\]
Then
\begin{align*}
    \|\tilde f_{r_f}(f,\cdot)-f\|_{L^2(\rhox)}
    &=\left\|\sum_{l=1}^{C_f}\beta_l(f)(h_l-f)\right\|_{L^2(\rhox)}\\
    &\le\sum_{l\in\mathcal J_f}\beta_l(f)\|h_l-f\|_{L^2(\rhox)}
    +\sum_{l\notin\mathcal J_f}\beta_l(f)\|h_l-f\|_{L^2(\rhox)}.
\end{align*}
For $l\in\mathcal J_f$, $\|h_l-f\|_{L^2(\rhox)}\le2r_f$, so
\[
    \sum_{l\in\mathcal J_f}\beta_l(f)\|h_l-f\|_{L^2(\rhox)}
    \le2r_f\sum_{l\in\mathcal J_f}\beta_l(f)
    \le2r_f.
\]
For $l\notin\mathcal J_f$,
\[
    r_f^2-\|f-h_l\|_{L^2(\rhox)}^2<-3r_f^2,
    \qquad
    \beta_l(f)\le\exp(-3M_fr_f^2).
\]
Since $\|h_l-f\|_{L^2(\rhox)}\le2B_f$,
\[
    \sum_{l\notin\mathcal J_f}\beta_l(f)\|h_l-f\|_{L^2(\rhox)}
    \le2B_fC_f\exp(-3M_fr_f^2)
    \le r_f.
\]
Here the last inequality follows from
\[
    M_f\ge\frac{1}{3r_f^2}\log\frac{2B_fC_f}{r_f}.
\]
Combining the two bounds yields
\[
    \|\tilde f_{r_f}(f,\cdot)-f\|_{L^2(\rhox)}
    \le2r_f+r_f=3r_f.
\]
This proves \eqref{outer-pou-error}.
\end{proof}

\subsection{Proof of Lemma \ref{lemma:inner_pou}}
\label{appsub:lemma:inner_pou}
We next apply the same localization argument on the input domain.

\begin{proof}[Proof of Lemma \ref{lemma:inner_pou}]
By Assumption \ref{assum:input_space}, choose an $r_x$-cover
$\{\zb_m\}_{m=1}^{C_x}\subset\cMx$ under the ambient Euclidean norm
$\|\cdot\|_2$ such that
$
    C_x\le C_2r_x^{-d_x}.
$
Fix $l\in[C_f]$ and $\xb\in\cMx$. For every $m\in[C_x]$,
\begin{align*}
    M_x\left(r_x^2-\|\xb-\zb_m\|_2^2\right)
    &=
    M_x\left(r_x^2-\|\xb\|_2^2\right)
    +2M_x\langle\xb,\zb_m\rangle
    -M_x\|\zb_m\|_2^2.
\end{align*}
Therefore,
\begin{align*}
    \eta_m(\xb)
    &=
    \frac{
    \exp\left(M_x(r_x^2-\|\xb\|_2^2)\right)
    \exp\left(2M_x\langle\xb,\zb_m\rangle-M_x\|\zb_m\|_2^2\right)}
    {\sum_{m'=1}^{C_x}
    \exp\left(M_x(r_x^2-\|\xb\|_2^2)\right)
    \exp\left(2M_x\langle\xb,\zb_{m'}\rangle
    -M_x\|\zb_{m'}\|_2^2\right)}\\
    &=
    \frac{
    \exp\left(2M_x\langle\xb,\zb_m\rangle-M_x\|\zb_m\|_2^2\right)}
    {\sum_{m'=1}^{C_x}
    \exp\left(2M_x\langle\xb,\zb_{m'}\rangle
    -M_x\|\zb_{m'}\|_2^2\right)}.
\end{align*}
Choose
$  m^*\in\argmin_{m\in[C_x]}\|\xb-\zb_m\|_2.
$
Since $\{\zb_m\}_{m=1}^{C_x}$ is an $r_x$-cover under $\|\cdot\|_2$, we have
$    \|\xb-\zb_{m^*}\|_2\le r_x.
$
Since $M_x\ge1$, the normalizing denominator in the distance form of
$\eta_m(\xb)$ satisfies
\[
    \sum_{m'=1}^{C_x}
    \exp\left(M_x(r_x^2-\|\xb-\zb_{m'}\|_2^2)\right)
    \ge
    \exp\left(M_x(r_x^2-\|\xb-\zb_{m^*}\|_2^2)\right)
    \ge1.
\]
Let
\[
    \mathcal J_\xb:=\{m\in[C_x]:\|\xb-\zb_m\|_2\le 2r_x\}.
\]
Then
\begin{align*}
    |\bar h_{l,r_x}(\xb)-h_l(\xb)|
    &=
    \left|\sum_{m=1}^{C_x}\eta_m(\xb)(h_l(\zb_m)-h_l(\xb))\right|\\
    &\le
    \sum_{m\in\mathcal J_\xb}\eta_m(\xb)|h_l(\zb_m)-h_l(\xb)|
    +\sum_{m\notin\mathcal J_\xb}\eta_m(\xb)|h_l(\zb_m)-h_l(\xb)|.
\end{align*}
For $m\in\mathcal J_\xb$,
\[
    |h_l(\zb_m)-h_l(\xb)|
    \le L_f\|\zb_m-\xb\|_2^\alpha
    \le L_f(2r_x)^\alpha
    \le 2L_fr_x^\alpha.
\]
Thus the near part is at most $2L_fr_x^\alpha$. For $m\notin\mathcal J_\xb$,
\[
    r_x^2-\|\xb-\zb_m\|_2^2<-3r_x^2,
    \qquad
    \eta_m(\xb)\le \exp(-3M_xr_x^2).
\]
Since $|h_l(\zb_m)-h_l(\xb)|\le2B_f$,
\[
    \sum_{m\notin\mathcal J_\xb}\eta_m(\xb)|h_l(\zb_m)-h_l(\xb)|
    \le
    2B_fC_x\exp(-3M_xr_x^2)
    \le L_fr_x^\alpha.
\]
Here the last inequality follows from
\[
    M_x\ge
    \frac{1}{3r_x^2}
    \log\frac{2B_fC_x}{L_fr_x^\alpha}.
\]
Combining the two bounds and taking the supremum over $\xb\in\cMx$ yields
\[
    \|\bar h_{l,r_x}-h_l\|_{L^\infty(\cMx)}
    =\sup_{\xb\in\cMx}|\bar h_{l,r_x}(\xb)-h_l(\xb)|
    \le 3L_fr_x^\alpha.
\]
This proves \eqref{inner-pou-error}.
\end{proof}

\subsection{Proof of Lemma \ref{lemma:emp_inner}}
\label{appsub:lemma:emp_inner}
We approximate the population inner products in the task-identification
weights by their empirical counterparts computed from the finite context.
For each fixed task, the approximation error is controlled uniformly over
the anchor functions by combining Hoeffding's inequality with a union bound,
together with the input-domain POU approximation error.

\begin{proof}[Proof of Lemma \ref{lemma:emp_inner}]
Fix $f\in\cM_f$ and $\delta\in(0,1)$ with $r_f\delta\le e^{-1}$. Write
\[
    P_nu=\frac1n\sum_{i=1}^n u(\xb_i),\qquad
    Pu=\int u\,d\rhox.
\]
For each $l\in[C_f]$, the random variables $f(\xb_i)h_l(\xb_i)$
are independent and lie in $[-B_f^2,B_f^2]$. Hoeffding's inequality
and a union bound over $l\in[C_f]$ give, for every $t>0$,
\[
    \PP\left\{\max_{l\in[C_f]}|(P_n-P)(fh_l)|>t\right\}
    \le\sum_{l=1}^{C_f}\PP\left\{|(P_n-P)(fh_l)|>t\right\}
    \le2C_f\exp\left(-\frac{nt^2}{2B_f^4}\right).
\]
Taking $t=B_f^2\sqrt{2\log(2C_f/\delta)/n}$, we obtain, with probability
at least $1-\delta$, for the fixed $f$,
\begin{align*}
    &\max_{l\in[C_f]}
    \left|\langle f,\bar h_{l,r_x}\rangle_n
    -\langle f,h_l\rangle_{L^2(\rhox)}\right|
    =\max_l\left|(P_n-P)(fh_l)
    +P_n\bigl(f(\bar h_{l,r_x}-h_l)\bigr)\right|\\
    &\quad\le\max_l|(P_n-P)(fh_l)|
    +B_f\max_l\|\bar h_{l,r_x}-h_l\|_{L^\infty(\cMx)}\\
    &\quad\le3B_fL_fr_x^\alpha
    +B_f^2\sqrt{\frac{2\log(2C_f/\delta)}{n}}\\
    &\quad\le3B_fL_fr_x^\alpha
    +B_f^2\sqrt{2\big(d_f+1+\log(2C_{\cM})\big)}\,
    \sqrt{\frac{\log(1/(r_f\delta))}{n}},
\end{align*}
where the first inequality uses the triangle inequality and $|f|\le B_f$;
the second uses Lemma~\ref{lemma:inner_pou} and the preceding Hoeffding bound;
and the last uses $C_f\le C_{\cM}r_f^{-d_f}$ from \eqref{eq:Cf} and
$\log(1/(r_f\delta))\ge1$.
This proves \eqref{empirical-error-bound}.
\end{proof}

\subsection{Proof of Proposition \ref{prop:joint_softmax}}
\label{appsub:prop:joint_softmax}
We can now combine task identification, input-domain evaluation, and the
empirical approximation bound. The following proof decomposes the resulting oracle error into the three
error terms controlled by the preceding lemmas.

\begin{proof}[Proof of Proposition \ref{prop:joint_softmax}]
The additive structure of the joint logits gives
$\gamma_{l,m}(\fraks)=\hat\beta_{l,n}(\frakc)\eta_m(\xb_{n+1})$.
Substituting this identity into the definition of the oracle yields the
joint Softmax representation.

For the error bound, for each $f\in\cM_f$ we work on the event of probability
at least $1-\delta$ in Lemma \ref{lemma:emp_inner}.
By Lemma \ref{Softmaxlipchitz},
\[
    \sum_{l=1}^{C_f}|\hat\beta_{l,n}(\frakc)-\beta_l(f)|
    \le4M_f\max_{l\in[C_f]}
    \left|\langle f,\bar h_{l,r_x}\rangle_n
    -\langle f,h_l\rangle_{L^2(\rhox)}\right|.
\]
The weights are constant with respect to the query, so
\begin{align*}
    &\|\hat f_{r_f,r_x,n}(\frakc,\cdot)-f\|_{L^2(\rhox)}={}
    \Biggl\|\sum_{l=1}^{C_f}\Bigl[
    \hat\beta_{l,n}(\frakc)(\bar h_{l,r_x}-h_l)
    +\bigl(\hat\beta_{l,n}(\frakc)-\beta_l(f)\bigr)h_l\Bigr]
    +\sum_{l=1}^{C_f}\beta_l(f)h_l-f\Biggr\|_{L^2(\rhox)}\\
    &\quad\le
    \biggl\|\sum_{l=1}^{C_f}\hat\beta_{l,n}(\frakc)
    (\bar h_{l,r_x}-h_l)\biggr\|_{L^2(\rhox)}
    +\biggl\|\sum_{l=1}^{C_f}
    \bigl(\hat\beta_{l,n}(\frakc)-\beta_l(f)\bigr)h_l\biggr\|_{L^2(\rhox)}
    +\|\tilde f_{r_f}(f,\cdot)-f\|_{L^2(\rhox)}\\
    &\quad\le 3L_fr_x^\alpha+3r_f
    +B_f\sum_{l=1}^{C_f}|\hat\beta_{l,n}(\frakc)-\beta_l(f)|\\
    &\quad\le3r_f+3L_f(1+4B_f^2M_f)r_x^\alpha
    +4B_f^3M_f\sqrt{2\big(d_f+1+\log(2C_{\cM})\big)}\,
    \sqrt{\frac{\log(1/(r_f\delta))}{n}},
\end{align*}
where the first inequality is the triangle inequality in $L^2(\rhox)$;
the second follows from Lemmas~\ref{lemma:outer_pou} and~\ref{lemma:inner_pou},
using $\hat\beta_{l,n}(\frakc)\ge0$,
$\sum_l\hat\beta_{l,n}(\frakc)=1$, $\|h_l\|_{L^2(\rhox)}\le B_f$,
and $\|\cdot\|_{L^2(\rhox)}\le\|\cdot\|_{L^\infty(\cMx)}$
(since $\rhox$ is a probability measure);
and the third combines the preceding Softmax Lipschitz estimate with
Lemma~\ref{lemma:emp_inner}.
This proves \eqref{joint-estimator-error}.
\end{proof}

\section{Proofs of the Transformer Construction Lemmas}\label{app:transformer-construction}

This section proves the construction lemmas used in
Theorem~\ref{thm:total_approx}. Starting from the preprocessed prompt
representation in Lemma~\ref{lemma:preprocessing}, we successively construct
approximations of the features required for task identification and query
localization, input-domain Softmax POU, assemble the joint logits, and task-space Softmax POU, to finally realize the approximation of the two-level Softmax POU oracle. The corresponding steps are carried
out in Lemmas~\ref{lemma_1stblock}, \ref{lemma_2ndblock},
\ref{lemma_2ndblock_ffn}, and \ref{lemma_final_block}, while tracking the
approximation errors needed for the final theorem.

\begin{proof}[Proof of Lemma \ref{lemma:preprocessing}]
Let $\bS \in \RR^{D \times P}$ be the sequence matrix:
$$
\bS = \begin{bmatrix}
    \xb_1 & \xb_2 & \cdots & \xb_n & \xb_{n+1} & \bm{0} & \cdots & \bm{0} & \bm{0} \\
    y_1 & y_2 & \cdots & y_n & 0 & 0 & \cdots & 0 & 0 \\
    \multicolumn{9}{c}{\bm{0}_{(2n+8)\times P}}
\end{bmatrix} \in \RR^{D \times P}.
$$
Let the structural and positional encoding matrix $\bP \in \RR^{D \times P}$ be:
$$
\bP = \begin{bmatrix}
    \multicolumn{9}{c}{\bm{0}_{(d+1)\times P}} \\
    0 & 0 & \cdots & 0 & 0 & 0 & \cdots & 0 & 1 \\
    \multicolumn{9}{c}{\bm{0}_{(2n+1) \times P}} \\
    1 & 1 & \cdots & 1 & 1 & 1 & \cdots & 1 & 1 \\
    1 & 1 & \cdots & 1 & 1 & 1 & \cdots & 1 & 0 \\
    \multicolumn{3}{c}{\underbrace{\sin(\phi_1) \ \cdots \ \sin(\phi_1)}_{C_x}} & \multicolumn{3}{c}{\underbrace{\sin(\phi_2) \ \cdots \ \sin(\phi_2)}_{C_x}} & \cdots & \underbrace{\sin(\phi_{C_f}) \ \cdots \ \sin(\phi_{C_f})}_{C_x} & 0 \\
    \multicolumn{3}{c}{\underbrace{\cos(\phi_1) \ \cdots \ \cos(\phi_1)}_{C_x}} & \multicolumn{3}{c}{\underbrace{\cos(\phi_2) \ \cdots \ \cos(\phi_2)}_{C_x}} & \cdots & \underbrace{\cos(\phi_{C_f}) \ \cdots \ \cos(\phi_{C_f})}_{C_x} & 0 \\
    \sin(\theta_1) & \sin(\theta_2) & \cdots & \cdots & \cdots & \cdots & \cdots & \sin(\theta_{P-1}) & \sin(\theta_P) \\
    \cos(\theta_1) & \cos(\theta_2) & \cdots & \cdots & \cdots & \cdots & \cdots & \cos(\theta_{P-1}) & \cos(\theta_P)
\end{bmatrix}.
$$
Let $\bX(\fraks)\in\RR^{(d+1)\times P}$ be the matrix formed by the
first $d+1$ rows of $\bS$. With
\[
    \bW_E=
    \begin{bmatrix}
        \bI_{d+1}\\
        \bm{0}_{(D-d-1)\times(d+1)}
    \end{bmatrix},
    \qquad
    \bb_E=\bm{0}_D,
\]
the shared token-wise affine embedding gives
$\bS=\bW_E\bX(\fraks)+\bb_E\bm{1}_P^\top$. Adding the fixed
structural-positional matrix $\bP$ therefore yields
$\cP(\fraks)=\bS+\bP=\bZ_0$.
\end{proof}

Starting from this representation, the first MHA layer extracts all required
affine feature rows in parallel. Sinusoidal positional rows localize both the
source prompt columns and the output columns, allowing the extracted features
to be routed to the active anchor-pair tokens.

\begin{proof}[Proof of Lemma \ref{lemma_1stblock}]
    Set
    \[
        M_s=\frac{M}{c},\qquad
        M_\theta=\frac{4M}{c}.
    \]
    Let $\chi_t:=1_{\{t<P\}}$ be row $D-4$ of $\bZ_0$, and let $\nu_t:=1_{\{t=P\}}$ be row $d+2$ of $\bZ_0$.
    Row $D-5$ is the all-one row, while row $D-4$ equals $(1,\ldots,1,0)$; the latter is used in the constant feature terms below so that the value at the null source column is zero.

    \textbf{Step 1: Feature extraction heads ($h\in[(2n+3)P]$).}
    In this step, the head $h$ fixes the output row $r$ and the output column $j$. The indices $t$ and $j'$ are used for rows and columns of the attention matrix.
    For $h=(r-1)P+j$, where $r\in[2n+3]$ and $j\in[P]$, define the source prompt column
    \[
        i_r=
        \begin{cases}
            r, & r\in[n+1],\\
            r-(n+1), & r\in\{n+2,\ldots,2n+1\},\\
            1, & r\in\{2n+2,2n+3\}.
        \end{cases}
    \]
    Set
    \[
        M_0:=M_\theta\Big(1-\frac c2\Big)+M_s.
    \]
    Hence this step uses exactly $(2n+3)P$ heads.
    Choose $\bw_{r,j}^\top\in\RR^{1\times D}$ by
    \[
        \bw_{r,j}^\top=
        \begin{cases}
            2M_x\sum_{a=1}^d(\zb_{m_j})_a\be_a^\top
            -M_x\|\zb_{m_j}\|_2^2\be_{D-4}^\top,
            & r\in[n+1],\ j<P,\\
            M_{null}\be_{D-4}^\top,
            & r\in[n+1],\ j=P,\\
            h_{\ell_j}(\zb_{m_j})\be_{d+1}^\top,
            & r\in\{n+2,\ldots,2n+1\},\ j<P,\\
            \bm{0}_{1\times D},
            & r\in\{n+2,\ldots,2n+1\},\ j=P,\\
            h_{\ell_j}(\zb_{m_j})\be_{D-4}^\top,
            & r=2n+2,\ j<P,\\
            \bm{0}_{1\times D},
            & r=2n+2,\ j=P,\\
            M_f\|h_{\ell_j}\|_{L^2(\rhox)}^2\be_{D-4}^\top,
            & r=2n+3,\ j<P,\\
            \bm{0}_{1\times D},
            & r=2n+3,\ j=P.
        \end{cases}
    \]
    Denote the affine features by
    \[
        F_{r,j}:=
        \bw_{r,j}^\top(\bZ_0)_{:,i_r}.
    \]
    Then
    \[
        F_{r,j}=
        \begin{cases}
            U_{r,m_j}, & r\in[n+1],\ j<P,\\
            M_{null}, & r\in[n+1],\ j=P,\\
            Y_{i_r,\ell_j,m_j}, & r\in\{n+2,\ldots,2n+1\},\ j<P,\\
            0, & r\in\{n+2,\ldots,2n+1\},\ j=P,\\
            H_{\ell_j,m_j}, & r=2n+2,\ j<P,\\
            0, & r=2n+2,\ j=P,\\
            V_{\ell_j}, & r=2n+3,\ j<P,\\
            0, & r=2n+3,\ j=P.
        \end{cases}
    \]
    Thus $F_{r,j}$ is the ideal affine feature that the first MHA layer targets at row $r$ and column $j$; the first $2n+3$ rows of \eqref{Z1hat} are the corresponding attention-generated approximations of $F_{r,j}$. Notice that $\cMx \subset[0,1]^d$ and $\xb_t=\bm{0}$ and $y_t=0$ for $t>n+1$, for $j<P$ we have
    \[
        \left|2M_x\langle \xb_t,\zb_{m_j}\rangle
        -M_x\|\zb_{m_j}\|_2^2\chi_t\right|
        \le M_x\sum_{a=1}^d
        \left|(\zb_{m_j})_a\big(2(\xb_t)_a-(\zb_{m_j})_a\chi_t\big)\right|
        \le dM_x ,
    \]
    Hence, for every $t\in[P]$,
    \[
        |\bw_{r,j}^\top(\bZ_0)_{:,t}|\le
        \begin{cases}
            \max\{M_{null},dM_x\}, & r\le n+1,\\
            B_f^2, & n+2\le r\le 2n+1,\\
            B_f, & r=2n+2,\\
            M_fB_f^2, & r=2n+3,
        \end{cases}
    \]
    We construct the query and key matrices as
    \[
        \bQ_1^h=
        \begin{bmatrix}
            \be_{D-5}^\top\\
            \be_{D-1}^\top\\
            \be_D^\top\\
            M_s\sin(\theta_{i_r})\be_{D-5}^\top\\
            M_s\cos(\theta_{i_r})\be_{D-5}^\top
        \end{bmatrix} \in \RR^{5\times D},
        \qquad
        \bK_1^h=
        \begin{bmatrix}
            M_0\be_{d+2}^\top\\
            M_\theta\sin(\theta_j)\be_{D-4}^\top\\
            M_\theta\cos(\theta_j)\be_{D-4}^\top\\
            \be_{D-1}^\top-\sin(\theta_P)\be_{d+2}^\top\\
            \be_D^\top-\cos(\theta_P)\be_{d+2}^\top
        \end{bmatrix} \in \RR^{5\times D}
    \]
    where $\bQ_1^h,\bK_1^h\in\RR^{5\times D}$.
    It follows that
    \[
        \bq_{j'}^h=(\bQ_1^h\bZ_0)_{:,j'}
        =
        \begin{bmatrix}
            1\\ \sin(\theta_{j'})\\ \cos(\theta_{j'})\\
            M_s\sin(\theta_{i_r})\\
            M_s\cos(\theta_{i_r})
        \end{bmatrix}
        \in\RR^5,
    \]
    and
    \[
        \bk_t^h=(\bK_1^h\bZ_0)_{:,t}
        =
        \begin{bmatrix}
            0\\ M_\theta\sin(\theta_j)\\ M_\theta\cos(\theta_j)\\
            \sin(\theta_t)\\ \cos(\theta_t)
        \end{bmatrix},
        \quad t<P,
        \qquad
        \bk_P^h=(\bK_1^h\bZ_0)_{:,P}
        =
        \begin{bmatrix}
            M_0\\0\\0\\0\\0
        \end{bmatrix}.
    \]
    Thus
    \[
        s_{t,j'}^h=
        \begin{cases}
            M_\theta\cos(\theta_{j'}-\theta_j)
            +M_s\cos(\theta_t-\theta_{i_r}),
            & t<P,\\
            M_0, & t=P.
        \end{cases}
    \]
    The source-column locating term satisfies
    \begin{align*}
        M_s\cos(\theta_t-\theta_{i_r})&=M_s,
        && t=i_r,\\
        M_s\cos(\theta_t-\theta_{i_r})&\le M_s(1-c),
        && t<P,\ t\ne i_r.
    \end{align*}
    The output-column locating term satisfies
    \begin{align*}
        M_\theta\cos(\theta_{j'}-\theta_j)&=M_\theta,
        && j'=j,\\
        M_\theta\cos(\theta_{j'}-\theta_j)&\le M_\theta(1-c),
        && j'\ne j.
    \end{align*}
    Choose value matrix as
    \[
        \bV_1^h=
        \begin{bmatrix}
            \bw_{r,j}^\top\\
            \bm{0}_{1\times D}
        \end{bmatrix}
        \in\RR^{2\times D}.
    \]
    Therefore
    \[
        (\bV_1^h\bZ_0)_{:,t}
        =
        \begin{bmatrix}
            \bw_{r,j}^\top(\bZ_0)_{:,t}\\0
        \end{bmatrix}
        \in\RR^2,
        \qquad
        |(\bV_1^h\bZ_0)_{1,t}|\le
        \begin{cases}
            \max\{M_{null},dM_x\}, & r\le n+1,\\
            B_f^2, & n+2\le r\le 2n+1,\\
            B_f, & r=2n+2,\\
            M_fB_f^2, & r=2n+3.
        \end{cases}
    \]
    We separate the target and non-target attention columns.
    \begin{itemize}
        \item \textbf{Target column $j'=j$.} For $t<P$ and $t\ne i_r$,
        \begin{align*}
            s_{i_r,j}^h-s_{t,j}^h
            &=M_s\big(1-\cos(\theta_t-\theta_{i_r})\big)
            \ge cM_s=M,\\
            s_{i_r,j}^h-s_{P,j}^h
            &=M_\theta+M_s-M_0
            =\frac c2M_\theta=2M.
        \end{align*}
        Therefore
        \[
            \sum_{t\ne i_r}(\bA_1^h)_{t,j}
            \le
            \sum_{t\ne i_r}e^{s_{t,j}^h-s_{i_r,j}^h}
            \le P e^{-M}.
        \]
        Hence
        \begin{align*}
            \left|\left(\bV_1^h\bZ_0\bA_1^h\right)_{1,j}-F_{r,j}\right|
            &=\left|
            \sum_{t=1}^P(\bA_1^h)_{t,j}\bw_{r,j}^\top(\bZ_0)_{:,t}
            -\bw_{r,j}^\top(\bZ_0)_{:,i_r}
            \right|\\
            &=\left|
            \sum_{t\ne i_r}(\bA_1^h)_{t,j}
            \left(\bw_{r,j}^\top(\bZ_0)_{:,t}
            -\bw_{r,j}^\top(\bZ_0)_{:,i_r}\right)
            \right|\\
            &\le
            \begin{cases}
                2P\max\{M_{null},dM_x\}e^{-M}, & r\le n+1,\\
                2P B_f^2 e^{-M}, & n+2\le r\le 2n+1,\\
                2P B_f e^{-M}, & r=2n+2,\\
                2P M_f B_f^2 e^{-M}, & r=2n+3.
            \end{cases}
        \end{align*}

        \item \textbf{Non-target columns $j'\ne j$.} For every $t<P$,
        \[
            s_{P,j'}^h-s_{t,j'}^h
            \ge M_0-\big(M_\theta(1-c)+M_s\big)
            =\frac c2M_\theta=2M.
        \]
        Therefore
        \[
            \sum_{t<P}(\bA_1^h)_{t,j'}
            \le
            \sum_{t<P}e^{s_{t,j'}^h-s_{P,j'}^h}
            \le P e^{-2M}
            \le P e^{-M}.
        \]
        Since $\bw_{r,j}^\top(\bZ_0)_{:,P}=0$, we have
        \begin{align*}
            \left|\left(\bV_1^h\bZ_0\bA_1^h\right)_{1,j'}\right|
            &=
            \left|\sum_{t<P}(\bA_1^h)_{t,j'}
            \bw_{r,j}^\top(\bZ_0)_{:,t}\right|\\
            &
            \le
            \begin{cases}
                P\max\{M_{null},dM_x\}e^{-M}, & r\le n+1,\\
                B_f^2 P e^{-M}, & n+2\le r\le 2n+1,\\
                B_f P e^{-M}, & r=2n+2,\\
                B_f^2 P M_f e^{-M}, & r=2n+3.
            \end{cases}
        \end{align*}
    \end{itemize}

    \textbf{Step 2: Indicator and positional encoding heads ($h\in\{(2n+3)P+1,\ldots,(2n+3)P+3\}$).}
    Let $h_0=(2n+3)P$. The head $h_0+1$ preserves the all-one row and the active-indicator row. We construct query and key matrices as
    \[
        \bQ_1^{h_0+1}=
        \begin{bmatrix}
            \be_{D-4}^\top\\
            \be_{d+2}^\top\\
            \bm{0}_{1\times D}\\
            \bm{0}_{1\times D}\\
            \bm{0}_{1\times D}
        \end{bmatrix} \in\RR^{5\times D},
        \qquad
        \bK_1^{h_0+1}=
        \begin{bmatrix}
            M\be_{D-4}^\top\\
            M\be_{d+2}^\top\\
            \bm{0}_{1\times D}\\
            \bm{0}_{1\times D}\\
            \bm{0}_{1\times D}
        \end{bmatrix} \in\RR^{5\times D}
    \]
    Hence
    \[
        s_{t,j'}^{h_0+1}
        =M\chi_t\chi_{j'}+M\nu_t\nu_{j'}.
    \]
    Define
    \[
        v_t^{\mathrm{ind}}
        :=\chi_t-(P-1)e^{-M}\nu_t.
    \]
    Then set value matrix as
    \[
        \bV_1^{h_0+1}
        =
        \begin{bmatrix}
            \be_{D-5}^\top\\
            \be_{D-4}^\top-(P-1)e^{-M}\be_{d+2}^\top
        \end{bmatrix}
        \in\RR^{2\times D},
        \qquad
        \bV_1^{h_0+1}(\bZ_0)_{:,t}
        =
        \begin{bmatrix}
            1\\
            v_t^{\mathrm{ind}}
        \end{bmatrix},
        \quad t\in[P].
    \]
    For $j'<P$,
    \begin{align*}
        \sum_{t=1}^P(\bA_1^{h_0+1})_{t,j'}v_t^{\mathrm{ind}}
        &=
        \frac{(P-1)e^M-(P-1)e^{-M}}{(P-1)e^M+1}
        =
        \frac{(P-1)(e^M-e^{-M})}{(P-1)e^M+1}: = a_M,
    \end{align*}
    and for $j'=P$,
    \[
        \sum_{t=1}^P(\bA_1^{h_0+1})_{t,P}v_t^{\mathrm{ind}}
        =
        \frac{(P-1)-e^M(P-1)e^{-M}}{(P-1)+e^M}=0.
    \]
    Therefore the head $h_0+1$ outputs
    \[
        \left(\bV_1^{h_0+1}\bZ_0\bA_1^{h_0+1}\right)_{:,j'}
        =
        \begin{cases}
            \begin{bmatrix}1\\a_M\end{bmatrix}, & j'<P,\\[1ex]
            \begin{bmatrix}1\\0\end{bmatrix}, & j'=P,
        \end{cases}
    \]
    The scalar $a_M$ is fixed and independent of the active column $j<P$.
    Moreover,
    \[
        1-a_M
        =\frac{1+(P-1)e^{-M}}{(P-1)e^M+1}
        \le 2e^{-M}.
    \]
    Since $P\ge2$ and $P e^{-M}\le1/4$, we have $e^{-M}\le1/8$ and
    \[
        a_M\ge 1-2e^{-M}\ge\frac34,\qquad
        a_M^{-1}\le2.
    \]

    Let $M_{\mathrm{id}}=M/c$. The heads $h_0+2,h_0+3$ preserve the outer positional rows and the $\theta$-rows. They use
    \[
        \bQ_1^{h_0+2}=\bQ_1^{h_0+3}=
        \begin{bmatrix}
            \bm{0}_{1\times D}\\
            \be_{D-1}^\top\\
            \be_D^\top\\
            \bm{0}_{1\times D}\\
            \bm{0}_{1\times D}
        \end{bmatrix},
        \qquad
        \bK_1^{h_0+2}=\bK_1^{h_0+3}=
        \begin{bmatrix}
            \bm{0}_{1\times D}\\
            M_{\mathrm{id}}\be_{D-1}^\top\\
            M_{\mathrm{id}}\be_D^\top\\
            \bm{0}_{1\times D}\\
            \bm{0}_{1\times D}
        \end{bmatrix}
    \]
    where $\bQ_1^{h_0+2},\bQ_1^{h_0+3},\bK_1^{h_0+2},\bK_1^{h_0+3}\in\RR^{5\times D}$.
    For the head $h_0+2$,
    \[
        s_{t,j'}^{h_0+2}=M_{\mathrm{id}}\cos(\theta_t-\theta_{j'}),
        \qquad
        s_{j',j'}^{h_0+2}-s_{t,j'}^{h_0+2}\ge M,\quad t\ne j'.
    \]
    Consequently,
    \[
        \sum_{t\ne j'}(\bA_1^{h_0+2})_{t,j'}
        \le \sum_{t\ne j'}e^{s_{t,j'}^{h_0+2}-s_{j',j'}^{h_0+2}}
        \le P e^{-M}.
    \]
    The value matrices of these two heads are
    \[
        \bV_1^{h_0+2}=\begin{bmatrix}\be_{D-3}^\top\\ \be_{D-2}^\top\end{bmatrix}\in\RR^{2\times D},
        \qquad
        \bV_1^{h_0+3}=\begin{bmatrix}\be_{D-1}^\top\\ \be_D^\top\end{bmatrix}\in\RR^{2\times D}.
    \]
    For the head $h_0+2$,
    \[
        \bV_1^{h_0+2}(\bZ_0)_{:,t}
        =
        \begin{cases}
            \begin{bmatrix}
                \sin(\phi_{\ell_t})\\
                \cos(\phi_{\ell_t})
            \end{bmatrix}, & t<P,\\[2ex]
            \begin{bmatrix}0\\0\end{bmatrix}, & t=P.
        \end{cases}
    \]
    For active $j'<P$, define
    \[
        \begin{bmatrix}
            \widetilde{\sin}(\phi_{\ell_{j'}})\\
            \widetilde{\cos}(\phi_{\ell_{j'}})
        \end{bmatrix}
        :=
        \left(\bV_1^{h_0+2}\bZ_0\bA_1^{h_0+2}\right)_{:,j'}.
    \]
    Then
    \begin{align*}
        |\widetilde{\sin}(\phi_{\ell_{j'}})-\sin(\phi_{\ell_{j'}})|
        &\le \sum_{t\ne j'}(\bA_1^{h_0+2})_{t,j'}
        |(\bV_1^{h_0+2}\bZ_0)_{1,t}-\sin(\phi_{\ell_{j'}})|
        \le 2P e^{-M},\\
        |\widetilde{\cos}(\phi_{\ell_{j'}})-\cos(\phi_{\ell_{j'}})|
        &\le \sum_{t\ne j'}(\bA_1^{h_0+2})_{t,j'}
        |(\bV_1^{h_0+2}\bZ_0)_{2,t}-\cos(\phi_{\ell_{j'}})|
        \le 2P e^{-M}.
    \end{align*}
    For the null column, define
    \[
        \begin{bmatrix}
            \widetilde{0}_{D-3}\\
            \widetilde{0}_{D-2}
        \end{bmatrix}
        :=
        \left(\bV_1^{h_0+2}\bZ_0\bA_1^{h_0+2}\right)_{:,P}.
    \]
    Since the target value at row $D-3,D-2$ is zero,
    \[
        |\widetilde{0}_{D-3}|\le P e^{-M},\qquad
        |\widetilde{0}_{D-2}|\le P e^{-M}.
    \]
    By the shift-invariant softmax calculation
    \citep[Proof of Lemma 5]{shi2026learning}, at attention column $j'$, the
    head $h_0+3$ gives the last two rows
    \[
        \left(\bV_1^{h_0+3}\bZ_0\bA_1^{h_0+3}\right)_{:,j'}
        =
        \lambda_{\mathrm{id}}
        \begin{bmatrix}\sin(\theta_{j'})\\ \cos(\theta_{j'})\end{bmatrix},
        \qquad
        \lambda_{\mathrm{id}}
        :=
        \frac{\sum_{t=1}^Pe^{M_{\mathrm{id}}\cos(\theta_t)}\cos(\theta_t)}
        {\sum_{t=1}^Pe^{M_{\mathrm{id}}\cos(\theta_t)}},
    \]
    with
    \[
        0\le1-\lambda_{\mathrm{id}}\le2(P-1)e^{-M}.
    \]

    \textbf{Step 3: Output projection and error bounds.}
    For each head $h\in[H^1]$, write
    \[
        \bC^h:=\bV_1^h\bZ_0\bA_1^h\in\RR^{2\times P}.
    \]
    For $r\in[2n+3]$ and $j,q\in[P]$, define
    \[
        G_{r,j,q}
        :=
        \left(\bC^{(r-1)P+j}\right)_{1,q}\in\RR.
    \]
    The $P$ heads assigned to feature row $r$ form the block
    \[
        \bC^{(r)}
        :=
        \begin{bmatrix}
            G_{r,1,1} & G_{r,1,2} & \cdots & G_{r,1,P}\\
            0 & 0 & \cdots & 0\\
            G_{r,2,1} & G_{r,2,2} & \cdots & G_{r,2,P}\\
            0 & 0 & \cdots & 0\\
            \vdots & \vdots & \ddots & \vdots\\
            G_{r,P,1} & G_{r,P,2} & \cdots & G_{r,P,P}\\
            0 & 0 & \cdots & 0
        \end{bmatrix}
        \in\RR^{2P\times P}.
    \]
    The remaining three heads give
    \[
        \bC^{\mathrm{ind}}
        :=
        \begin{bmatrix}
            1 & 1 & \cdots & 1 & 1\\
            a_M & a_M & \cdots & a_M & 0
        \end{bmatrix}
        \in\RR^{2\times P},
    \]
    \[
        \bC^\phi
        :=
        \begin{bmatrix}
            \widetilde{\sin}(\phi_{\ell_1})
            & \widetilde{\sin}(\phi_{\ell_2})
            & \cdots
            & \widetilde{\sin}(\phi_{\ell_{P-1}})
            & \widetilde{0}_{D-3}\\
            \widetilde{\cos}(\phi_{\ell_1})
            & \widetilde{\cos}(\phi_{\ell_2})
            & \cdots
            & \widetilde{\cos}(\phi_{\ell_{P-1}})
            & \widetilde{0}_{D-2}
        \end{bmatrix}
        \in\RR^{2\times P},
    \]
    and
    \[
        \bC^\theta
        :=
        \lambda_{\mathrm{id}}
        \begin{bmatrix}
            \sin(\theta_1) & \sin(\theta_2) & \cdots & \sin(\theta_{P-1}) & \sin(\theta_P)\\
            \cos(\theta_1) & \cos(\theta_2) & \cdots & \cos(\theta_{P-1}) & \cos(\theta_P)
        \end{bmatrix}
        \in\RR^{2\times P}.
    \]
    Concatenating all head outputs gives
    \[
        \bC_1
        =
        \begin{bmatrix}
            \bC^{(1)}\\
            \vdots\\
            \bC^{(2n+3)}\\
            \bC^{\mathrm{ind}}\\
            \bC^\phi\\
            \bC^\theta
        \end{bmatrix}
        \in\RR^{2H^1\times P}.
    \]
    Let
    \[
        \bI_{\mathrm{sum}}
        :=
        \begin{bmatrix}
            1&0&1&0&\cdots&1&0
        \end{bmatrix}
        \in\RR^{1\times 2P},
    \]
    and define
    \[
        \bR_{\mathrm{feat}}
        :=
        \begin{bmatrix}
            \bI_{\mathrm{sum}} & \bm{0}_{1\times 2P} & \cdots & \bm{0}_{1\times 2P}\\
            \bm{0}_{1\times 2P} & \bI_{\mathrm{sum}} & \cdots & \bm{0}_{1\times 2P}\\
            \vdots & \vdots & \ddots & \vdots\\
            \bm{0}_{1\times 2P} & \bm{0}_{1\times 2P} & \cdots & \bI_{\mathrm{sum}}
        \end{bmatrix}
        \in\RR^{(2n+3)\times 2(2n+3)P}.
    \]
    Choose
    \[
        \bW_1^O
        =
        \begin{bmatrix}
            \bR_{\mathrm{feat}} & \bm{0}_{(2n+3)\times 6}\\
            \bm{0}_{(D-2n-9)\times 2(2n+3)P} & \bm{0}_{(D-2n-9)\times 6}\\
            \bm{0}_{6\times 2(2n+3)P} & \bI_6
        \end{bmatrix}
        \in\RR^{D\times 2H^1}.
    \]
    Denote
    \[
        \widehat F_{r,j}:=\sum_{j'=1}^P G_{r,j',j},
    \]
    we have the output of MHA layer in the first encoder block $\widehat{\bZ}_1=\bW_1^O\bC_1$ to be
    \[
        (\widehat{\bZ}_1)_{:,j}
        =
        \begin{bmatrix}
            \widehat F_{1,j}\\
            \vdots\\
            \widehat F_{2n+3,j}\\
            \bm{0}_{D-2n-9}\\
            1\\
            a_M\\
            \widetilde{\sin}(\phi_{\ell_j})\\
            \widetilde{\cos}(\phi_{\ell_j})\\
            \lambda_{\mathrm{id}}\sin(\theta_j)\\
            \lambda_{\mathrm{id}}\cos(\theta_j)
        \end{bmatrix},
        \quad j<P,
        \qquad
        (\widehat{\bZ}_1)_{:,P}
        =
        \begin{bmatrix}
            \widehat F_{1,P}\\
            \vdots\\
            \widehat F_{2n+3,P}\\
            \bm{0}_{D-2n-9}\\
            1\\
            0\\
            \widetilde{0}_{D-3}\\
            \widetilde{0}_{D-2}\\
            \lambda_{\mathrm{id}}\sin(\theta_P)\\
            \lambda_{\mathrm{id}}\cos(\theta_P)
        \end{bmatrix}.
    \]
    For each row $r$ and each column $j$, one feature head contributes the target-column error and the remaining $P-1$ feature heads contribute non-target-column errors. Hence
    \begin{align*}
        |\widehat F_{r,j}-F_{r,j}|
        &=
        \left|G_{r,j,j}-F_{r,j}+\sum_{j'\ne j}G_{r,j',j}\right|\\
        &\le
        |G_{r,j,j}-F_{r,j}|+\sum_{j'\ne j}|G_{r,j',j}|\\
        &\le
        \begin{cases}
            2P^2\max\{M_{null},dM_x\}e^{-M}, & r\le n+1,\\
            2P^2B_f^2e^{-M}, & n+2\le r\le 2n+1,\\
            2P^2B_f e^{-M}, & r=2n+2,\\
            2P^2M_fB_f^2 e^{-M}, & r=2n+3.
        \end{cases}
    \end{align*}
    These are the feature bounds in the statement.

    \textbf{Step 4: Point-wise FFN restoration.}
    Define the diagonal matrix
    \[
        \bLambda_1
        :=\mathrm{diag}\big(
        \underbrace{1,\ldots,1}_{D-5},
        a_M^{-1},
        1,1,
        \lambda_{\mathrm{id}}^{-1},\lambda_{\mathrm{id}}^{-1}
        \big)\in\RR^{D\times D}.
    \]
    Choose the point-wise FFN width $d_{\mathrm{ff}}^1=2D$ and set
    \[
        \bW_1^1=
        \begin{bmatrix}
            \bI_D\\
            -\bI_D
        \end{bmatrix},
        \qquad
        \bb_1^1=\bm{0}_{2D},
        \qquad
        \bW_1^2=
        \begin{bmatrix}
            \bLambda_1&-\bLambda_1
        \end{bmatrix},
        \qquad
        \bb_1^2=\bm{0}_D .
    \]
    For every column $(\widehat{\bZ}_1)_{:,j}$,
    \[
        (\bZ_1)_{:,j}
        =\bW_1^2\sigma(\bW_1^1(\widehat{\bZ}_1)_{:,j}+\bb_1^1)+\bb_1^2
        =\bLambda_1(\widehat{\bZ}_1)_{:,j}.
    \]
    Hence, for every $j\in[P]$,
    \[
        (\bZ_1)_{D-5,j}=1,\qquad
        (\bZ_1)_{D-4,j}=1_{\{j<P\}},
        \qquad
        \begin{bmatrix}
            (\bZ_1)_{D-1,j}\\
            (\bZ_1)_{D,j}
        \end{bmatrix}
        =
        \begin{bmatrix}\sin(\theta_j)\\ \cos(\theta_j)\end{bmatrix}.
    \]
    Thus the FFN output is exactly \eqref{Z1hat}. Under
    $Pe^{-M}\le1/4$,
    \[
        a_M^{-1}\le2,\qquad
        \lambda_{\mathrm{id}}^{-1}\le\frac{1}{1-2(P-1)e^{-M}}\le2.
    \]
    Hence the displayed FFN matrices give $M_{\cF_1}\le2$.
    For the MHA layer, $\|\bW_1^O\|_{\max}\le1$ and
    \[
        |M_0|
        \le M_\theta+M_s
        =\frac{5M}{c}.
    \]
    The displayed head matrices therefore yield
    \[
        M_{\cA_1}
        \le
        \max\left\{
        1,M_{null},2M_x,dM_x,B_f,M_fB_f^2,
        M,\frac{5M}{c}
        \right\}.
    \]
    Since $c\le2$ and $M_x,M_f,M_{null}\ge1$, these estimates imply
    \[
        \max\{M_{\cA_1},M_{\cF_1}\}
        \le C_1
        \max\Big\{M_{null},\,M_x,\,M_f,\,\frac{M}{c}\Big\},
    \]
    where $C_1:=2d+2B_f+2B_f^2+5$.
\end{proof}

The first block has placed the input-localization, context-value,
anchor-value, and anchor-bias features in designated rows. The second-block
MHA layer now combines these rows to approximate the input-domain POU
contribution of every context observation at every active anchor-pair token.

\begin{proof}[Proof of Lemma \ref{lemma_2ndblock} (MHA Layer 2)]
    According to $\bZ_1$, the feature rows $\widetilde U_i$, $\widetilde Y_i$, $\widetilde H$, and $\widetilde V$ are located at rows $i$, $n+1+i$, $2n+2$, and $2n+3$, respectively. The routing construction uses the structural rows $D-5,D-4$, the approximate trunk positional rows $D-3,D-2$, and the exact inner positional rows $D-1,D$. We construct $\cA_2:\RR^{D\times P}\to\RR^{D\times P}$ using $nP+3$ heads.

    \textbf{Step 1: POU Feature Extraction Heads ($h\in[nP]$).}
    Each head $h=(r-1)P+j$ aims at extracting a POU feature to target row $r\in[n]$ and target column $j\in[P]$. For an active source token $t<P$, $\ell_t$ and $m_t$ denote its trunk and intra-trunk indices. We use $M_1$ for source-trunk isolation and $M_2$ for target-column routing. With $d_k^2=5$ and $d_v^2=2$, for $j<P$, define query, key, and value matrices $\bQ_2^h,\bK_2^h\in\RR^{5\times D}$, $\bV_2^h\in\RR^{2\times D}$ as follows
    \begin{align*}
        \bQ_2^h &=
        \begin{bmatrix}
            \be_{D-5}^{\top}\\
            \be_{D-1}^{\top}\\
            \be_D^{\top}\\
            M_1\sin(\phi_{\ell_j})\be_{D-4}^{\top}\\
            M_1\cos(\phi_{\ell_j})\be_{D-4}^{\top}
        \end{bmatrix},&
        \bK_2^h &=
        \begin{bmatrix}
            \be_r^{\top}\\
            M_2\sin(\theta_j)\be_{D-4}^{\top}\\
            M_2\cos(\theta_j)\be_{D-4}^{\top}\\
            \be_{D-3}^{\top}\\
            \be_{D-2}^{\top}
        \end{bmatrix},&
        \bV_2^h &=
        \begin{bmatrix}
            \frac{2M_f}{n}\be_{n+1+r}^{\top}\\
            \bm{0}_{1\times D}
        \end{bmatrix}.
    \end{align*}
    Hence, the query vectors are
    \[
        \bq_{j'}^h:=(\bQ_2^h\bZ_1)_{:,j'}
        =
        \begin{bmatrix}
            1\\ \sin(\theta_{j'})\\ \cos(\theta_{j'})\\ M_1\sin(\phi_{\ell_j})\\ M_1\cos(\phi_{\ell_j})
        \end{bmatrix},
        \quad j'<P,
        \qquad
        \bq_P^h:=(\bQ_2^h\bZ_1)_{:,P}
        =
        \begin{bmatrix}
            1\\ \sin(\theta_P)\\ \cos(\theta_P)\\0\\0
        \end{bmatrix}.
    \]
    The key vectors are
    \[
        \bk_t^h:=(\bK_2^h\bZ_1)_{:,t}
        =
        \begin{bmatrix}
            \widetilde{U}_{r,\ell_t,m_t}\\ M_2\sin(\theta_j)\\ M_2\cos(\theta_j)\\
            \widetilde{\sin}(\phi_{\ell_t})\\ \widetilde{\cos}(\phi_{\ell_t})
        \end{bmatrix},
        \quad t<P,
        \qquad
        \bk_P^h:=(\bK_2^h\bZ_1)_{:,P}
        =
        \begin{bmatrix}
            \widetilde{M}_{null}\\0\\0\\\widetilde{0}_{D-3}\\\widetilde{0}_{D-2}
        \end{bmatrix}.
    \]
    The value vectors are
    \[
        \bv_t^h:=(\bV_2^h\bZ_1)_{:,t}
        =
        \begin{bmatrix}\frac{2M_f}{n}\widetilde{Y}_{r,\ell_t,m_t}\\0\end{bmatrix},
        \quad t<P,
        \qquad
        \bv_P^h:=(\bV_2^h\bZ_1)_{:,P}
        =
        \begin{bmatrix}\frac{2M_f}{n}\widetilde{0}_{n+1+r}\\0\end{bmatrix}.
    \]
    Define
    \[
        \xi_{t,\ell_j}:=\sin(\phi_{\ell_j})(\widetilde{\sin}(\phi_{\ell_t})-\sin(\phi_{\ell_t}))
        +\cos(\phi_{\ell_j})(\widetilde{\cos}(\phi_{\ell_t})-\cos(\phi_{\ell_t})).
    \]
    For active rows $t<P$ and active columns $j'<P$,
    \begin{align*}
        s_{t,j'}^h
        &=(\bk_t^h)^\top\bq_{j'}^h\\
        &=\widetilde U_{r,\ell_t,m_t}
        +M_2\cos(\theta_{j'}-\theta_j)
        +M_1\Big[
        \sin(\phi_{\ell_j})\widetilde{\sin}(\phi_{\ell_t})
        +\cos(\phi_{\ell_j})\widetilde{\cos}(\phi_{\ell_t})
        \Big]\\
        &=\widetilde U_{r,\ell_t,m_t}
        +M_2\cos(\theta_{j'}-\theta_j)
        +M_1\Big[
        \sin(\phi_{\ell_j})\sin(\phi_{\ell_t})
        +\cos(\phi_{\ell_j})\cos(\phi_{\ell_t})
        \Big]
        +M_1\xi_{t,\ell_j}\\
        &=\widetilde U_{r,\ell_t,m_t}
        +M_2\cos(\theta_{j'}-\theta_j)
        +M_1\cos(\phi_{\ell_t}-\phi_{\ell_j})
        +M_1\xi_{t,\ell_j}.
    \end{align*}
    For the null token $t=P$ and active columns $j'<P$,
    \[
        s_{P,j'}^h=\widetilde{M}_{null}
        +M_1\sin(\phi_{\ell_j})\widetilde{0}_{D-3}
        +M_1\cos(\phi_{\ell_j})\widetilde{0}_{D-2}.
    \]
    For the null column $j'=P$,
    \[
        s_{t,P}^h=\widetilde U_{r,\ell_t,m_t}
        +M_2\cos(\theta_P-\theta_j),
        \quad t<P,
        \qquad
        s_{P,P}^h=\widetilde{M}_{null}.
    \]
    By Lemma \ref{lemma_1stblock},
    \begin{align*}
        |\xi_{t,\ell_j}|
        &\le 4Pe^{-M},\\
        |\widetilde U_{r,\ell_t,m_t}-U_{r,m_t}|,
        \ |\widetilde{M}_{null,r}-M_{null}|
        &\le 2P^2\max\{M_{null},dM_x\}e^{-M}
        \le 2M_2P^2e^{-M},\\
        |\widetilde Y_{r,\ell_t,m_t}-Y_{r,\ell_t,m_t}|,
        \ |\widetilde{0}_{n+1+r}|
        &\le 2B_f^2P^2e^{-M}.
    \end{align*}
    \begin{itemize}
    \item \textbf{Target Column ($j'=j$).}
    Since $M_2\ge M_1$, we have
    \[
        4M_2P^2e^{-M}+8M_1Pe^{-M}
        \le 12M_2P^2e^{-M}\le1.
    \]
    Let $\cI_{\ell_j}=\{t<P:\ell_t=\ell_j\}$. If $t\in\cI_{\ell_j}$ and $t'<P$, $t'\notin\cI_{\ell_j}$, then
    \begin{align*}
        s_{t,j}^h-s_{t',j}^h
        &=(\widetilde U_{r,\ell_t,m_t}-\widetilde U_{r,\ell_{t'},m_{t'}})
        +M_1\big(1-\cos(\phi_{\ell_{t'}}-\phi_{\ell_j})\big)
        +M_1(\xi_{t,\ell_j}-\xi_{t',\ell_j})\\
        &\ge -2dM_x+cM_1-4M_2P^2e^{-M}-8M_1Pe^{-M}\\
        &\ge (2dM_x+M+3)-2dM_x-1>M.
    \end{align*}
    Moreover, by Lemma \ref{lemma_1stblock},
    \begin{align*}
        s_{t,j}^h-s_{P,j}^h
        &=(\widetilde U_{r,\ell_t,m_t}-\widetilde M_{null})
        +M_2+M_1
        +M_1\xi_{t,\ell_j}
        -M_1\sin(\phi_{\ell_j})\widetilde{0}_{D-3}
        -M_1\cos(\phi_{\ell_j})\widetilde{0}_{D-2}\\
        &\ge -dM_x-M_{null}+M_2+M_1
        -4M_2P^2e^{-M}-6M_1Pe^{-M}\\
        &=\frac c2M_2-dM_x+M_1
        -4M_2P^2e^{-M}-6M_1Pe^{-M}\\
        &\ge dM_x+2M_1+M+2-dM_x+M_1-1>M.
    \end{align*}
    Consequently,
    \[
        \sum_{t'\notin\cI_{\ell_j}}(\bA_2^h)_{t',j}
        =
        \frac{\sum_{t'\notin\cI_{\ell_j}}\exp(s_{t',j}^h)}
        {\sum_{a=1}^P\exp(s_{a,j}^h)}
        \le
        \frac{\sum_{t'\notin\cI_{\ell_j}}\exp(s_{t',j}^h)}
        {\exp(s_{t,j}^h)}
        \le
        \sum_{t'\notin\cI_{\ell_j}}e^{s_{t',j}^h-s_{t,j}^h}
        \le Pe^{-M}.
    \]
    For $t\in\cI_{\ell_j}$, define the ideal score
    \[
        s_{t,j}^{*,h}:=U_{r,m_t}+M_2+M_1.
    \]
    Define
    \[
        \widetilde\eta_t^h:=
        \frac{\exp(s_{t,j}^h)}
        {\sum_{a\in\cI_{\ell_j}}\exp(s_{a,j}^h)},
        \qquad
        \bar\eta_t^h:=
        \frac{\exp(s_{t,j}^{*,h})}
        {\sum_{a\in\cI_{\ell_j}}\exp(s_{a,j}^{*,h})}.
    \]
    Since $M_2+M_1$ is constant on $\cI_{\ell_j}$,
    \[
        \bar\eta_t^h
        =
        \frac{\exp(U_{r,m_t}+M_2+M_1)}
        {\sum_{a\in\cI_{\ell_j}}\exp(U_{r,m_a}+M_2+M_1)}
        =
        \frac{\exp(U_{r,m_t})}
        {\sum_{a\in\cI_{\ell_j}}\exp(U_{r,m_a})}
        =
        \eta_{m_t}(\xb_r).
    \]
    Therefore
    \[
        W_{r,\ell_j}
        =
        \frac{2M_f}{n}\sum_{m=1}^{C_x}\eta_m(\xb_r)Y_{r,\ell_j,m}
        =
        \frac{2M_f}{n}\sum_{t\in\cI_{\ell_j}}\bar\eta_t^hY_{r,\ell_j,m_t}.
    \]
    Since, for $t\in\cI_{\ell_j}$,
    \begin{align*}
        |s_{t,j}^h-s_{t,j}^{*,h}|
        &=
        \left|
        \widetilde U_{r,\ell_t,m_t}-U_{r,m_t}
        +M_1\xi_{t,\ell_j}
        \right|\\
        &\le
        |\widetilde U_{r,\ell_t,m_t}-U_{r,m_t}|
        +M_1|\xi_{t,\ell_j}|\\
        &\le 2M_2P^2e^{-M}+4M_1Pe^{-M},
    \end{align*}
    Lemma \ref{Softmaxlipchitz} gives
    \[
        \sum_{t\in\cI_{\ell_j}}|\widetilde\eta_t^h-\bar\eta_t^h|
        \le4M_2P^2e^{-M}+8M_1Pe^{-M}.
    \]
    The normalization loss outside $\cI_{\ell_j}$ satisfies
    \[
        \rho_h
        :=
        \sum_{t\notin\cI_{\ell_j}}(\bA_2^h)_{t,j}
        \le Pe^{-M}.
    \]
    For $t\in\cI_{\ell_j}$,
    \[
        (\bA_2^h)_{t,j}
        =
        \frac{\exp(s_{t,j}^h)}
        {\sum_{a=1}^P\exp(s_{a,j}^h)}
        =
        (1-\rho_h)
        \frac{\exp(s_{t,j}^h)}
        {\sum_{a\in\cI_{\ell_j}}\exp(s_{a,j}^h)}
        =
        (1-\rho_h)\widetilde\eta_t^h.
    \]
    Hence
    \[
        \sum_{t\in\cI_{\ell_j}}|(\bA_2^h)_{t,j}-\widetilde\eta_t^h|
        =
        \rho_h\sum_{t\in\cI_{\ell_j}}\widetilde\eta_t^h
        =
        \rho_h
        \le Pe^{-M}.
    \]
    Therefore, by the triangle inequality and $P, M_2\ge1$, $M_2\ge M_1$,
    \begin{align*}
        \sum_{t\in\cI_{\ell_j}}|(\bA_2^h)_{t,j}-\bar\eta_t^h|
        &\le
        \sum_{t\in\cI_{\ell_j}}|(\bA_2^h)_{t,j}-\widetilde\eta_t^h|
        +\sum_{t\in\cI_{\ell_j}}|\widetilde\eta_t^h-\bar\eta_t^h|\\
        &\le Pe^{-M}
        +4M_2P^2e^{-M}
        +8M_1Pe^{-M}\le 13M_2P^2e^{-M}.
    \end{align*}
    For the target column of this head,
    \begin{align*}
        \left(\bV_2^h\bZ_1\bA_2^h\right)_{1,j}
        &=
        \frac{2M_f}{n}\bigg(
        \sum_{t\in\cI_{\ell_j}}(\bA_2^h)_{t,j}\widetilde Y_{r,\ell_j,m_t}
        +\sum_{\substack{t<P\\t\notin\cI_{\ell_j}}}(\bA_2^h)_{t,j}
        \widetilde Y_{r,\ell_t,m_t}
        +(\bA_2^h)_{P,j}\widetilde{0}_{n+1+r}
        \bigg).
    \end{align*}
    Hence, by Lemma \ref{lemma_1stblock} and $Pe^{-M}\le1/4$,
    \begin{align*}
        \left|\left(\bV_2^h\bZ_1\bA_2^h\right)_{1,j}
        -W_{r,\ell_j}\right|
        &=\frac{2M_f}{n}\bigg|
        \sum_{t\in\cI_{\ell_j}}(\bA_2^h)_{t,j}
        (\widetilde Y_{r,\ell_j,m_t}-Y_{r,\ell_j,m_t})
        +\sum_{t\in\cI_{\ell_j}}\big((\bA_2^h)_{t,j}-\bar\eta_t^h\big)Y_{r,\ell_j,m_t}\\
        &\qquad\qquad
        +\sum_{\substack{t<P\\t\notin\cI_{\ell_j}}}(\bA_2^h)_{t,j}
        \widetilde Y_{r,\ell_t,m_t}
        +(\bA_2^h)_{P,j}\widetilde{0}_{n+1+r}
        \bigg|\\
        &\le \frac{2M_f}{n}\Big\{
        2B_f^2P^2e^{-M}
        +13B_f^2M_2P^2e^{-M}
        +Pe^{-M}\big[B_f^2+4B_f^2P^2e^{-M}\big]
        \Big\}\\
        &\le \frac{2M_f}{n}\Big\{
        2B_f^2P^2e^{-M}
        +13B_f^2M_2P^2e^{-M}
        +2B_f^2P^2e^{-M}
        \Big\}\\
        &\le \frac{34B_f^2}{n}M_fM_2P^2e^{-M}.
    \end{align*}

    \item \textbf{Non-Target Columns ($j'\in[P]$, $j'\ne j$).}
    Since $|U_{r,m}|\le dM_x$ and $2M_2P^2e^{-M}\le1$,
    \[
        |\widetilde U_{r,\ell_t,m_t}|\le dM_x+1.
    \]
    Also, since $M_2\ge M_1$,
    \[
        4M_2P^2e^{-M}+8M_1Pe^{-M}
        \le 12M_2P^2e^{-M}\le1.
    \]
    For any active $t<P$ and active non-target column $j'<P$, $j'\ne j$,
    \begin{align*}
        s_{P,j'}^h-s_{t,j'}^h
        &=
        \widetilde{M}_{null}-\widetilde U_{r,\ell_t,m_t}
        -M_2\cos(\theta_{j'}-\theta_j)
        -M_1\cos(\phi_{\ell_t}-\phi_{\ell_j})\\
        &\qquad
        +M_1\sin(\phi_{\ell_j})\widetilde{0}_{D-3}
        +M_1\cos(\phi_{\ell_j})\widetilde{0}_{D-2}
        -M_1\xi_{t,\ell_j}\\
        &\ge M_{null}-dM_x-M_2(1-c)-M_1
        -4M_2P^2e^{-M}
        -6M_1Pe^{-M}\\
        &=\frac{c}{2}M_2-dM_x-M_1
        -4M_2P^2e^{-M}
        -6M_1Pe^{-M}\\
        &\ge M_1+M+2-1>M.
    \end{align*}
    For the null column $j'=P$ and any active $t<P$,
    \begin{align*}
        s_{P,P}^h-s_{t,P}^h
        &=\widetilde{M}_{null}-\widetilde U_{r,\ell_t,m_t}
        -M_2\cos(\theta_P-\theta_j)\\
        &\ge M_{null}-dM_x-M_2(1-c)-4M_2P^2e^{-M}\\
        &=\frac{c}{2}M_2-dM_x-4M_2P^2e^{-M}\\
        &\ge 2M_1+M+2-1
        >M.
    \end{align*}
    Thus, for every $j'\ne j$,
    \[
        1-(\bA_2^h)_{P,j'}
        =
        \sum_{t<P}(\bA_2^h)_{t,j'}
        =
        \frac{\sum_{t<P}\exp(s_{t,j'}^h)}
        {\sum_{a=1}^P\exp(s_{a,j'}^h)}
        \le
        \frac{\sum_{t<P}\exp(s_{t,j'}^h)}
        {\exp(s_{P,j'}^h)}
        =
        \sum_{t<P}e^{s_{t,j'}^h-s_{P,j'}^h}
        \le Pe^{-M}.
    \]
    Hence
    \begin{align*}
        \left|\left(\bV_2^h\bZ_1\bA_2^h\right)_{1,j'}\right|
        &=
        \frac{2M_f}{n}\left|
        \sum_{t<P}(\bA_2^h)_{t,j'}\widetilde{Y}_{r,\ell_t,m_t}
        +(\bA_2^h)_{P,j'}\widetilde{0}_{n+1+r}
        \right|\\
        &\le \frac{2M_f}{n}
        \left\{|\widetilde{0}_{n+1+r}|+
        \sum_{t<P}(\bA_2^h)_{t,j'}|\widetilde{Y}_{r,\ell_t,m_t}|\right\}\\
        &\le \frac{2M_f}{n}
        \Big\{2B_f^2P^2e^{-M}
        +Pe^{-M}\big[B_f^2+2B_f^2P^2e^{-M}\big]\Big\} \le \frac{8B_f^2}{n}M_fP^2e^{-M}.
    \end{align*}
    \end{itemize}
    For $j=P$, set $\bQ_2^h=\bK_2^h=\bm{0}_{5\times D}$ and $\bV_2^h=\bm{0}_{2\times D}$. Then the output of this head is
    \[
        \bV_2^h\bZ_1\bA_2^h=\bm{0}_{2\times P}.
    \]

    \textbf{Step 2: Auxiliary Heads.}
    Let $h_0:=nP$ and $M_{\mathrm{id}}:=M/c$. The heads $h_0+1$ and $h_0+2$ preserve rows $n+1,2n+2$, and $2n+3$ of $\bZ_1$. We construct query and key matrices $\bQ_2^{h_0+1},\bQ_2^{h_0+2},\bK_2^{h_0+1},\bK_2^{h_0+2}\in\RR^{5\times D}$ as
    \[
        \bQ_2^{h_0+1}=\bQ_2^{h_0+2}=
        \begin{bmatrix}
            \bm{0}_{1\times D}\\
            \be_{D-1}^{\top}\\
            \be_D^{\top}\\
            \bm{0}_{1\times D}\\
            \bm{0}_{1\times D}
        \end{bmatrix},\qquad
        \bK_2^{h_0+1}=\bK_2^{h_0+2}=
        \begin{bmatrix}
            \bm{0}_{1\times D}\\
            M_{\mathrm{id}}\be_{D-1}^{\top}\\
            M_{\mathrm{id}}\be_D^{\top}\\
            \bm{0}_{1\times D}\\
            \bm{0}_{1\times D}
        \end{bmatrix}.
    \]
    For $a\in\{1,2\}$,
    \[
        s_{t,j}^{h_0+a}=M_{\mathrm{id}}\cos(\theta_t-\theta_j),\qquad
        s_{j,j}^{h_0+a}-s_{t,j}^{h_0+a}\ge M,\quad t\ne j.
    \]
    Let $\bA_2^{h_0+a}$ be the corresponding attention matrix. Then
    \[
        \sum_{t\ne j}(\bA_2^{h_0+a})_{t,j}
        \le\sum_{t\ne j}e^{s_{t,j}^{h_0+a}-s_{j,j}^{h_0+a}}
        \le Pe^{-M},
        \qquad a\in\{1,2\}.
    \]
    We construct the value matrices $\bV_2^{h_0+1},\bV_2^{h_0+2}\in\RR^{2\times D}$ as
    \[
        \bV_2^{h_0+1}=\begin{bmatrix}\be_{n+1}^{\top}\\ \be_{2n+2}^{\top}\end{bmatrix},\qquad
        \bV_2^{h_0+2}=\begin{bmatrix}\be_{2n+3}^{\top}\\ \bm{0}_{1\times D}\end{bmatrix}.
    \]
    For $j<P$, define
    \begin{align*}
        \overline U_{n+1,\ell_j,m_j}
        &:=(\bV_2^{h_0+1}\bZ_1\bA_2^{h_0+1})_{1,j},\\
        \overline H_{\ell_j,m_j}
        &:=(\bV_2^{h_0+1}\bZ_1\bA_2^{h_0+1})_{2,j},\\
        \overline V_{\ell_j,m_j}
        &:=(\bV_2^{h_0+2}\bZ_1\bA_2^{h_0+2})_{1,j}.
    \end{align*}
    and for the null column define
    \[
        \overline M_{null}:=(\bV_2^{h_0+1}\bZ_1\bA_2^{h_0+1})_{1,P},
        \quad
        \overline{0}_H:=(\bV_2^{h_0+1}\bZ_1\bA_2^{h_0+1})_{2,P},
        \quad
        \overline{0}_V:=(\bV_2^{h_0+2}\bZ_1\bA_2^{h_0+2})_{1,P}.
    \]
    By Lemma \ref{lemma_1stblock}, for $t<P$ and $t\ne j$,
    \begin{align*}
        |\widetilde U_{n+1,\ell_t,m_t}-U_{n+1,m_j}|
        &\le
        |\widetilde U_{n+1,\ell_t,m_t}-U_{n+1,m_t}|
        +|U_{n+1,m_t}-U_{n+1,m_j}|\\
        &\le 2M_2P^2e^{-M}+2dM_x\le 1+2dM_x
        \le M_2.
    \end{align*}
    For the null source column,
    \begin{align*}
        |\widetilde M_{null,n+1}-U_{n+1,m_j}|
        &\le
        |\widetilde M_{null,n+1}-M_{null}|
        +|M_{null}-U_{n+1,m_j}|\\
        &\le 2M_2P^2e^{-M}+M_{null}+dM_x\\
        &\le 1+M_2-(2M_1+M+2)\le M_2.
    \end{align*}
    Hence, since $P\ge1$ and $Pe^{-M}\le1/4$, Lemma \ref{lemma_1stblock} gives
    \begin{align*}
        |\overline U_{n+1,\ell_j,m_j}-U_{n+1,m_j}|
        &=
        \left|
        \begin{aligned}
        \sum_{t<P}(\bA_2^{h_0+1})_{t,j}
        \big(\widetilde U_{n+1,\ell_t,m_t}-U_{n+1,m_j}\big)
        \\
        +(\bA_2^{h_0+1})_{P,j}
        \big(\widetilde M_{null,n+1}-U_{n+1,m_j}\big)
        \end{aligned}
        \right|\\
        &\le
        |\widetilde U_{n+1,\ell_j,m_j}-U_{n+1,m_j}|
        +\sum_{t\ne j}(\bA_2^{h_0+1})_{t,j}
        M_2\\
        &\le 2M_2P^2e^{-M}+M_2Pe^{-M}\\
        &\le 3M_2P^2e^{-M}.
    \end{align*}
    For the null column, first observe that for $t<P$,
    \begin{align*}
        |\widetilde U_{n+1,\ell_t,m_t}-M_{null}|
        &\le
        |\widetilde U_{n+1,\ell_t,m_t}-U_{n+1,m_t}|
        +|U_{n+1,m_t}-M_{null}|\\
        &\le 2M_2P^2e^{-M}+dM_x+M_{null}\\
        &\le 1+M_2-(2M_1+M+2)\le M_2.
    \end{align*}
    Therefore
    \begin{align*}
        |\overline M_{null}-M_{null}|
        &=
        \left|
        \sum_{t<P}(\bA_2^{h_0+1})_{t,P}
        \big(\widetilde U_{n+1,\ell_t,m_t}-M_{null}\big)
        +(\bA_2^{h_0+1})_{P,P}
        \big(\widetilde M_{null,n+1}-M_{null}\big)
        \right|\\
        &\le
        |\widetilde M_{null,n+1}-M_{null}|
        +\sum_{t\ne P}(\bA_2^{h_0+1})_{t,P}
        M_2\\
        &\le 2M_2P^2e^{-M}
        +M_2Pe^{-M}\le 3M_2P^2e^{-M}.
    \end{align*}
    For $t<P$ and $t\ne j$,
    \begin{align*}
        |\widetilde H_{\ell_t,m_t}-H_{\ell_j,m_j}|
        &\le
        |\widetilde H_{\ell_t,m_t}-H_{\ell_t,m_t}|
        +|H_{\ell_t,m_t}-H_{\ell_j,m_j}|\\
        &\le 2B_fP^2e^{-M}+2B_f \le 3B_f.
    \end{align*}
    For the null source column,
    \[
        |\widetilde{0}_{2n+2}-H_{\ell_j,m_j}|
        \le 2B_fP^2e^{-M}+B_f\le 2B_f.
    \]
    Hence
    \begin{align*}
        |\overline H_{\ell_j,m_j}-H_{\ell_j,m_j}|
        &=
        \left|
        \sum_{t<P}(\bA_2^{h_0+1})_{t,j}
        \big(\widetilde H_{\ell_t,m_t}-H_{\ell_j,m_j}\big)
        +(\bA_2^{h_0+1})_{P,j}
        \big(\widetilde{0}_{2n+2}-H_{\ell_j,m_j}\big)
        \right|\\
        &\le
        |\widetilde H_{\ell_j,m_j}-H_{\ell_j,m_j}|
        +\sum_{t\ne j}(\bA_2^{h_0+1})_{t,j}
        3B_f\\
        &\le 2B_fP^2e^{-M}
        +3B_fPe^{-M}
        \le 5B_fP^2e^{-M}.
    \end{align*}
    For the null column,
    \begin{align*}
        |\overline{0}_H|
        &=
        \left|
        \sum_{t<P}(\bA_2^{h_0+1})_{t,P}\widetilde H_{\ell_t,m_t}
        +(\bA_2^{h_0+1})_{P,P}\widetilde{0}_{2n+2}
        \right|\\
        &\le
        |\widetilde{0}_{2n+2}|
        +\sum_{t\ne P}(\bA_2^{h_0+1})_{t,P}
        2B_f\\
        &\le 2B_fP^2e^{-M}
        +2B_fPe^{-M}
        \le 4B_fP^2e^{-M}.
    \end{align*}
    For $t<P$ and $t\ne j$,
    \begin{align*}
        |\widetilde V_{\ell_t,m_t}-V_{\ell_j}|
        &\le
        |\widetilde V_{\ell_t,m_t}-V_{\ell_t}|
        +|V_{\ell_t}-V_{\ell_j}|\\
        &\le 2B_f^2M_fP^2e^{-M}+2B_f^2M_f
        \le 3B_f^2M_f.
    \end{align*}
    For the null source column,
    \[
        |\widetilde{0}_{2n+3}-V_{\ell_j}|
        \le 2B_f^2M_fP^2e^{-M}+B_f^2M_f
        \le 2B_f^2M_f.
    \]
    Hence
    \begin{align*}
        |\overline V_{\ell_j,m_j}-V_{\ell_j}|
        &=
        \left|
        \sum_{t<P}(\bA_2^{h_0+2})_{t,j}
        \big(\widetilde V_{\ell_t,m_t}-V_{\ell_j}\big)
        +(\bA_2^{h_0+2})_{P,j}
        \big(\widetilde{0}_{2n+3}-V_{\ell_j}\big)
        \right|\\
        &\le
        |\widetilde V_{\ell_j,m_j}-V_{\ell_j}|
        +\sum_{t\ne j}(\bA_2^{h_0+2})_{t,j}
        3B_f^2M_f\\
        &\le 2B_f^2M_fP^2e^{-M}
        +3B_f^2M_fPe^{-M}
        \le 5B_f^2M_fP^2e^{-M}.
    \end{align*}
    For the null column,
    \begin{align*}
        |\overline{0}_V|
        &=
        \left|
        \sum_{t<P}(\bA_2^{h_0+2})_{t,P}\widetilde V_{\ell_t,m_t}
        +(\bA_2^{h_0+2})_{P,P}\widetilde{0}_{2n+3}
        \right|\\
        &\le
        |\widetilde{0}_{2n+3}|
        +\sum_{t\ne P}(\bA_2^{h_0+2})_{t,P}
        2B_f^2M_f\\
        &\le 2B_f^2M_fP^2e^{-M}
        +2B_f^2M_fPe^{-M}
        \le 4B_f^2M_fP^2e^{-M}.
    \end{align*}

    The head $h_0+3$ constructs the constant-one row and the active-indicator row. Let
    \[
        \chi_j:=1_{\{j<P\}},\qquad \nu_j:=1-\chi_j,\qquad j\in[P].
    \]
    With $a_M$ as in the lemma statement, construct the query and key matrices
    $\bQ_2^{h_0+3},\bK_2^{h_0+3}\in\RR^{5\times D}$ as
    \[
        \bQ_2^{h_0+3}
        =
        \begin{bmatrix}
            \be_{D-4}^{\top}\\
            \be_{D-5}^{\top}-\be_{D-4}^{\top}\\
            \bm{0}_{1\times D}\\
            \bm{0}_{1\times D}\\
            \bm{0}_{1\times D}
        \end{bmatrix}
        \qquad
        \bK_2^{h_0+3}
        =
        \begin{bmatrix}
            M\be_{D-4}^{\top}\\
            M(\be_{D-5}^{\top}-\be_{D-4}^{\top})\\
            \bm{0}_{1\times D}\\
            \bm{0}_{1\times D}\\
            \bm{0}_{1\times D}
        \end{bmatrix},
    \]
    and construct the value matrix $\bV_2^{h_0+3}\in\RR^{2\times D}$ as
    \[
        \bV_2^{h_0+3}
        =
        \begin{bmatrix}
            \be_{D-5}^{\top}\\
            \be_{D-4}^{\top}-(P-1)e^{-M}(\be_{D-5}^{\top}-\be_{D-4}^{\top})
        \end{bmatrix}.
    \]
    Let $\bA_2^{h_0+3}$ be the corresponding attention matrix. This is the same indicator-head construction as in the proof of Lemma \ref{lemma_1stblock}. Hence the same calculation gives
    \[
        \left(\bV_2^{h_0+3}\bZ_1\bA_2^{h_0+3}\right)_{:,j}
        =
        \begin{bmatrix}1\\a_M\end{bmatrix},
        \quad j<P,
        \qquad
        \left(\bV_2^{h_0+3}\bZ_1\bA_2^{h_0+3}\right)_{:,P}
        =
        \begin{bmatrix}1\\0\end{bmatrix}.
    \]

    \textbf{Step 3: Concatenation and Output Projection.}
    For each POU feature head $h\in[nP]$, write
    \[
        \bC_2^h:=\bV_2^h\bZ_1\bA_2^h\in\RR^{2\times P}.
    \]
    For $r\in[n]$ and $j,q\in[P]$, define
    \[
        G_{r,j,q}
        :=
        \left(\bC_2^{(r-1)P+j}\right)_{1,q}\in\RR .
    \]
    The $P$ POU feature heads assigned to row $r$ form the block
    \[
        \bC_2^{(r)}
        :=
        \begin{bmatrix}
            G_{r,1,1} & G_{r,1,2} & \cdots & G_{r,1,P}\\
            0 & 0 & \cdots & 0\\
            G_{r,2,1} & G_{r,2,2} & \cdots & G_{r,2,P}\\
            0 & 0 & \cdots & 0\\
            \vdots & \vdots & \ddots & \vdots\\
            G_{r,P,1} & G_{r,P,2} & \cdots & G_{r,P,P}\\
            0 & 0 & \cdots & 0
        \end{bmatrix}
        \in\RR^{2P\times P}.
    \]
    The three auxiliary heads give
    \[
        \bC_2^{\mathrm{id},1}
        :=
        \bV_2^{h_0+1}\bZ_1\bA_2^{h_0+1}
        =
        \begin{bmatrix}
            \overline U_{n+1,1} & \cdots & \overline U_{n+1,P-1} & \overline M_{null}\\
            \overline H_{\ell_1,m_1} & \cdots & \overline H_{\ell_{P-1},m_{P-1}} & \overline{0}_H
        \end{bmatrix},
    \]
    \[
        \bC_2^{\mathrm{id},2}
        :=
        \bV_2^{h_0+2}\bZ_1\bA_2^{h_0+2}
        =
        \begin{bmatrix}
            \overline V_1 & \cdots & \overline V_{P-1} & \overline{0}_V\\
            0 & \cdots & 0 & 0
        \end{bmatrix},
    \]
    and
    \[
        \bC_2^{\mathrm{ind}}
        :=
        \bV_2^{h_0+3}\bZ_1\bA_2^{h_0+3}
        =
        \begin{bmatrix}
            1 & \cdots & 1 & 1\\
            a_M & \cdots & a_M & 0
        \end{bmatrix}.
    \]
    Concatenating all head outputs gives
    \[
        \bC_2
        =
        \begin{bmatrix}
            \bC_2^{(1)}\\
            \vdots\\
            \bC_2^{(n)}\\
            \bC_2^{\mathrm{id},1}\\
            \bC_2^{\mathrm{id},2}\\
            \bC_2^{\mathrm{ind}}
        \end{bmatrix}
        \in\RR^{2H^2\times P}.
    \]
    Let
    \[
        \bI_{\mathrm{sum}}
        :=
        \begin{bmatrix}
            1&0&1&0&\cdots&1&0
        \end{bmatrix}
        \in\RR^{1\times 2P},
    \]
    and define
    \[
        \bR_{\mathrm{POU}}
        :=
        \begin{bmatrix}
            \bI_{\mathrm{sum}} & \bm{0}_{1\times 2P} & \cdots & \bm{0}_{1\times 2P}\\
            \bm{0}_{1\times 2P} & \bI_{\mathrm{sum}} & \cdots & \bm{0}_{1\times 2P}\\
            \vdots & \vdots & \ddots & \vdots\\
            \bm{0}_{1\times 2P} & \bm{0}_{1\times 2P} & \cdots & \bI_{\mathrm{sum}}
        \end{bmatrix}
        \in\RR^{n\times 2nP}.
    \]
    Also define
    \[
        \bR_{\mathrm{aux}}
        :=
        \begin{bmatrix}
            1&0&0&0&0&0\\
            0&1&0&0&0&0\\
            0&0&1&0&0&0\\
            0&0&0&0&1&0\\
            0&0&0&0&0&1
        \end{bmatrix}
        \in\RR^{5\times 6}.
    \]
    Choose the output projection matrix
    \[
        \bW_2^O
        =
        \begin{bmatrix}
            \bR_{\mathrm{POU}} & \bm{0}_{n\times 6}\\
            \bm{0}_{5\times 2nP} & \bR_{\mathrm{aux}}\\
            \bm{0}_{(D-n-5)\times 2nP} & \bm{0}_{(D-n-5)\times 6}
        \end{bmatrix}
        \in\RR^{D\times 2H^2}.
    \]
    Define
    \[
        \widetilde W_{r,\ell_j,m_j}:=\sum_{q=1}^P G_{r,q,j},
        \quad r\in[n],\ j<P,
        \qquad
        \overline{0}_r:=\sum_{q=1}^P G_{r,q,P},
        \quad r\in[n].
    \]
    Then $\widehat{\bZ}_2=\bW_2^O\bC_2$ satisfies
    \[
        (\widehat{\bZ}_2)_{:,j}
        =
        \begin{bmatrix}
            \widetilde W_{1,\ell_j,m_j}\\
            \vdots\\
            \widetilde W_{n,\ell_j,m_j}\\
            \overline U_{n+1,\ell_j,m_j}\\
            \overline H_{\ell_j,m_j}\\
            \overline V_{\ell_j,m_j}\\
            1\\
            a_M\\
            \bm{0}_{D-n-5}
        \end{bmatrix},
        \quad j<P,
        \qquad
        (\widehat{\bZ}_2)_{:,P}
        =
        \begin{bmatrix}
            \overline{0}_1\\
            \vdots\\
            \overline{0}_n\\
            \overline M_{null}\\
            \overline{0}_H\\
            \overline{0}_V\\
            1\\
            0\\
            \bm{0}_{D-n-5}
        \end{bmatrix}.
    \]
    This is exactly the matrix form \eqref{Z2hat}.

    By Step 1, for $r\in[n]$ and $j<P$,
    \[
        |G_{r,j,j}-W_{r,\ell_j}|
        \le \frac{34B_f^2}{n}M_fM_2P^2e^{-M},
        \qquad
        |G_{r,j,q}|
        \le \frac{8B_f^2}{n}M_fP^2e^{-M},\quad q\ne j,
    \]
    and $G_{r,P,q}=0$ for every $q\in[P]$. 
    Moreover,
    $c=2\sin^2(\pi/P)\le2\pi^2/P^2$, and therefore
    \[
        M_2
        \ge \frac{(dM_x+2M_1+M+2)P^2}{\pi^2}
        \ge \frac{6P^2}{\pi^2}
        \ge P.
    \]
    For each active column $j<P$,
    \begin{align*}
        |\widetilde W_{r,\ell_j,m_j}-W_{r,\ell_j}|
        &=
        \left|
        G_{r,j,j}-W_{r,\ell_j}
        +\sum_{\substack{q=1,q\ne j}}^P G_{r,q,j}
        \right| \le
        |G_{r,j,j}-W_{r,\ell_j}|
        +\sum_{\substack{q=1,q\ne j}}^P |G_{r,q,j}|\\
        &\le
        \frac{34B_f^2}{n}M_fM_2P^2e^{-M}
        +(P-1)\frac{8B_f^2}{n}M_fP^2e^{-M} \le \frac{42B_f^2}{n}M_fM_2P^2e^{-M}.
    \end{align*}
    For the null column,
    \begin{align*}
        |\overline{0}_r|
        &=
        \left|\sum_{q=1}^P G_{r,q,P}\right|
        =
        \left|\sum_{q=1}^{P-1} G_{r,q,P}\right|
        \le
        \sum_{q=1}^{P-1}|G_{r,q,P}|
        \le \frac{8B_f^2}{n}M_fP^3e^{-M}.
    \end{align*}
    The bounds for $\overline U_{n+1,\ell_j,m_j}$, $\overline M_{null}$, $\overline H_{\ell_j,m_j}$, $\overline{0}_H$, $\overline V_{\ell_j,m_j}$, and $\overline{0}_V$ are the estimates obtained in Step 2.

    \textbf{Step 4: Parameter bound.}
    The matrices in Steps 1--3 and $\|\bW_2^O\|_{\max}=1$ give
    \[
        M_{\cA_2}
        \le
        \max\left\{
        1,M_1,M_2,\frac{2M_f}{n},\frac{M}{c},M,
        1+(P-1)e^{-M}
        \right\}.
    \]
    Here $Pe^{-M}\le1/4$ and $n\ge1$. Moreover,
    $P\ge n+2\ge3$ implies $0<c<2$, and the definition of $M_2$ gives
    \[
        M_2
        =\frac{2(dM_x+2M_1+M+2)}{c}
        \ge
        \max\left\{\frac{4M_1}{c},\frac{2M}{c}\right\}
        \ge
        \max\left\{M_1,\frac{M}{c},M\right\}.
    \]
    Using also $2M_f/n\le2M_f$ and
    $1+(P-1)e^{-M}\le5/4$, we conclude that
    \[
        M_{\cA_2}\le2\max\{M_f,M_2\},
    \]
    as claimed.
\end{proof}

The output of the second-block MHA layer stores the context
contributions separately. The FFN layer further sums them and
adds the query-localization and anchor-bias terms to form the joint oracle
logits while preserving the associated anchor values.

\begin{proof}[Proof of Lemma \ref{lemma_2ndblock_ffn}]
\textbf{Step 1: Point-wise FFN construction.}
Set
\[
    M_{\rm out}
    :=
    dM_x+(1+8B_f^2)M_f+M .
\]
Define
\[
    \bw_{\Xi}^{\top}
    :=
    \sum_{r=1}^{n+1}\be_r^{\top}
    -\be_{n+3}^{\top}
    +a_M^{-1}(M_{null}+M_{\rm out})\be_{n+5}^{\top}
    \in\RR^{1\times D}.
\]
Choose $d_{\mathrm{ff}}^2=2D$ and set
\[
    \bW_2^1
    =
    \begin{bmatrix}
        \bw_{\Xi}^{\top}\\
        -\bw_{\Xi}^{\top}\\
        \be_{n+2}^{\top}\\
        -\be_{n+2}^{\top}\\
        \bm{0}_{(2D-4)\times D}
    \end{bmatrix}
    \in\RR^{2D\times D},
    \qquad
    \bb_2^1
    =
    \begin{bmatrix}
        -(M_{null}+M_{\rm out})\\
        M_{null}+M_{\rm out}\\
        0\\
        0\\
        \bm{0}_{2D-4}
    \end{bmatrix}
    \in\RR^{2D},
\]
and
\[
    \bW_2^2
    =
    \begin{bmatrix}
        1&-1&0&0&\bm{0}_{1\times(2D-4)}\\
        0&0&1&-1&\bm{0}_{1\times(2D-4)}\\
        \bm{0}_{(D-2)\times1}
        &\bm{0}_{(D-2)\times1}
        &\bm{0}_{(D-2)\times1}
        &\bm{0}_{(D-2)\times1}
        &\bm{0}_{(D-2)\times(2D-4)}
    \end{bmatrix}
    \in\RR^{D\times2D},
    \qquad
    \bb_2^2=\be_3\in\RR^D.
\]
For each $j\in[P]$, define
\[
    \widetilde\Xi_j
    :=
    \bw_{\Xi}^{\top}(\widehat{\bZ}_2)_{:,j}
    -(M_{null}+M_{\rm out}).
\]
By \eqref{Z2hat}, the gate
$a_M^{-1}(\widehat{\bZ}_2)_{n+5,j}$ equals $1$ for $j<P$ and
$0$ for $j=P$. Consequently,
\begin{align*}
    \widetilde\Xi_j
    &=\sum_{i=1}^n\widetilde W_{i,\ell_j,m_j}
    -\overline V_{\ell_j,m_j}
    +\overline U_{n+1,\ell_j,m_j},
    &&j<P,\\
    \widetilde\Xi_P
    &=\sum_{i=1}^n\overline{0}_i
    -\overline{0}_V
    +\overline M_{null}
    -M_{null}-M_{\rm out}.
\end{align*}
Using $x=\sigma(x)-\sigma(-x)$, the FFN places these logits in the
first row, copies the $(n+2)$-nd coordinate into the second row, places
$1$ in the third row, and sets the remaining coordinates to zero.
Together with \eqref{Z2hat}, this proves \eqref{Z2}.

\textbf{Step 2: Output error bounds.}
By Lemma \ref{lemma_2ndblock} and the triangle inequality, for $j<P$,
\begin{align*}
    |\widetilde\Xi_j-\Xi_{\ell_j,m_j}(\fraks)|
    &\le
    3M_2P^2e^{-M}
    +42B_f^2M_fM_2P^2e^{-M}
    +5B_f^2M_fP^2e^{-M}\\
    &\le
    (3+47B_f^2)M_fM_2P^2e^{-M},\\
    |\overline H_{\ell_j,m_j}-H_{\ell_j,m_j}|
    &\le 5B_fP^2e^{-M}.
\end{align*}
Moreover, the definitions of $U$, $W$, and $V$ give
\[
    |U_{n+1,m_j}|\le dM_x,
    \qquad
    \sum_{i=1}^n|W_{i,\ell_j}|\le 2M_fB_f^2,
    \qquad
    |V_{\ell_j}|\le M_fB_f^2.
\]
Thus the active-logit error bound yields
\[
    \max_{j\in[P-1]}\widetilde\Xi_j
    \ge
    -dM_x-3M_fB_f^2
    -(3+47B_f^2)M_fM_2P^2e^{-M}.
\]
For the null column, the same bounds give
\begin{align*}
    |\widetilde\Xi_P+M_{\rm out}|
    &\le
    3M_2P^2e^{-M}
    +8B_f^2M_fP^3e^{-M}
    +4B_f^2M_fP^2e^{-M}\\
    &\le
    (3+12B_f^2)M_fM_2P^2e^{-M},\\
    |\overline{0}_H|
    &\le 4B_fP^2e^{-M},
\end{align*}
where the second inequality uses $M_f,M_2\ge1$ and $M_2\ge P$.
Since $12M_2P^2e^{-M}\le1$,
\[
    (6+59B_f^2)M_fM_2P^2e^{-M}
    \le (1+5B_f^2)M_f.
\]
Therefore, by the definition of $M_{\rm out}$,
\begin{align*}
    \widetilde\Xi_P
    -\max_{j\in[P-1]}\widetilde\Xi_j
    &\le
    -M_{\rm out}+dM_x+3M_fB_f^2
    +(6+59B_f^2)M_fM_2P^2e^{-M}\\
    &\le
    -M.
\end{align*}

\textbf{Step 3: Parameter bound.}
As in the proof of Lemma \ref{lemma_1stblock}, $a_M^{-1}\le2$.
Since $M_{\rm out}\ge M>1$, inspection of the matrices above gives
\[
    M_{\cF_2}\le2(M_{null}+M_{\rm out}).
\]
Furthermore, $P\ge3$ implies $0<c<2$, so $M_{null}\le M_2$.
The definitions of $M_1$ and $M_2$ give
$M_2>2M_1>2dM_x+M$, and hence
\[
    M_{\rm out}
    \le
    M_2+(1+8B_f^2)M_f.
\]
Consequently,
\[
    M_{\cF_2}
    \le4M_2+2(1+8B_f^2)M_f
    \le
    (6+16B_f^2)\max\{M_f,M_2\}.
\]
\end{proof}

The first two encoder blocks provide approximate joint logits together with
the corresponding anchor values. The final block combines these quantities
through a joint Softmax aggregation, producing the network output as an
approximation of the oracle.

\begin{proof}[Proof of Lemma \ref{lemma_final_block}]
\textbf{Step 1: Final MHA layer.}
Set $H^3=1$ and $d_k^3=d_v^3=1$, and choose
\[
    \bQ_3^1=\be_3^\top,\qquad
    \bK_3^1=\be_1^\top,\qquad
    \bV_3^1=\be_2^\top,
    \qquad
    \bW_3^O=\be_1\in\RR^{D\times1}.
\]
Let $\bA_3^1$ denote the attention matrix of this head. By
\eqref{Z2}, for $t,q\in[P]$,
\[
    (\bQ_3^1\bZ_2)_{1,q}=1,
    \qquad
    (\bK_3^1\bZ_2)_{1,t}=\widetilde\Xi_t,
    \qquad
    (\bV_3^1\bZ_2)_{1,t}
    =
    \begin{cases}
        \overline H_{\ell_t,m_t},&t<P,\\
        \overline{0}_H,&t=P.
    \end{cases}
\]
Following the column-wise definition in Section \ref{sec:architecture},
for every $q\in[P]$,
\[
    (\bA_3^1)_{:,q}
    =
    \softmax\!\left(
    \left(\bZ_2^\top{\bK_3^1}^\top
    \bQ_3^1\bZ_2\right)_{:,q}
    \right).
\]
Since every query column has the same scores, for $t,q\in[P]$,
\[
    (\bA_3^1)_{t,q}
    =
    \frac{\exp(\widetilde\Xi_t)}
    {\sum_{t'=1}^{P}\exp(\widetilde\Xi_{t'})}
    =:\widetilde\gamma_t.
\]
Since $\bW_3^O=\be_1$, the first output coordinate is
\begin{align*}
    (\widehat{\bZ}_3)_{1,1}
    &=\sum_{t=1}^{P}\widetilde\gamma_t
    (\bV_3^1\bZ_2)_{1,t}=\sum_{j=1}^{P-1}\widetilde\gamma_j
    \overline H_{\ell_j,m_j}
    +\widetilde\gamma_P\overline{0}_H .
\end{align*}
For the final FFN, choose $d_{\mathrm{ff}}^3=2D$ and set
\[
    \bW_3^1
    =
    \begin{bmatrix}
        \be_1^\top\\
        -\be_1^\top\\
        \bm{0}_{(2D-2)\times D}
    \end{bmatrix},
    \qquad
    \bb_3^1=\bm{0}_{2D},
\]
\[
    \bW_3^2
    =
    \begin{bmatrix}
        \be_1&-\be_1&\bm{0}_{D\times(2D-2)}
    \end{bmatrix},
    \qquad
    \bb_3^2=\bm{0}_D.
\]
Since $x=\sigma(x)-\sigma(-x)$,
$\bZ_3=\be_1\be_1^\top\widehat{\bZ}_3$, and hence
$(\bZ_3)_{1,1}=(\widehat{\bZ}_3)_{1,1}$.
Choose $\bc_4=\be_1\in\RR^{DP}$. Then
\[
    \rmT^*(\fraks)
    =\bc_4^\top\mathrm{vec}(\bZ_3)
    =(\bZ_3)_{1,1}.
\]
All parameters in this MHA layer, FFN layer, and readout have magnitude
at most $1$.

\textbf{Step 2: Null-token mass.}
By the null-logit separation in Lemma \ref{lemma_2ndblock_ffn},
\[
    \widetilde\Xi_P
    -\max_{j\in[P-1]}\widetilde\Xi_j
    \le -M .
\]
Therefore,
\[
    \widetilde\gamma_P
    \le
    \exp\left(
    \widetilde\Xi_P-\max_{j\in[P-1]}\widetilde\Xi_j
    \right)
    \le e^{-M}.
\]
For $j\in[P-1]$, define the active-only softmax weight explicitly by
\[
    \bar\gamma_j
    :=
    \frac{\exp(\widetilde\Xi_j)}
    {\sum_{j'=1}^{P-1}\exp(\widetilde\Xi_{j'})}.
\]
By Proposition \ref{prop:joint_softmax}, the corresponding ideal softmax
weights are $\gamma_{\ell_j,m_j}(\fraks)$. Combining the logit error bound
in Lemma \ref{lemma_2ndblock_ffn} with Lemma \ref{Softmaxlipchitz} gives
\[
    \sum_{j=1}^{P-1}
    |\bar\gamma_j-\gamma_{\ell_j,m_j}(\fraks)|
    \le 2(3+47B_f^2)M_fM_2P^2e^{-M}.
\]
Since
$\widetilde\gamma_j=(1-\widetilde\gamma_P)\bar\gamma_j$
for every $j\in[P-1]$ and $\sum_{j=1}^{P-1}\bar\gamma_j=1$, inserting
$\bar\gamma_j$ decomposes the error into the null-token normalization
error and the active-logit approximation error:
\begin{align*}
    \sum_{j=1}^{P-1}
    |\widetilde\gamma_j-\gamma_{\ell_j,m_j}(\fraks)|
    &\le
    \sum_{j=1}^{P-1}|\widetilde\gamma_j-\bar\gamma_j|
    +\sum_{j=1}^{P-1}
    |\bar\gamma_j-\gamma_{\ell_j,m_j}(\fraks)|\\
    &=
    \widetilde\gamma_P\sum_{j=1}^{P-1}\bar\gamma_j
    +\sum_{j=1}^{P-1}
    |\bar\gamma_j-\gamma_{\ell_j,m_j}(\fraks)|\\
    &\le
    e^{-M}+2(3+47B_f^2)M_fM_2P^2e^{-M}.
\end{align*}

\textbf{Step 3: Readout error.}
By Lemma \ref{lemma_2ndblock_ffn},
\[
    |\overline H_{\ell_j,m_j}-H_{\ell_j,m_j}|
    \le 5B_fP^2e^{-M},\quad j<P,
    \qquad
    |\overline{0}_H|\le 4B_fP^2e^{-M}.
\]
Using the joint representation of $\hat f_{r_f,r_x,n}$ in
Proposition \ref{prop:joint_softmax},
\begin{align*}
    \rmT^*(\fraks)-\hat f_{r_f,r_x,n}(\fraks)
    &=
    \sum_{j=1}^{P-1}\widetilde\gamma_j
    (\overline H_{\ell_j,m_j}-H_{\ell_j,m_j})
    +\sum_{j=1}^{P-1}
    (\widetilde\gamma_j-\gamma_{\ell_j,m_j}(\fraks))
    H_{\ell_j,m_j}
    +\widetilde\gamma_P\overline{0}_H.
\end{align*}
Therefore,
\begin{align*}
    \left|\rmT^*(\fraks)-\hat f_{r_f,r_x,n}(\fraks)\right|
    &\le
    \sum_{j=1}^{P-1}\widetilde\gamma_j
    |\overline H_{\ell_j,m_j}-H_{\ell_j,m_j}|
    +\sum_{j=1}^{P-1}
    |\widetilde\gamma_j-\gamma_{\ell_j,m_j}(\fraks)|
    \,|H_{\ell_j,m_j}|
    +\widetilde\gamma_P|\overline{0}_H|\\
    &\le
    5B_fP^2e^{-M}
    +2B_f(3+47B_f^2)M_fM_2P^2e^{-M}\\
    &\quad
    +B_fe^{-M}+4B_fP^2e^{-2M}.
\end{align*}
Since $M_f,M_2,P\ge1$ and $Pe^{-M}\le1/4$,
\begin{align*}
    \left|\rmT^*(\fraks)-\hat f_{r_f,r_x,n}(\fraks)\right|
    &\le
    \left(5+2(3+47B_f^2)+1+1\right)
    B_fM_fM_2P^2e^{-M}\\
    &\le
    (13+94B_f^2)B_fM_fM_2P^2e^{-M}.
\end{align*}
This proves \eqref{final-network-error}.
\end{proof}

\section{Proofs of the Statistical Lemmas}
\label{app:generalization_proofs}

This section proves the two statistical lemmas stated in Section
\ref{sec:proof-theorem2}. We first establish the metric-entropy estimate of the Transformer hypothesis class in
Lemma~\ref{lemma:icl_covering_number} by propagating parameter perturbations
through the three encoder blocks. We then use a bounded Bernstein inequality
to prove the oracle inequality in Lemma~\ref{lemma:icl_oracle}.

\subsection{Proof of Lemma \ref{lemma:icl_covering_number}}
\label{appsub:lemma:icl_covering_number}
\begin{proof}[Proof of Lemma \ref{lemma:icl_covering_number}]
Let $\rmT_{\theta},\rmT_{\widetilde\theta}\in\cT$ satisfy
\[
    \|\theta\|_\infty\le M_{\max},\qquad
    \|\widetilde\theta\|_\infty\le M_{\max},\qquad
    \|\theta-\widetilde\theta\|_\infty\le\delta .
\]
For brevity, write $M:=M_{\max}$. Since $\cMx\subset[0,1]^d$ and
$|f(\xb)|\le B_f$, every entry of the prompt matrix $\bX(\fraks)$ is
bounded by $\max\{1,B_f\}$. Recall from the pre-processing formula
in Lemma \ref{lemma:preprocessing} that
\[
    \bZ_0
    =\bW_E\bX(\fraks)+\bb_E\bm{1}_P^\top+\bP,
    \qquad
    \widetilde{\bZ}_0
    =\widetilde{\bW}_E\bX(\fraks)
    +\widetilde{\bb}_E\bm{1}_P^\top+\widetilde{\bP}.
\]
Consequently,
\begin{align*}
    \max\{\|\bZ_0\|_{\max},
    \|\widetilde{\bZ}_0\|_{\max}\}
    &\le
    \big((d+1)\max\{1,B_f\}+2\big)M
    =C_0M,\\
    \|\bZ_0-\widetilde{\bZ}_0\|_{\max}
    &\le
    (d+1)\max\{1,B_f\}
    \|\bW_E-\widetilde{\bW}_E\|_{\max}
    +\|\bb_E-\widetilde{\bb}_E\|_{\max}
    +\|\bP-\widetilde{\bP}\|_{\max}\\
    &\le
    \big((d+1)\max\{1,B_f\}+2\big)\delta
    =C_0\delta,
\end{align*}
where $C_0:=(d+1)\max\{1,B_f\}+2$.
For $i\in\{0,1,2,3\}$, define
\[
    R_i:=\max\{\|\bZ_i\|_{\max},\|\widetilde{\bZ}_i\|_{\max}\},
    \qquad E_i:=\|\bZ_i-\widetilde{\bZ}_i\|_{\max}.
\]
Consider the $h$-th attention head in block $i$. Writing
\[
    \bS_i^h
    =\bZ_{i-1}^\top(\bK_i^h)^\top\bQ_i^h\bZ_{i-1},
    \qquad
    \widetilde{\bS}_i^h
    =
    \widetilde{\bZ}_{i-1}^{\top}
    (\widetilde{\bK}_i^h)^\top\widetilde{\bQ}_i^h
    \widetilde{\bZ}_{i-1},
\]
we have, since $d_k^i\le5$,
\[
    \|(\bK_i^h)^\top\bQ_i^h
    -(\widetilde{\bK}_i^h)^\top\widetilde{\bQ}_i^h\|_{\max}
    \le d_k^i\left(
    \|\bK_i^h-\widetilde{\bK}_i^h\|_{\max}\|\bQ_i^h\|_{\max}
    +\|\widetilde{\bK}_i^h\|_{\max}
    \|\bQ_i^h-\widetilde{\bQ}_i^h\|_{\max}\right)
    \le10M\delta.
\]
We also have
\[
    \max\{\|(\bK_i^h)^\top\bQ_i^h\|_{\max},
    \|(\widetilde{\bK}_i^h)^\top\widetilde{\bQ}_i^h\|_{\max}\}
    \le d_k^iM^2\le5M^2.
\]
Expanding the score difference gives
\begin{align*}
    \bS_i^h-\widetilde{\bS}_i^h
    &=(\bZ_{i-1}-\widetilde{\bZ}_{i-1})^\top
    (\bK_i^h)^\top\bQ_i^h\bZ_{i-1}
    +\widetilde{\bZ}_{i-1}^{\top}
    \left((\bK_i^h)^\top\bQ_i^h
    -(\widetilde{\bK}_i^h)^\top\widetilde{\bQ}_i^h\right)
    \bZ_{i-1}\\
    &\quad+\widetilde{\bZ}_{i-1}^{\top}
    (\widetilde{\bK}_i^h)^\top\widetilde{\bQ}_i^h
    (\bZ_{i-1}-\widetilde{\bZ}_{i-1}).
\end{align*}
Consequently,
\[
    \|\bS_i^h-\widetilde{\bS}_i^h\|_{\max}
    \le
    10D^2\left(
    M^2R_{i-1}E_{i-1}+MR_{i-1}^2\delta\right).
\]
By Lemma \ref{Softmaxlipchitz}, for every column $j$,
\[
    \|(\bA_i^h)_{:,j}-(\widetilde{\bA}_i^h)_{:,j}\|_1
    \le2\|\bS_i^h-\widetilde{\bS}_i^h\|_{\max}
    \le20D^2\left(
    M^2R_{i-1}E_{i-1}+MR_{i-1}^2\delta\right).
\]
Since both attention matrices are column-stochastic,
\begin{align*}
    &\bV_i^h\bZ_{i-1}\bA_i^h
    -\widetilde{\bV}_i^h\widetilde{\bZ}_{i-1}
    \widetilde{\bA}_i^h\\
    &\quad=
    (\bV_i^h-\widetilde{\bV}_i^h)\bZ_{i-1}\bA_i^h
    +\widetilde{\bV}_i^h
    (\bZ_{i-1}-\widetilde{\bZ}_{i-1})\bA_i^h
    +\widetilde{\bV}_i^h\widetilde{\bZ}_{i-1}
    (\bA_i^h-\widetilde{\bA}_i^h),
\end{align*}
so we have
\begin{align*}
    &\|\bV_i^h\bZ_{i-1}\bA_i^h
    -\widetilde{\bV}_i^h\widetilde{\bZ}_{i-1}
    \widetilde{\bA}_i^h\|_{\max}\\
    &\le
    DR_{i-1}\delta+DME_{i-1}+
    DMR_{i-1}\max_{j\in[P]}
    \|(\bA_i^h)_{:,j}-(\widetilde{\bA}_i^h)_{:,j}\|_1\\
    &\le
    DME_{i-1}+DR_{i-1}\delta+
    20D^3\left(
    M^3R_{i-1}^2E_{i-1}+M^2R_{i-1}^3\delta\right).
\end{align*}
Also,
\begin{align*}
    &\max\{\|\bV_i^h\bZ_{i-1}\bA_i^h\|_{\max},
    \|\widetilde{\bV}_i^h\widetilde{\bZ}_{i-1}
    \widetilde{\bA}_i^h\|_{\max}\}\\
    &\quad\le DMR_{i-1}\max_{j\in[P]}
    \max\{\|(\bA_i^h)_{:,j}\|_1,
    \|(\widetilde{\bA}_i^h)_{:,j}\|_1\}=DMR_{i-1}.
\end{align*}
Collect the head outputs as
\[
    \bG_i:=
    \begin{bmatrix}
        \bV_i^1\bZ_{i-1}\bA_i^1\\[-1mm]
        \vdots\\[-1mm]
        \bV_i^{H^i}\bZ_{i-1}\bA_i^{H^i}
    \end{bmatrix},
    \qquad
    \widetilde{\bG}_i:=
    \begin{bmatrix}
        \widetilde{\bV}_i^1\widetilde{\bZ}_{i-1}
        \widetilde{\bA}_i^1\\[-1mm]
        \vdots\\[-1mm]
        \widetilde{\bV}_i^{H^i}\widetilde{\bZ}_{i-1}
        \widetilde{\bA}_i^{H^i}
    \end{bmatrix}.
\]
Since $\cA_i(\bZ_{i-1})=\bW_i^O\bG_i$ and
$\widetilde{\cA}_i(\widetilde{\bZ}_{i-1})
=\widetilde{\bW}_i^O\widetilde{\bG}_i$, we have
\[
    \cA_i(\bZ_{i-1})
    -\widetilde{\cA}_i(\widetilde{\bZ}_{i-1})
    =(\bW_i^O-\widetilde{\bW}_i^O)\bG_i
    +\widetilde{\bW}_i^O(\bG_i-\widetilde{\bG}_i).
\]
Therefore, using $d_v^i\le2$,
\begin{equation}
\label{one-mha-cover-proof}
\begin{aligned}
    &\|\cA_i(\bZ_{i-1})
    -\widetilde{\cA}_i(\widetilde{\bZ}_{i-1})\|_{\max}\\
    &\le
    H^i d_v^i\delta
    \max_h\|\bV_i^h\bZ_{i-1}\bA_i^h\|_{\max}\\
    &\quad+H^i d_v^iM\max_h
    \|\bV_i^h\bZ_{i-1}\bA_i^h
    -\widetilde{\bV}_i^h\widetilde{\bZ}_{i-1}
    \widetilde{\bA}_i^h\|_{\max}\\
    &\le
    2H^iDMR_{i-1}\delta
    +2H^iM\left[
    DME_{i-1}+DR_{i-1}\delta+
    20D^3\left(
    M^3R_{i-1}^2E_{i-1}+M^2R_{i-1}^3\delta\right)
    \right]\\
    &\le
    2H^iDM^2E_{i-1}+4H^iDMR_{i-1}\delta+
    40H^iD^3\left(
    M^4R_{i-1}^2E_{i-1}+M^3R_{i-1}^3\delta\right)\\
    &\le
    46H^iD^3M^4(1+R_{i-1})^3(E_{i-1}+\delta).
\end{aligned}
\end{equation}
Similarly,
\begin{equation}
\label{one-mha-mag-proof}
    \max\{\|\cA_i(\bZ_{i-1})\|_{\max},
    \|\widetilde{\cA}_i(\widetilde{\bZ}_{i-1})\|_{\max}\}
    \le2H^iDM^2(1+R_{i-1}).
\end{equation}
For the point-wise FFN, denote its inputs in the two networks by
\[
    \bZ_i^{\cA}:=\cA_i(\bZ_{i-1}),\qquad
    \widetilde{\bZ}_i^{\cA}
    :=\widetilde{\cA}_i(\widetilde{\bZ}_{i-1}),
\]
and set
\[
    R_i^{\cA}
    :=\max\{\|\bZ_i^{\cA}\|_{\max},
    \|\widetilde{\bZ}_i^{\cA}\|_{\max}\},
    \qquad
    E_i^{\cA}
    :=\|\bZ_i^{\cA}-\widetilde{\bZ}_i^{\cA}\|_{\max}.
\]
For every column $j\in[P]$,
\begin{align*}
    &\max\left\{
    \left\|\bW_i^1(\bZ_i^{\cA})_{:,j}+\bb_i^1\right\|_\infty,
    \left\|\widetilde{\bW}_i^1
    (\widetilde{\bZ}_i^{\cA})_{:,j}
    +\widetilde{\bb}_i^1\right\|_\infty
    \right\} \le DMR_i^{\cA}+M\le DM(1+R_i^{\cA}),\\
    &\bW_i^1(\bZ_i^{\cA})_{:,j}+\bb_i^1
    -\widetilde{\bW}_i^1(\widetilde{\bZ}_i^{\cA})_{:,j}
    -\widetilde{\bb}_i^1 =
    (\bW_i^1-\widetilde{\bW}_i^1)(\bZ_i^{\cA})_{:,j}
    +\widetilde{\bW}_i^1
    \left((\bZ_i^{\cA})_{:,j}
    -(\widetilde{\bZ}_i^{\cA})_{:,j}\right)
    +\bb_i^1-\widetilde{\bb}_i^1,\\
    &\left\|
    \bW_i^1(\bZ_i^{\cA})_{:,j}+\bb_i^1
    -\widetilde{\bW}_i^1(\widetilde{\bZ}_i^{\cA})_{:,j}
    -\widetilde{\bb}_i^1
    \right\|_\infty \le D\delta R_i^{\cA}+DME_i^{\cA}+\delta \le DM(1+R_i^{\cA})(E_i^{\cA}+\delta).
\end{align*}
Using the $1$-Lipschitz property of ReLU and
$d_{\mathrm{ff}}^i\le2D$, the difference between the corresponding
FFN output columns decomposes as
\begin{align*}
    &\bW_i^2\sigma\left(
    \bW_i^1(\bZ_i^{\cA})_{:,j}+\bb_i^1\right)+\bb_i^2 -
    \widetilde{\bW}_i^2\sigma\left(
    \widetilde{\bW}_i^1(\widetilde{\bZ}_i^{\cA})_{:,j}
    +\widetilde{\bb}_i^1\right)
    -\widetilde{\bb}_i^2\\
    &\quad=
    (\bW_i^2-\widetilde{\bW}_i^2)
    \sigma\left(\bW_i^1(\bZ_i^{\cA})_{:,j}+\bb_i^1\right)\\
    &\qquad+
    \widetilde{\bW}_i^2
    \left[
    \sigma\left(\bW_i^1(\bZ_i^{\cA})_{:,j}+\bb_i^1\right)
    -\sigma\left(
    \widetilde{\bW}_i^1(\widetilde{\bZ}_i^{\cA})_{:,j}
    +\widetilde{\bb}_i^1\right)
    \right]
    +\bb_i^2-\widetilde{\bb}_i^2.
\end{align*}
Therefore,
\begin{equation}
\label{one-ffn-cover-proof}
\begin{aligned}
    &\|\cF_i(\bZ_i^{\cA})
    -\widetilde{\cF}_i(\widetilde{\bZ}_i^{\cA})\|_{\max}\\
    &\le
    d_{\mathrm{ff}}^i\delta
    \max_{j\in[P]}
    \left\|\sigma\left(
    \bW_i^1(\bZ_i^{\cA})_{:,j}+\bb_i^1\right)\right\|_\infty\\
    &\quad+
    d_{\mathrm{ff}}^iM
    \max_{j\in[P]}
    \left\|
    \sigma\left(\bW_i^1(\bZ_i^{\cA})_{:,j}+\bb_i^1\right)
    -\sigma\left(
    \widetilde{\bW}_i^1(\widetilde{\bZ}_i^{\cA})_{:,j}
    +\widetilde{\bb}_i^1\right)
    \right\|_\infty+\delta\\
    &\le
    2D^2M(1+R_i^{\cA})\delta
    +2D^2M^2(1+R_i^{\cA})(E_i^{\cA}+\delta)
    +\delta\\
    &\le5D^2M^2(1+R_i^{\cA})(E_i^{\cA}+\delta).
\end{aligned}
\end{equation}
Moreover,
\begin{equation}
\label{one-ffn-mag-proof}
\begin{aligned}
    &1+\max\{\|\cF_i(\bZ_i^{\cA})\|_{\max},
    \|\widetilde{\cF}_i(\widetilde{\bZ}_i^{\cA})\|_{\max}\}\\
    &\le
    1+d_{\mathrm{ff}}^iM
    \max_{j\in[P]}
    \max\left\{
    \left\|\sigma\left(
    \bW_i^1(\bZ_i^{\cA})_{:,j}+\bb_i^1\right)\right\|_\infty,
    \left\|\sigma\left(
    \widetilde{\bW}_i^1(\widetilde{\bZ}_i^{\cA})_{:,j}
    +\widetilde{\bb}_i^1\right)\right\|_\infty
    \right\}+M\\
    &\le4D^2M^2(1+R_i^{\cA}).
\end{aligned}
\end{equation}
Since $\bZ_i=\cF_i(\bZ_i^{\cA})$ and
$\widetilde{\bZ}_i
=\widetilde{\cF}_i(\widetilde{\bZ}_i^{\cA})$, combining
\eqref{one-mha-cover-proof}--\eqref{one-ffn-mag-proof} within block $i$
yields
\begin{align*}
    1+R_i
    &\le
    4D^2M^2\left(1+2H^iDM^2(1+R_{i-1})\right) \le12H^iD^3M^4(1+R_{i-1}),\\
    E_i
    &\le
    5D^2M^2\left(1+2H^iDM^2(1+R_{i-1})\right)
    \left(
    46H^iD^3M^4(1+R_{i-1})^3(E_{i-1}+\delta)
    +\delta\right)\\
    &\le
    705(H^i)^2D^6M^8(1+R_{i-1})^4(E_{i-1}+\delta).
\end{align*}
For the architecture of Theorem \ref{thm:total_approx}, since
$P\ge n+2\ge3$,
\[
    H^1=(2n+3)P+3\le(2n+4)P\le6nP,\qquad
    H^2=nP+3\le(n+1)P\le2nP,
    \qquad H^3=1.
\]
The initial estimates give
$1+R_0\le2C_0M$ and $E_0+\delta\le2C_0\delta$.
Substitution into the preceding recursions, block by block, gives
\begin{align*}
    1+R_1&\le\exp(4)C_0H^1D^3M^5,&
    E_1&\le\exp(11)C_0^5(H^1)^2D^6M^{12}\delta,\\
    1+R_2&\le\exp(6)C_0H^1H^2D^6M^9,&
    E_2&\le
    \exp(30)C_0^9(H^1)^6(H^2)^2D^{24}M^{40}\delta,\\
    1+R_3
    &\le\exp(9)C_0H^1H^2D^9M^{13},&
    E_3
    &\le\exp(59)
    C_0^{13}(H^1)^{10}(H^2)^6D^{54}M^{84}\delta.
\end{align*}
For the readout vectors $\bc,\widetilde\bc\in\RR^{DP}$,
\begin{align*}
    |\rmT_\theta(\fraks)-\rmT_{\widetilde\theta}(\fraks)|
    &\le
    \|\bc\|_1
    \|\cT_{3,\theta}(\fraks)-\cT_{3,\widetilde\theta}(\fraks)\|_{\max}
    +\|\bc-\widetilde\bc\|_1
    \|\cT_{3,\widetilde\theta}(\fraks)\|_{\max}\\
    &\le
    DPME_3+DP\delta R_3\\
    &\le
    \exp(59)C_0^{13}
    P(H^1)^{10}(H^2)^6D^{55}M^{85}\delta
    +\exp(9)C_0PH^1H^2D^{10}M^{13}\delta\\
    &\le
    \exp(82)C_0^{13}
    P^{17}D^{55}n^{16}M^{85}\delta .
\end{align*}
Thus the parameter-to-function map is Lipschitz with constant
\begin{equation}
\label{parameter-lipschitz-bound}
    L_{\mathrm{par}}
    :=
    \exp(82)C_0^{13}
    P^{17}D^{55}n^{16}M^{85},
\end{equation}
An $\eta$-cover in the function-space sup norm is obtained by taking an
$\ell_\infty$ $\delta$-cover, with $\delta=\eta/L_{\mathrm{par}}$, of the
free-parameter cube $[-M,M]^{\mathcal N_{\mathrm{total}}}$.
This cube has such a $\delta$-cover of cardinality at most
$(1+2M/\delta)^{\mathcal N_{\mathrm{total}}}$. Since
$\eta\le1$ and $M,L_{\mathrm{par}}\ge1$, this is at most
$(3ML_{\mathrm{par}}/\eta)^{\mathcal N_{\mathrm{total}}}$. Hence
\begin{align*}
    \log\mathcal N(\eta,\cT,\|\cdot\|_\infty)
    &\le
    \mathcal N_{\mathrm{total}}
    \log\left(\frac{3ML_{\mathrm{par}}}{\eta}\right)\\
    &\le
    \mathcal N_{\mathrm{total}}
    \log\left(
    \frac{C_{\mathrm{cov}}P^{17}D^{55}n^{16}M^{86}}
    {\eta}
    \right).
\end{align*}
Here
\[
    C_{\mathrm{cov}}
    :=3\exp(82)C_0^{13}
    =3\exp(82)
    \left((d+1)\max\{1,B_f\}+2\right)^{13}.
\]
Since $\pi_{B_f}$ is $1$-Lipschitz,
\[
    \mathcal N(\eta,\pi_{B_f}\cT,\|\cdot\|_\infty)
    \le
    \mathcal N(\eta,\cT,\|\cdot\|_\infty).
\]
This proves \eqref{icl-covering-bound}.
\end{proof}

\subsection{Proof of Lemma \ref{lemma:icl_oracle}}
\label{appsub:lemma:icl_oracle}
We use the following standard bounded form of Bernstein's inequality
\citep{boucheron2013concentration}.

\begin{lemma}[Bernstein inequality for bounded nonnegative variables]
\label{lemma:bounded-bernstein}
Let $Z_1,\ldots,Z_\Gamma$ be independent copies of a random variable
$Z\in[0,b]$, where $b>0$. Write
\[
    PZ:=\EE Z,
    \qquad
    P_\Gamma Z:=\frac1\Gamma\sum_{\gamma=1}^{\Gamma}Z_\gamma.
\]
Since $0\le Z\le b$ almost surely,
\[
    \operatorname{Var}(Z)
    \le\EE Z^2
    \le b\EE Z
    =bPZ.
\]
Consequently, for every $t>0$,
\[
\begin{aligned}
    \PP\{P_\Gamma Z-PZ\ge t\}
    &\le
    \exp\left(-\frac{\Gamma t^2}
    {2bPZ+\frac23bt}\right),\\
    \PP\{PZ-P_\Gamma Z\ge t\}
    &\le
    \exp\left(-\frac{\Gamma t^2}
    {2bPZ+\frac23bt}\right).
\end{aligned}
\]
\end{lemma}

\begin{proof}[Proof of Lemma \ref{lemma:icl_oracle}]
For $\rmT\in\cT$, write
\[
    \psi_{\rmT}(\fraks,y_{n+1}):=(\rmT(\fraks)-y_{n+1})^2.
\]
Then
\[
    0\le\psi_{\rmT}\le4B_f^2,
    \qquad
    \operatorname{Var}(\psi_{\rmT})
    \le\EE\psi_{\rmT}^2
    \le4B_f^2\cL(\rmT).
\]
By compactness, choose $\rmT^\circ\in\cT$ with
$\cL(\rmT^\circ)=\inf_{\rmT\in\cT}\cL(\rmT)\le R^2$, and set
\[
\begin{aligned}
    S_1&:=\cL_{\frakS}(\rmT^\circ)-\cL(\rmT^\circ),\qquad
    S_2&:=\cL(\widehat{\rmT}_{\frakS})
          -\cL_{\frakS}(\widehat{\rmT}_{\frakS}).
\end{aligned}
\]
The ERM property gives
\begin{equation}
\label{oracle-risk-decomposition}
\begin{aligned}
    \cL(\widehat{\rmT}_{\frakS})
    &=\cL_{\frakS}(\widehat{\rmT}_{\frakS})+S_2
      \le\cL_{\frakS}(\rmT^\circ)+S_2 =\cL(\rmT^\circ)+S_1+S_2
      \le R^2+S_1+S_2.
\end{aligned}
\end{equation}

We first bound $S_1$. For $\eta\ge R^2$, Lemma~\ref{lemma:bounded-bernstein}
applied to $\psi_{\rmT^\circ}$ yields
\begin{equation}
\label{oracle-S1-bound}
    \PP\{S_1>\eta\}
    \le\exp\left(
    -\frac{\Gamma\eta^2}
    {8B_f^2\cL(\rmT^\circ)+\frac83B_f^2\eta}
    \right)
    \le\exp\left(-\frac{3\Gamma\eta}{32B_f^2}\right).
\end{equation}

We next control $S_2$.
Let $\cT_\eta\subset\cT$ be a
$\min\{1,\eta/(64B_f)\}$-net in $\|\cdot\|_\infty$ such that
\[
    \log|\cT_\eta|
    \le A\log\left(V\max\left\{1,\frac{64B_f}{\eta}\right\}\right).
\]
For each $\rmT\in\cT$, choose a corresponding $\rmT_0\in\cT_\eta$. Then
\[
    \|\psi_{\rmT}-\psi_{\rmT_0}\|_\infty
    \le4B_f\|\rmT-\rmT_0\|_\infty
    \le\frac{\eta}{16}.
\]
For every $\rmT_0\in\cT_\eta$, Lemma~\ref{lemma:bounded-bernstein} gives
\begin{align*}
    \PP\left\{
      \cL(\rmT_0)-\cL_{\frakS}(\rmT_0)
      >\frac12\cL(\rmT_0)+\eta
      \right\} &\le\exp\left(
      -\frac{\Gamma\big(\frac12\cL(\rmT_0)+\eta\big)^2}
      {8B_f^2\cL(\rmT_0)+\frac83B_f^2\big(\frac12\cL(\rmT_0)+\eta\big)}
      \right)\\
    &\le\exp\left(
      -\frac{3\Gamma\big(\frac12\cL(\rmT_0)+\eta\big)}{56B_f^2}
      \right)
      \le\exp\left(-\frac{3\Gamma\eta}{56B_f^2}\right).
\end{align*}
By a union bound, with probability at least
$1-|\cT_\eta|\exp(-3\Gamma\eta/(56B_f^2))$, simultaneously for all
$\rmT\in\cT$,
\begin{align*}
    \cL(\rmT)-\cL_{\frakS}(\rmT)
    &\le\cL(\rmT_0)-\cL_{\frakS}(\rmT_0)
       +2\|\psi_{\rmT}-\psi_{\rmT_0}\|_\infty \le\frac12\cL(\rmT_0)+\eta
       +2\|\psi_{\rmT}-\psi_{\rmT_0}\|_\infty\\
    &\le\frac12\cL(\rmT)+\eta
       +\frac52\|\psi_{\rmT}-\psi_{\rmT_0}\|_\infty \le\frac12\cL(\rmT)+\frac32\eta.
\end{align*}
Taking $\rmT=\widehat{\rmT}_{\frakS}$ gives
\begin{equation}
\label{oracle-S2-bound}
    \PP\left\{S_2>\frac12\cL(\widehat{\rmT}_{\frakS})+\frac32\eta\right\}
    \le\exp\left(\log|\cT_\eta|-\frac{3\Gamma\eta}{56B_f^2}\right).
\end{equation}
Combining \eqref{oracle-risk-decomposition}--\eqref{oracle-S2-bound}, for
$\eta\ge R^2$,
\begin{equation}
\label{oracle-tail-scale}
\begin{aligned}
    \PP\left\{\cL(\widehat{\rmT}_{\frakS})>2R^2+5\eta\right\} &\le\PP\{S_1>\eta\}
    +\PP\left\{S_2>\frac12\cL(\widehat{\rmT}_{\frakS})+\frac32\eta\right\}\\
    &\le\exp\left(-\frac{3\Gamma\eta}{32B_f^2}\right)
    +\exp\left(\log|\cT_\eta|-\frac{3\Gamma\eta}{56B_f^2}\right).
\end{aligned}
\end{equation}

For an exponential tail bound on the population risk, set
\[
    \eta_R:=R^2+\frac{80B_f^2}{\Gamma}
    \left(A\log\frac{8(1+B_f)V}{R}+1\right).
\]
For $\eta\ge\eta_R$, the assumptions $R\le2B_f$ and $V\ge1$ imply
\[
    V\max\left\{1,\frac{64B_f}{\eta}\right\}
    \le V\max\left\{1,\frac{64B_f}{R^2}\right\}
    \le\left(\frac{8(1+B_f)V}{R}\right)^2.
\]
Consequently,
\begin{align*}
    \log|\cT_\eta|
    &\le A\log\left(V\max\left\{1,\frac{64B_f}{\eta}\right\}\right)
      \le2A\log\frac{8(1+B_f)V}{R}\\
    &\le\frac{15}{7}\left(A\log\frac{8(1+B_f)V}{R}+1\right)
      \le\frac{3\Gamma\eta_R}{112B_f^2}
      \le\frac{3\Gamma\eta}{112B_f^2}.
\end{align*}
Thus \eqref{oracle-tail-scale} yields
\[
    \PP\left\{\cL(\widehat{\rmT}_{\frakS})>2R^2+5\eta\right\}
    \le2\exp\left(-\frac{3\Gamma\eta}{112B_f^2}\right),
    \qquad\eta\ge\eta_R.
\]

Integrating this tail bound gives the expected population risk:
\begin{align*}
    \EE_{\frakS}\cL(\widehat{\rmT}_{\frakS})
    &=\int_0^\infty
       \PP\left\{\cL(\widehat{\rmT}_{\frakS})>t\right\}\,dt\\
    &\le2R^2+5\eta_R
       +\int_{2R^2+5\eta_R}^{\infty}
       \PP\left\{\cL(\widehat{\rmT}_{\frakS})>t\right\}\,dt\\
    &=2R^2+5\eta_R
       +5\int_{\eta_R}^{\infty}
       \PP\left\{\cL(\widehat{\rmT}_{\frakS})>2R^2+5\eta\right\}\,d\eta\\
    &\le2R^2+5\eta_R
       +10\int_{\eta_R}^{\infty}
       \exp\left(-\frac{3\Gamma\eta}{112B_f^2}\right)\,d\eta\\
    &\le2R^2+5\eta_R+\frac{1120B_f^2}{3\Gamma} =7R^2+\frac{B_f^2}{\Gamma}
       \left(400A\log\frac{8(1+B_f)V}{R}+\frac{2320}{3}\right)\\
    &\le640R^2+640B_f^2
       \frac{A\log(640(1+B_f)V/R)+1}{\Gamma},
\end{align*}
where the last inequality uses
\[
    \log\frac{8(1+B_f)V}{R}\ge\log4>\frac23,
    \qquad
    \frac{2320}{3}\le640+240A\log\frac{8(1+B_f)V}{R}.
\]
This proves \eqref{icl-oracle-expectation}.
\end{proof}

\bibliographystyle{plain}
\bibliography{ref}

\end{document}